\documentclass[a4paper,reqno]{amsart}
\usepackage[british]{babel}
\usepackage[T1]{fontenc}
\usepackage{amssymb,amsthm,bm}
\usepackage{mathtools}
\usepackage{nicefrac}
  \mathtoolsset{mathic}
  \allowdisplaybreaks
\usepackage{courier,mathptmx,stmaryrd}
\usepackage{longtable}
\usepackage[numbers,sort&compress]{natbib}
\usepackage{hyperref,url}
  \hypersetup{  % bookmarks
    bookmarksnumbered
  }
  \hypersetup{  % links
    breaklinks=false,
    pdfborderstyle={/S/U/W 0},  
    citebordercolor=.235 .702 .443,
    urlbordercolor=.255 .412 .882,
    linkbordercolor=.804 .149 .19,
  }
  \hypersetup{ % metadata
    pdfauthor={Gert de Cooman},
    pdftitle={Coherent and Archimedean choice in general Banach spaces},
    pdfkeywords={}
  }
\usepackage[all]{hypcap} 
\usepackage{xcolor}
\usepackage{tikz}
\usepackage{enumitem}
\title[Coherent and Archimedean choice in general Banach spaces]{Coherent and Archimedean choice in general Banach spaces}
\author{Gert de Cooman}
\address{Ghent University, Foundations Lab for Imprecise Probabilities, Technologiepark-Zwijnaarde 125, 9052 Zwijnaarde, Belgium}
\email{gert.decooman@ugent.be}
\usepackage{etoolbox}

\DeclarePairedDelimiter{\group}{(}{)}
\DeclarePairedDelimiter{\sqgroup}{[}{]}
\DeclarePairedDelimiter{\set}{\{}{\}}
\DeclarePairedDelimiter{\structure}{\langle}{\rangle}
\DeclarePairedDelimiter{\normof}{\Vert}{\Vert}
\DeclarePairedDelimiter{\abs}{\vert}{\vert}

\newcommand{\naturals}{\mathbb{N}}

\newcommand{\reals}{\mathbb{R}}
\newcommand{\posreals}{\reals_{>0}}
\newcommand{\nonnegreals}{\reals_{\geq0}}
\newcommand{\unit}{[0,1]}
\newcommand{\states}{\mathcal{X}}
\newcommand{\rewards}{\mathcal{R}}

\newcommand{\statesandrewards}{\states\times\rewards}

\newcommand{\gbls}{\mathcal{G}}

\newcommand{\gblson}[1]{\gbls(#1)}
\newcommand{\gblsonstates}{\gblson{\states}}

\newcommand{\gblsonstatesandrewards}{\gblson{\statesandrewards}}

\newcommand{\hls}{\mathcal{H}}

\newcommand{\difs}{\mathcal{D}} % we may want to choose a different symbol here, multiple use.

\newcommand{\dif}{d}
\newcommand{\altdif}{d'}
\newcommand{\gblgt}{>}

\newcommand{\gblplt}{\olessthan}
\newcommand{\gblpgt}{\ogreaterthan}
\newcommand{\optgt}[1][]{\mathrel{\succ_{#1}}}

\newcommand{\prefgt}[1][]{\mathrel{\rhd_{#1}}}

\newcommand{\hlsgt}{\mathrel{>^\ast}}

\newcommand{\hlpgt}{\mathrel{\ogreaterthan^\ast}}

\newcommand{\hlprefgt}[1][]{\mathrel{\rhd^\ast_{#1}}}

\newcommand{\bestreward}{\top}
\newcommand{\worstreward}{\bot}
\newcommand{\binaryrewards}{\set{\bestreward,\worstreward}}
\newcommand{\opt}[1][]{u_{#1}}
\newcommand{\altopt}[1][]{v_{#1}}
\newcommand{\altopttoo}[1][]{w_{#1}}

\newcommand{\opts}{\mathcal{V}}
\newcommand{\posopts}{\opts_{\optgt0}}

\newcommand{\optset}[1][]{A_{#1}}
\newcommand{\altoptset}[1][]{B_{#1}}

\newcommand{\optsets}{\mathcal{Q}}

\newcommand{\assessment}{\mathcal{A}}

\newcommand{\desirset}[1][]{D_{#1}}
\newcommand{\maxdesirset}[1][]{\hat{D}_{#1}}
\newcommand{\desirsets}{\mathbf{D}}
\newcommand{\cohdesirsets}{\overline{\desirsets}}
\newcommand{\totdesirsets}{\overline{\desirsets}_{\mathrm{T}}}
\newcommand{\mixdesirsets}{\overline{\desirsets}_{\mathrm{M}}}
\newcommand{\archdesirsets}{\overline{\desirsets}_{\mathrm{A}}}

\newcommand{\rejectset}[1][]{K_{#1}}

\newcommand{\natexdesirset}{\closure_{\cohdesirsets}}

\newcommand{\archnatexdesirset}{\closure_{\archdesirsets}}
\newcommand{\natexrejectset}{\closure_{\cohrejectsets}}

\newcommand{\archnatexrejectset}{\closure_{\archcohrejectsets}}
\newcommand{\rejectsets}{\mathbf{K}}
\newcommand{\cohrejectsets}{\overline{\rejectsets}}
\newcommand{\archcohrejectsets}{\overline{\rejectsets}_{\mathrm{A}}}

\newcommand{\mixrejectsets}{\overline{\rejectsets}_{\mathrm{M}}}

\newcommand{\archrejectsets}{\overline{\rejectsets}_{\mathrm{A}}}

\newcommand{\setofdesirsets}{\mathcal{D}}
\newcommand{\setofdualopts}{\mathcal{L}}
\newcommand{\setofldualopts}{\underline{\setofdualopts}}
\newcommand{\choicefun}[1][]{C_{#1}}
\newcommand{\rejectfun}[1][]{R_{#1}}

\newcommand{\hl}[1][]{H_{#1}}
\newcommand{\althl}[1][]{G_{#1}}
\newcommand{\althltoo}[1][]{F_{#1}}

\newcommand{\gbl}[1][]{h_{#1}}
\newcommand{\altgbl}[1][]{g_{#1}}

\newcommand{\dualopt}[1][]{\Lambda_{#1}}

\newcommand{\dualopts}{\opts^\star}
\newcommand{\posdualopts}{\opts^\star_{\optgt0}}

\newcommand{\ddualopt}[1][]{\Gamma_{#1}}
\newcommand{\ddualopts}{\opts^\circ}
\newcommand{\ldualopt}[1][]{\underline{\Lambda}_{#1}}
\newcommand{\udualopt}[1][]{\overline{\Lambda}_{#1}}
\newcommand{\maxludualopt}[2][]{\abs{\overline{\underline{\Lambda}}}\group[#1]{#2}}
\newcommand{\ldualopto}[1][\desirset]{\underline{\Lambda}_{#1,\opt[o]}}
\newcommand{\udualopto}[1][\desirset]{\overline{\Lambda}_{#1,\opt[o]}}
\newcommand{\dualopto}[1][\desirset]{\Lambda_{#1,\opt[o]}}
\newcommand{\ldualopts}[1][]{\underline{\opts}^\star}
\newcommand{\posldualopts}{\underline{\opts}^\star_{\optgt0}}

\newcommand{\optnorm}[2][]{{\normof[#1]{#2}}_{\opts}}
\newcommand{\ddualoptnorm}[2][]{{\normof[#1]{#2}}_{\ddualopts}}
\newcommand{\nml}[1][{\opt[o]}]{\normalise_{#1}}
\newcommand{\cset}[3][]{\set[#1]{#2\colon#3}}
\newcommand{\ind}[1]{\mathbb{I}_{#1}}
\newcommand{\indset}[1]{\ind{\set{#1}}}
\newcommand{\then}{\Rightarrow}

\newcommand{\ifandonlyif}{\Leftrightarrow}

\newcommand{\co}[1]{#1^\mathrm{c}}
\newcommand{\bolleke}{\vcenter{\hbox{\scalebox{0.8}{\(\bullet\)}}}}
\DeclareMathOperator{\posi}{posi}

\DeclareMathOperator{\linspan}{span}

\DeclareMathOperator{\closure}{cl}
\DeclareMathOperator{\topcls}{Cl}
\DeclareMathOperator{\topint}{Int}

\DeclareMathOperator{\normalise}{N}
\newcommand{\heads}{\mathtt{H}}
\newcommand{\tails}{\mathtt{T}}
\newcommand{\coinstates}{\set{\heads,\tails}}
\newcommand{\coinstatesandrewards}{\coinstates\times\binaryrewards}
\newcommand{\stopit}{\hfill\(\triangle\)}
\newcommand{\unitdif}{1_\difs}
\newcommand{\difnorm}[2][]{{\normof[#1]{#2}}_{\difs}}
\newcommand{\ddualdifs}{\difs^\circ}
\newcommand{\dualdifs}{\difs^\star}
\newcommand{\ldualdifs}{\underline{\difs}^\star}
\newcommand{\ddualdifnorm}[2][]{{\normof[#1]{#2}}_{\ddualdifs}}
\newcommand{\lp}{\underline{p}}
\newcommand{\up}{\overline{p}}
\newcommand{\lup}{{[\lp,\up]}}
\newcommand{\ex}[1][]{E_{#1}}
\newcommand{\lex}[1][]{\underline{E}_{#1}}

\newtheorem{theorem}{Theorem}
\newtheorem{proposition}[theorem]{Proposition}

\newtheorem{corollary}[theorem]{Corollary}
\theoremstyle{definition}

\theoremstyle{remark}

\newtheorem{runningexample}{Running example}

\newcommand{\archim}{Ar\-chi\-medean}
\newcommand{\essarchim}{essentially Ar\-chi\-medean}

\newcommand{\archimty}{Ar\-chi\-medeanity}
\newcommand{\essarchimty}{essential Ar\-chi\-medeanity}
\newcommand{\Essarchimty}{Essential Ar\-chi\-medeanity}
\newcommand{\weakcircle}{weak\({}^\circ\)}
\newcommand{\weakstar}{weak\({}^\star\)}
\begin{document}
\begin{abstract}
I introduce and study a new notion of Archimedeanity for binary and non-binary choice between options that live in an abstract Banach space, through a very general class of choice models, called sets of desirable option sets.
In order to be able to bring an important diversity of contexts into the fold, amongst which choice between horse lottery options, I pay special attention to the case where these linear spaces don't include all `constant' options.
I consider the frameworks of conservative inference associated with Archimedean (and coherent) choice models, and also pay quite a lot of attention to representation of general (non-binary) choice models in terms of the simpler, binary ones.
The representation theorems proved here provide an axiomatic characterisation for, amongst many other choice methods, Levi's E-admissibility and Walley--Sen maximality.
\end{abstract}
\maketitle

%KEYWORDS
% choice functions
% sets of desirable options
% sets of desirable option sets
% coherence
% representation theorem
% natural extension theorem

\section{Introduction}\label{sec:introduction}
This paper deals with rational decision making under uncertainty using choice functions, along the lines established by Teddy Seidenfeld and colleagues~\cite{seidenfeld2010}.
What are the underlying ideas?
A subject is to choose between options~\(\opt\), which are typically uncertain rewards, and which live in a so-called option space~\(\opts\).
Her choices are typically represented using a \emph{rejection function}~\(\rejectfun\) or a \emph{choice function}~\(\choicefun\).
For any finite set~\(\optset\) of options, the subset~\(\rejectfun(\optset)\subseteq\optset\) contains those options that our subject rejects from the option set~\(\optset\), and the remaining options in~\(\choicefun(\optset)\coloneqq\optset\setminus\rejectfun(\optset)\) are the ones that are then still being considered.
It is important to stress here that \(\choicefun(\optset)\) is not necessarily a singleton, so this approach allows for indecision and indifference.
Also, the \emph{binary} choices are the ones where \(\optset\) has only two elements, and I won't be assuming that these binary choices completely determine the behaviour of the functions~\(\rejectfun\) or~\(\choicefun\) on option sets with more than two elements: I'll be considering choice behaviour that is not necessarily binary in nature.

My aim here is to present a theory of \emph{coherent} and \emph{\archim} choice (functions) that is quite general in that it allows for indecision, allows choice to be non-binary, and works with very general option spaces: general Banach spaces that needn't have constants.
This theory comes with a notion of \emph{conservative inference}, which allows us start with simple assessments and use the coherence axioms to check their internal consistency and to infer what they imply.\footnote{I will use `I' and `me' whenever I feel that what is being described is a result of personal choice or inclination. The editorial `we' and `us' will be used to refer to you and me in the common endeavour of working ourselves through this paper.}
It also comes with powerful \emph{representation results}, which show how this type of choice behaviour can be captured in terms of simpler, binary choice models.

For the basic theory of coherent choice (functions) on general linear option spaces but without {\archimty}, I'll rely fairly heavily on earlier work by Jasper De Bock and myself~\cite{debock2018:choice:arxiv,debock2018:choice:smps,debock2019:choice:arxiv,debock2019:choice:isipta}, where we show how to provide such general---but not necessarily {\archim}---choice models with an interpretation that leads to simple coherence (or rationality) axioms, conservative inference, and representation results.
The present paper expands that work to include a discussion of a novel notion of {\archimty}.
Since this more involved approach needs a notion of closeness, I'll need to focus on option spaces that are Banach, but I still want to keep the treatment general enough so as to avoid the need for including constant options.

{\archimty} is typically introduced as some (or other) simplifying assumption for choice models, and its aim is to make sure that the real number scale in some way suffices for representing the complexity of the choice behaviour that they represent.
It typically leads to choice models that involve real functionals of some kind.
Here, I'll introduce and study a version of {\archimty} that will allow both binary and more general choice models to be represented through the intervention of (super)linear bounded real functionals on the option space.

The manifold reasons for working with option spaces that are general linear spaces were discussed at length in earlier papers~\cite{debock2019:choice:arxiv,debock2019:choice:isipta}.
The ideas behind this we first explored by Van Camp~\cite{2017vancamp:phdthesis}.
In summary, doing so allows us to deal with options that are \emph{gambles}, i.e.~bounded real-valued maps on some set of possible states~\(\states\), which are considered as uncertain rewards~\cite{walley1991,troffaes2013:lp}.
These maps constitute a linear space~\(\gblsonstates\), closed under point-wise addition and point-wise multiplication with real numbers.
But working with general linear option spaces also brings in, in one fell swoop, vector-valued gambles~\cite{zaffalon2017:incomplete:preferences,2017vancamp:phdthesis}, polynomial gambles to deal with exchangeability~\cite{cooman2010,vancamp2018:exchangeable:choice}, equivalence classes of gambles to deal with indifference~\cite{vancamp2018:choice:and:indifference}, and abstract gambles defined without an underlying possibility space~\cite{williams2007}.
All these examples come with their own particular version of a natural, or background ordering between options that is always present, even in the absence of any real preferences or beliefs on the part of the subject.
In all these cases, the space of options essentially includes all real constants---or constant gambles.
But there are interesting cases where we also want to consider option spaces that don't include all real constants.

Let me give two examples to indicate why we should want that. 
First of all, in the case of horse lottery options, we consider a finite set~\(\rewards\) of rewards, a set of possible states~\(\states\), and the set~\(\hls\) of all state-dependent probability mass functions~\(\hl\colon\statesandrewards\to\reals\) on rewards, also called \emph{horse lotteries}, where
\begin{equation*}
\hl(x,r)\geq0
\text{ and }
\sum_{s\in\rewards}\hl(x,s)=1
\text{ for all \(x\in\states\) and \(r\in\rewards\).}
\end{equation*}
Horse lotteries are often considered as options between which preferences can be expressed~\cite{anscombe1963,seidenfeld1995,seidenfeld2010,zaffalon2017:incomplete:preferences,2017vancamp:phdthesis}, but they're rather cumbersome to work with, because the set~\(\hls\subseteq\gblsonstatesandrewards\) of all horse lotteries with state space~\(\states\) and reward set~\(\rewards\) doesn't constitute a linear space, but is only closed under convex mixtures.
There are a number of ways to circumvent this problem, which consist in `embedding' the horse lotteries, and the preferences between them, into a suitably chosen linear space of gambles, without adding or losing preference information.

One such approach was suggested by Zaffalon and Miranda \cite{zaffalon2017:incomplete:preferences} for binary preferences, and extended to more general, possibly non-binary, preferences by Van Camp \cite{2017vancamp:phdthesis}.
While mathematically sound, it appears to me to have a slight disadvantage: it requires singling out one reward~\(\worstreward\) in the reward set~\(\rewards\) as special, to reduce the reward set to \(\rewards_\worstreward\coloneqq\rewards\setminus\set{\worstreward}\), and to work with the linear space of gambles \(\gblson{\states\times\rewards_\worstreward}\) as an option space.

It turns out we can get rid of this arguably somewhat arbitrary choice by focusing on another linear subspace of~\(\gblsonstatesandrewards\):
\begin{equation*}
\difs
\coloneqq\linspan(\hls-\hls)
=\cset{\alpha(\hl-\althl)}{\alpha\in\posreals\text{ and }\hl,\althl\in\hls}.
\end{equation*}
For more context and details, I refer to the running example further on, as well as the relevant discussion in Arthur Van Camp's PhD thesis \cite{2017vancamp:phdthesis}.
For now, it is enough to notice that, since \(\sum_{r\in\rewards}\alpha[\hl(\bolleke,r)-\althl(\bolleke,r)]=\alpha(1-1)=0\), \(0\) is the only `real constant' that this linear space of gambles~\(\difs\) contains.
This indicates why, in order to build a theory general enough to incorporate the horse lottery approach without too many restrictions and arbitrary choices, it will help to also pay attention to linear option spaces that don't necessarily include (all) real constants.

A second example is drawn from recent discussions of decision-theoretic uncertainty modelling in quantum mechanics, by Benavoli et al.~\cite{benavoli2016:hermitian}.
In that context, the options are the measurements, which correspond to Hermitian operators on a complex \(n\)-dimensional Hilbert space, and constitute an \(n^2\)-dimensional real linear space without `real constants'. 
The background ordering is related to the notion of positive definiteness for such Hermitian operators.
It turns out that we can meaningfully single out one operator---such as the identity operator---as somehow special, call it the unit constant, and in this way embed the real numbers, or real constants, into the option space as all real multiples of this unit constant.
This simple idea, suitably generalised, will play an instrumental part further on in the discussion.

In order to keep the length of this paper manageable, I've decided to focus on the mathematical developments, and to keep the discussion fairly abstract.
For a detailed exposition of the motivation for and the interpretation of the choice models discussed below, I refer to earlier joint papers by Jasper De Bock and myself~\cite{debock2018:choice:arxiv,debock2018:choice:smps,debock2019:choice:arxiv,debock2019:choice:isipta}.
I also recommend Jasper De Bock's most recent paper on {\archimty}~\cite{debock2020:axiomatic:archimedean:arxiv}, as it contains a persuasive motivation for the new {\archimty} condition I am about to discuss, in the more restrictive and concrete context where options are gambles.

How do I plan to proceed?
I briefly introduce binary choice models on abstract option spaces in Section~\ref{sec:binary:choice}, and extend the discussion to general---not necessarily binary---choice in Section~\ref{sec:non-binary-choice}.
I rely on results in earlier papers to provide an axiomatisation and a conservative inference framework for these choice models, and recall that there are general theorems that allow for representation of general models in terms of binary ones.
After these introductory sections, I focus on adding {\archimty} to the picture.
The basic representation tools that will turn out to be useful in this more restricted context, namely linear and superlinear bounded real functionals, are discussed in Section~\ref{sec:technical-aspects}.
The more classical approach to {\archimty}~\cite{walley1991} for binary choice---which I'll call {\essarchimty}---is given an abstract treatment in Section~\ref{sec:essential:archimedeanity}.
I discuss the simple trick that allows us to deal elegantly with option spaces without constants---namely proclaiming some option to be constant as far as representation is concerned---in Section~\ref{sec:normalisation}.
Sections~\ref{sec:archimedanity:binary} and~\ref{sec:archimedeanity:non-binary} then deal with the new notion of {\archimty} in the cases of binary, and general, preferences, and discuss conservative inference and the representation of general {\archim} choice models in terms of binary essentially {\archim} ones.
I conclude in Section~\ref{sec:conclusions} by stressing the relevance of my findings: they show that the axioms presented here allow for a complete characterisation of several decision methods in the literature, amongst which Levi's  E-admissibility~\cite{levi1980a} and Walley--Sen maximality~\cite{walley1991}.
The axiomatic characterisation of E-admissibility that I present here as an interesting special case, seems to generalise the one given by Seidenfeld et al.~\cite{seidenfeld2010} from horse lotteries to options that live in general Banach spaces, but I believe the axioms and representation results used here have the advantage of being more elegant, transparent and easily interpretable.

Where convenient, and in order to anchor the discussion somewhat, I'll use a simple running example to shed more concrete light on the more abstract ideas and constructions in the main text.

\section{Coherent sets of desirable options}\label{sec:binary:choice}
We begin by considering a linear space~\(\opts\), whose elements~\(\opt\) are called \emph{options}, and which represent the objects that a subject can choose between.
This \emph{option space}~\(\opts\) has some so-called \emph{background ordering}~\(\optgt\), which is `natural' in that we will assume that our subject's choices will always at least respect this ordering, even before she has started reflecting on her preferences.
\emph{This background ordering~\(\optgt\) is taken to be a strict vector ordering on~\(\opt\)}, so an irreflexive and transitive binary relation that is compatible with the addition and scalar multiplication of options.
We'll use the obvious notation~\(\posopts\) for the convex cone~\(\cset{\opt\in\opts}{\opt\optgt0}\) of all options that are strictly preferred to the zero option~\(0\) with respect to this background ordering.

\begin{runningexample}\label{ex:setting:the:stage}
As a simple running example, I suggest we consider flipping a coin with two sides: heads~\(\heads\) and tails~\(\tails\).
We let \(\states\coloneqq\coinstates\) be the set of possible outcomes of the coin flip, and also consider a reward set~\(\rewards\coloneqq\set{\bestreward,\worstreward}\), with a better reward \(\bestreward\) and a worse reward \(\worstreward\).\footnote{This is the simplest reward set possible. It is definitely feasible to develop this and more complicated examples for larger reward sets, but as this running example has to serve as an illustration whose aim is to provide you with some intuition, I've decided to keep it as simple as possible.}
As already suggested in the Introduction, a horse lottery in this context is a map \(\hl\colon\statesandrewards\to\reals\) such that
\begin{equation*}
\hl(\heads,\bestreward)\geq0
\text{ and }
\hl(\tails,\bestreward)\geq0
\text{ and }
\hl(\heads,\worstreward)\geq0
\text{ and }
\hl(\tails,\worstreward)\geq0
\text{ and }
\hl(\bolleke,\bestreward)+\hl(\bolleke,\worstreward)=1,
\end{equation*}
or in other words, it's a probability mass function on the rewards---a so-called \emph{lottery}---that depends on the outcome of the coin flip.
We collect all such horse lotteries in the set~\(\hls\), which is closed under convex combinations.

We use this convex set~\(\hls\) to construct the following linear subspace of the four-dimensional linear space~\(\gblson{\coinstatesandrewards}\) of all gambles on~\(\statesandrewards=\coinstatesandrewards\):
\begin{equation}\label{eq:define:the:difs}
\difs
\coloneqq\linspan(\hls-\hls)
=\cset{\alpha(\hl-\althl)}{\alpha\in\posreals\text{ and }\hl,\althl\in\hls}.
\end{equation}
Observe that for all~\(\dif\in\difs\):
\begin{equation*}
\dif(\bolleke,\bestreward)+\dif(\bolleke,\worstreward)=0,
\end{equation*}
which tells us that \(\dif\) is two-dimensional, and also that if \(d\) is everywhere equal to some real constant~\(c\), then necessarily \(c+c=0\), or in other words, \(c=0\).
The only constant gamble in the linear space~\(\difs\) is therefore the zero gamble.

The option space~\(\difs\) is isomorphic to the linear space~\(\gblson{\coinstates}\) of all gambles on~\(\coinstates\), using the correspondence~\(\dif\mapsto\dif(\bolleke,\bestreward)=-\dif(\bolleke,\worstreward)\).
I'll also use this correspondence in further instalments of this running example, simply because it will allow me to give more direct graphical illustrations of the various notions to be introduced further on.
The alternative option space~\(\gblson{\coinstates}\) does include all real constants, but it has the disadvantage that it singles out one reward~\(\bestreward\) in the reward set as special something which, as we'll see, can be avoided by working with the option space~\(\difs\).
That such a special reward needs to be singled out becomes even more apparent when we use reward sets~\(\rewards\) with a cardinality higher than two. 
I refer to the relevant papers by Zaffalon and Miranda \cite{zaffalon2017:incomplete:preferences}, and Van Camp \cite{2017vancamp:phdthesis}, for more details on this alternative approach to relating choice models on horse lotteries to choice models on gambles.
\stopit
\end{runningexample}

We'll assume that our subject's binary choices between options can be modelled by a so-called \emph{set of desirable options}~\(\desirset\subseteq\opts\), where an option is called \emph{desirable} when the subject strictly prefers it to the zero option.
We'll denote the set of all possible sets of desirable options---all subsets of~\(\opts\)---by~\(\desirsets\).
Of course, a set of desirable options~\(\desirset\) strictly speaking only covers the strict preferences~\(\prefgt\) between options~\(\opt\) and~\(0\): \(\opt\prefgt0\ifandonlyif\opt\in\desirset\).
For other strict preferences, it is assumed that they're compatible with the vector addition of options: \(\opt\prefgt\altopt\ifandonlyif\opt-\altopt\prefgt0\ifandonlyif\opt-\altopt\in\desirset\).

We impose the following rationality requirements on a subject's strict preferences.
A set of desirable options~\(\desirset\in\desirsets\) is called \emph{coherent}~\cite{walley2000,cooman2010,couso2011} if it satisfies the following axioms:
\begin{enumerate}[label=\(\mathrm{D}_{\arabic*}\).,ref=\(\mathrm{D}_{\arabic*}\),leftmargin=*]
\item\label{ax:desirs:nozero} \(0\notin\desirset\);
\item\label{ax:desirs:cone} if \(\opt,\altopt\in\desirset\) and \((\lambda,\mu)>0\), then \(\lambda\opt+\mu\altopt\in\desirset\);
\item\label{ax:desirs:pos} \(\posopts\subseteq\desirset\).
\end{enumerate}
We'll use the notation~\((\lambda,\mu)>0\) to mean that \(\lambda,\mu\) are non-negative real numbers such that \(\lambda+\mu>0\).
We denote the set of all coherent sets of desirable options by~\(\cohdesirsets\).

\begin{runningexample}\label{ex:from:horse:lotteries:to:sets:of:desirable:options}
Let's now go back to our simple coin example.
We'll consider two particular strict vector orderings~\(\gblgt\) and~\(\gblpgt\) as potential background orderings on the linear space~\(\gblson{\coinstates}\):
\begin{align*}
\gbl\gblgt\altgbl
&\ifandonlyif\group{\forall x\in\states}\gbl(x)\geq\altgbl(x)\text{ and }\group{\exists x\in\states}\gbl(x)>\altgbl(x),\\
\gbl\gblpgt\altgbl
&\ifandonlyif\group{\forall x\in\states}\gbl(x)>\altgbl(x).
\end{align*}
By the way, any strict vector ordering~\(\prefgt\) on~\(\gblson{\coinstates}\) induces an `equivalent' strict vector ordering on~\(\difs\), for which we use the same notation~\(\prefgt\), as follows:
\begin{equation*}
\dif\prefgt\altdif\ifandonlyif\dif(\bolleke,\bestreward)\prefgt\altdif(\bolleke,\bestreward),
\text{ for all \(\dif,\altdif\in\difs\)}.
\end{equation*}
In particular, we have that
\begin{equation}\label{eq:define:the:background:orderings}
\left\{
\begin{aligned}
\dif\in\difs_{\gblgt0}
&\ifandonlyif\dif\gblgt0
\ifandonlyif\dif(\heads,\bestreward)\geq0\text{ and }\dif(\tails,\bestreward)\geq0\text{ and }\dif(\heads,\bestreward)+\dif(\tails,\bestreward)>0\\
\dif\in\difs_{\gblpgt0}
&\ifandonlyif\dif\gblpgt0
\ifandonlyif\dif(\heads,\bestreward)>0\text{ and }\dif(\tails,\bestreward)>0.
\end{aligned}
\right.
\end{equation}

The relationship~\eqref{eq:define:the:difs} between~\(\difs\) and~\(\hls\) provides an intuitive justification for introducing the following relationship between such a strict vector ordering~\(\prefgt\) on~\(\difs\) and a strict partial ordering~\(\hlprefgt\) on~\(\hls\):
\begin{equation}\label{eq:define:the:ordering:on:difs}
\alpha(\hl-\althl)\prefgt0\ifandonlyif\hl\hlprefgt\althl
\text{ for all \(\hl,\althl\in\hls\)}.
\end{equation}
Observe, by the way, that defining a correspondence in this manner makes sense---is consistent---if and only if the strict partial ordering~\(\hlprefgt\) on~\(\hls\) satisfies the so-called \emph{mixture independence} condition \cite{seidenfeld1995,seidenfeld2010,nau2006,aumann1962} for horse lotteries:
\begin{equation}\label{eq:mixture:independence}
\hl\hlprefgt\althl\ifandonlyif\alpha\hl+(1-\alpha)\althltoo\hlprefgt\alpha\althl+(1-\alpha)\althltoo
\text{ for all \(\hl,\althl,\althltoo\in\hls\) and all \(\alpha\in(0,1]\)}.
\end{equation}
The two candidate background orderings~\(\hlsgt\) and~\(\hlpgt\) on~\(\hls\) that correspond to the respective background orderings~\(\gblgt\) and~\(\gblpgt\) on~\(\difs\)---and on~\(\gblson{\coinstates}\)---are then given by:
\begin{equation*}
\left.
\begin{aligned}
\hl\hlsgt\althl\ifandonlyif\hl(\bolleke,\bestreward)\gblgt\althl(\bolleke,\bestreward)&\\
\hl\hlpgt\althl\ifandonlyif\hl(\bolleke,\bestreward)\gblpgt\althl(\bolleke,\bestreward)&
\end{aligned}
\right\}
\text{ for all \(\hl,\althl\in\hls\)}.
\end{equation*}
\par
If we start out with a strict partial order~\(\hlprefgt\) on~\(\hls\) that satisfies the mixture independence condition~\eqref{eq:mixture:independence}, define the corresponding strict preference ordering~\(\prefgt\) on the option space~\(\difs\) using the correspondence~\eqref{eq:define:the:ordering:on:difs}, and then construct the corresponding set of desirable options~\(\desirset\) by letting
\begin{equation*}
\desirset
\coloneqq\cset{\dif\in\difs}{\dif\prefgt0}
=\cset{\alpha(\hl-\althl)}{\alpha\in\posreals\text{ and }\hl\hlprefgt\althl},
\end{equation*}
then this~\(\desirset\) automatically satisfies the coherence requirements~\ref{ax:desirs:nozero} and~\ref{ax:desirs:cone}.
For~\ref{ax:desirs:pos}, it suffices to require that~\({\hlsgt}\subseteq{\hlprefgt}\) or~\({\hlpgt}\subseteq{\hlprefgt}\), depending on the actual choice of the background ordering.
An option \(\dif\in\difs\) is then considered to be desirable by our subject when \(\dif\in\desirset\), or equivalently, when~\(\dif\prefgt0\).
This can also---equivalently---be taken to mean that the subject considers the uncertain reward \(\dif(\bolleke,\bestreward)\), expressed in units (utiles) of a linear utility scale, as strictly preferable to the zero reward.
The gamble, or uncertain reward, \(\dif(\bolleke,\bestreward)\) yields \(\dif(\heads,\bestreward)\) utiles when the coin flip results in~\(\heads\), and \(\dif(\tails,\bestreward)\) when it results in~\(\tails\).

We conclude that working with strict partial orderings on the set of horse lotteries~\(\hls\) that satisfy the mixture independence condition and that respect some background ordering amounts to working with coherent sets of desirable options in the option space~\(\difs\).
This idea was already explored, in a more generic context, by Van Camp \cite{2017vancamp:phdthesis}. 

In this running example, I'll therefore shift the focus towards modelling choice using the latter model.
In Figure~\ref{fig:some:heads:and:tails}, I depict a few coherent sets of desirable options in the option space~\(\difs\), with background orderings~\(\gblgt\) and~\(\gblpgt\), using the representations of their elements~\(\dif\) as uncertain rewards~\(\dif(\bolleke,\bestreward)\)---or gambles---in~\(\gbls{\coinstates}\).
\stopit
\end{runningexample}

\begin{figure}[h]
\centering
\begin{tikzpicture}[scale=.285]\footnotesize
\fill[blue!50] (0,4) -- (4,4) -- (4,0) -- (0,0) -- cycle;
\draw[gray,->] (-2.2,0) -- (4.2,0) node[below right] {\(\heads\)};
\draw[gray,->] (0,-2.2) -- (0,4.2) node[above left] {\(\tails\)};
\draw[blue,semithick] (0,0) -- (0,4);
\draw[blue,semithick] (0,0) -- (4,0);
\node[draw=blue,fill=white,circle,inner sep=1pt] at (0,0) {};
\node[white] at (2,2) {\(\difs_{\gblgt0}\)};
\end{tikzpicture}
\quad
\begin{tikzpicture}[scale=.285]\footnotesize
\fill[blue!50] (0,4) -- (4,4) -- (4,0) -- (0,0) -- cycle;
\draw[gray,->] (-2.2,0) -- (4.2,0) node[below right] {\(\heads\)};
\draw[gray,->] (0,-2.2) -- (0,4.2) node[above left] {\(\tails\)};
\draw[blue,densely dotted,thick] (0,0) -- (4,0);
\draw[blue,densely dotted,thick] (0,0) -- (0,4);
\node[draw=blue,fill=white,circle,inner sep=1pt] at (0,0) {};
\node[white] at (2,2) {\(\difs_{\gblpgt0}\)};
\end{tikzpicture}
\quad
\begin{tikzpicture}[scale=.285]\footnotesize
\fill[blue!50] (4,4) -- (4,-1) -- (0,0) -- (-2,4) -- cycle;
\draw[gray,->] (-2.2,0) -- (4.2,0) node[below right] {\(\heads\)};
\draw[gray,->] (0,-2.2) -- (0,4.2) node[above left] {\(\tails\)};
\draw[blue,semithick] (0,0) -- (-2,4);
\draw[blue,semithick] (0,0) -- (4,-1);
\node[draw=blue,fill=white,circle,inner sep=1pt] at (0,0) {};
\end{tikzpicture}
\quad
\begin{tikzpicture}[scale=.285]\footnotesize
\fill[blue!50] (4,4) -- (4,-2) -- (0,0) -- (-1,4) -- cycle;
\draw[gray,->] (-2.2,0) -- (4.2,0) node[below right] {\(\heads\)};
\draw[gray,->] (0,-2.2) -- (0,4.2) node[above left] {\(\tails\)};
\draw[blue,densely dotted,thick] (0,0) -- (-1,4);
\draw[blue,densely dotted,thick] (0,0) -- (4,-2);
\node[draw=blue,fill=white,circle,inner sep=1pt] at (0,0) {};
\end{tikzpicture}
\quad
\begin{tikzpicture}[scale=.285]\footnotesize
\fill[blue!50] (4,4) -- (4,-2) -- (1,-2) -- (0,0) -- (-2,4) -- cycle;
\draw[gray,->] (-2.2,0) -- (4.2,0) node[below right] {\(\heads\)};
\draw[gray,->] (0,-2.2) -- (0,4.2) node[above left] {\(\tails\)};
\draw[blue,densely dotted,thick] (0,0) -- (-2,4);
\draw[blue,semithick] (0,0) -- (1,-2);
\node[draw=blue,fill=white,circle,inner sep=1pt] at (0,0) {};
\end{tikzpicture}
\caption{Representations (in blue) of a number of sets of desirable options that are coherent under both background orderings~\(\gblgt\) and~\(\gblpgt\). The two plots on the left represent the `positive' cones~\(\difs_{\gblgt0}\) and~\(\difs_{\gblpgt0}\) of options that are strictly preferred to the zero option under the respective background orderings~\(\gblgt\) and~\(\gblpgt\). The rightmost plot represents a total set of desirable options.
Each of these sets of desirable options~\(\desirset\) is graphically represented by the values~\(\dif(\bolleke,\bestreward)\) that its elements~\(\dif\in\desirset\) assume in the reward~\(\bestreward\). Full blue lines indicate `borders' that are included in the set, dotted lines represent `borders' that aren't.}
\label{fig:some:heads:and:tails}
\end{figure}
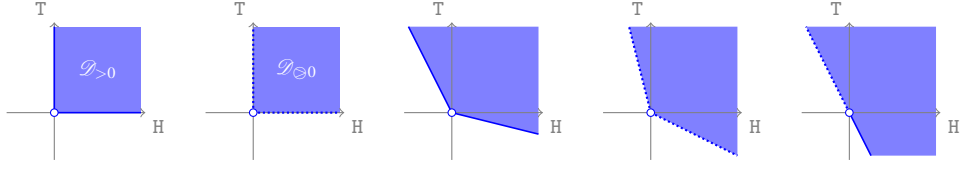

It is easy to see that the set~\(\cohdesirsets\) of all coherent sets of desirable options is an intersection structure: for any non-empty family of sets of desirable options~\(\desirset[i]\in\cohdesirsets\), \(i\in I\), its intersection~\(\bigcap_{i\in I}\desirset[i]\) also belongs to~\(\cohdesirsets\).
This also implies that we can introduce a \emph{coherent closure} operator~\(\natexdesirset\colon\desirsets\to\cohdesirsets\cup\set{\opts}\) by letting
\begin{equation*}
\natexdesirset(\assessment)
\coloneqq\bigcap\cset{\desirset\in\cohdesirsets}{\assessment\subseteq\desirset}
\text{ for all \(\assessment\subseteq\opts\)}
\end{equation*}
be the smallest---if any---coherent set of desirable options that includes~\(\assessment\).
Such an~\(\assessment\) typically represents a so-called \emph{assessment}: a not necessarily exhaustive collection of options that a subject states to be desirable.

We call such an assessment~\(\assessment\subseteq\opts\) \emph{consistent} if \(\natexdesirset(\assessment)\neq\opts\), or equivalently, if \(\assessment\) is included in some coherent set of desirable options.
The closure operator~\(\natexdesirset\) implements \emph{conservative inference} with respect to the coherence axioms, in that it extends a consistent assessment~\(\assessment\) to the most conservative---smallest possible---coherent set of desirable options~\(\natexdesirset(\assessment)\) that includes it.

It is clear from this discussion that the so-called \emph{vacuous set of desirable options}~\(\posopts\) is the smallest element of \(\cohdesirsets\) with respect to set inclusion, and that therefore \(\natexdesirset(\emptyset)=\posopts\): it corresponds to making no assessment at all, whence, of course, its name.

A coherent set of desirable options~\(\maxdesirset\) is called \emph{maximal} if none of its strict supersets is coherent: \((\forall\desirset\in\cohdesirsets)(\maxdesirset\subseteq\desirset\then\maxdesirset=\desirset)\).
This turns out to be equivalent to the following so-called \emph{totality} condition on~\(\maxdesirset\)~\cite{cooman2010,couso2011}:
\begin{enumerate}[label=\(\mathrm{D}_{\mathrm{T}}\).,ref=\(\mathrm{D}_{\mathrm{T}}\),leftmargin=*]
\item\label{ax:desirs:totality} for all~\(\opt\in\opts\setminus\set{0}\), either~\(\opt\in\maxdesirset\) or~\(-\opt\in\maxdesirset\).
\end{enumerate}
The set of all maximal sets of desirable options is denoted by~\(\totdesirsets\).
These maximal elements can be used to represent all coherent sets of desirable options via intersection.

\begin{theorem}[Closure {\protect\cite{cooman2010}}]\label{theo:coherent:representation:desirsets}
Consider any~\(\desirset\in\desirsets\), then \(\natexdesirset(\desirset)=\bigcap\cset{\maxdesirset\in\totdesirsets}{\desirset\subseteq\maxdesirset}\).
Hence, \(\desirset\) is consistent if and only if \(\cset{\maxdesirset\in\totdesirsets}{\desirset\subseteq\maxdesirset}\neq\emptyset\).
And a consistent~\(\desirset\) is coherent if and only if \(\desirset=\bigcap\cset{\maxdesirset\in\totdesirsets}{\desirset\subseteq\maxdesirset}\).
\end{theorem}

\begin{corollary}[Representation]\label{cor:coherent:representation:desirsets}
A set of desirable options~\(\desirset\in\desirsets\) is coherent if and only if there is some non-empty~\(\setofdesirsets\subseteq\totdesirsets\) such that \(\desirset=\bigcap\cset{\maxdesirset}{\maxdesirset\in\setofdesirsets}\).
In that case, the largest such set~\(\setofdesirsets\) is \(\cset{\maxdesirset\in\totdesirsets}{\desirset\subseteq\maxdesirset}\).
\end{corollary}
\noindent
For more details on these issues, and more `constructive' expressions for~\(\natexdesirset\), see~\cite{cooman2010,debock2019:choice:isipta,debock2019:choice:arxiv}.

I also want to mention another, additional, rationality property, central in Teddy Seidenfeld's work~\cite{seidenfeld1995,seidenfeld2010}, but introduced there in a form more appropriate for strict preferences between horse lotteries.
We can get to the appropriate counterpart here when we introduce the \(\posi\group{\bolleke}\) operator, which, for any subset~\(V\) of~\(\opts\), returns the set of all positive linear combinations of its elements:
\begin{equation*}%\label{eq:posioperator}
\posi(V)\coloneqq\cset[\bigg]{\sum_{k=1}^n\lambda_k\opt[k]}{n\in\naturals,\lambda_k\in\posreals,\opt[k]\in V}.
\end{equation*}
We call a set of desirable options~\(\desirset\in\desirsets\) \emph{mixing} if it's coherent and satisfies the following mixingness axiom~\cite{debock2019:choice:isipta,debock2019:choice:arxiv}:\footnote{The term `mixing' goes back to Seidenfeld's work~\cite{seidenfeld1995,seidenfeld2010} and is appropriate when the focus is on choosing between options. As is apparent from Axiom~\ref{ax:desirs:mixing}, in our context---which focuses on rejecting options rather than choosing them (see the discussion following Theorem~\ref{theo:mixingrepresentation:twosided} below)---the term `unmixing' would be more suitable, as the property it describes allows us to infer something about the desirability of unmixed options from the desirability of some of their mixtures.}
\begin{enumerate}[label=\(\mathrm{D}_{\mathrm{M}}\).,ref=\(\mathrm{D}_{\mathrm{M}}\),leftmargin=*]
\item\label{ax:desirs:mixing} for all finite subsets~\(\optset\) of~\(\opts\), if \(\posi(\optset)\cap\desirset\neq\emptyset\), then also \(\optset\cap\desirset\neq\emptyset\).
\end{enumerate}
We denote the set of all mixing sets of desirable options by~\(\mixdesirsets\).
It does not necessarily constitute an intersection structure, and therefore does not come with a simple conservative inference apparatus.
Its elements can be characterised as follows.

\begin{proposition}[{\protect\cite{2017vancamp:phdthesis,2018vancamp:lexicographic}}]\label{prop:mixingequalslexicographic}
Consider any set of desirable options~\(\desirset\in\cohdesirsets\) and let \(\co{\desirset}\coloneqq\opts\setminus\desirset\).
Then \(\desirset\) is mixing if and only if\/ \(\posi(\co{\desirset})=\co{\desirset}\), or equivalently, \(\desirset\cap\posi(\co{\desirset})=\emptyset\).
\end{proposition}

When the option space is a set of gambles on a finite state space, provided with the point-wise partial ordering and the corresponding strict vector ordering as a background ordering, these mixing sets of desirable options correspond to the so-called \emph{lexicographic} sets of desirable options introduced by Van Camp et al.~\cite{2017vancamp:phdthesis,2018vancamp:lexicographic}.
The name derives from the fact that they can be associated with lexicographic probability orderings; for more details, see also~\cite{2017vancamp:phdthesis,2018vancamp:lexicographic,debock2019:choice:isipta,debock2019:choice:arxiv}.

\begin{runningexample}\label{ex:assessment:identical:sides}
In our coin example, we consider an assessment
\begin{equation*}
\assessment_\heads\coloneqq\cset[\big]{(\indset{\heads}-\alpha\indset{\tails})\unitdif}{\alpha\in\nonnegreals}
\subseteq\difs,
\end{equation*}
where we consider the map~\(\unitdif\colon\coinstatesandrewards\to\reals\) with \(\unitdif(\bolleke,\bestreward)=-\unitdif(\bolleke,\worstreward)\coloneqq1\), and the so-called \emph{indicators}~\(\indset{\heads}\) and~\(\indset{\tails}\) defined by
\begin{equation*}
\indset{\heads}(x)\coloneqq
\begin{cases}
1&\text{ if \(x=\heads\)}\\
0&\text{ if \(x=\tails\)}
\end{cases}
\quad\text{and}\quad
\indset{\tails}(x)\coloneqq
\begin{cases}
1&\text{ if \(x=\tails\)}\\
0&\text{ if \(x=\heads\)}.
\end{cases}
\end{equation*}
Check that \(\unitdif\in\difs_{\gblgt0}\) and \(\unitdif\in\difs_{\gblpgt0}\).
As indicated above, if our subject makes this assessment, this means that she will accept gambles of the type~\(\indset{\heads}-\alpha\indset{\tails}\) on the outcome of the coin flip, or in other words, accept bets on the outcome being heads~\(\heads\) at all odds \(\alpha\in\nonnegreals\) against, expressing that she is \emph{practically certain} \cite{walley1991} that the outcome will be heads.
This could correspond to her knowing that the coin has two equal sides, both heads.
Similarly, the assessment
\begin{equation*}
\assessment_\tails\coloneqq\cset[\big]{(\indset{\tails}-\alpha\indset{\heads})\unitdif}{\alpha\in\nonnegreals}
\subseteq\difs
\end{equation*}
expresses that our subject is practically certain that the outcome will be tails, which could derive from her knowing that the coin has two equal sides, both tails.

It is not difficult to see that both assessments are consistent, under each of the background orderings~\(\gblgt\) and~\(\gblpgt\).
The corresponding coherent closures \(\desirset[\heads]\coloneqq\natexdesirset\group{\assessment_\heads}\) and \(\desirset[\tails]\coloneqq\natexdesirset\group{\assessment_\tails}\) with respect to the background orderings~\(\gblgt\) and~\(\gblpgt\) are depicted in Figure~\ref{fig:twice:heads:and:tails}.
They are the smallest convex cones that include the assessments~\(\assessment_\heads\) and~\(\assessment_\tails\), and the background cones~\(\difs_{\gblgt0}\) and~\(\difs_{\gblpgt0}\), respectively.
These coherent sets of desirable options are clearly also mixing, because their complements are convex cones; see Proposition~\ref{prop:mixingequalslexicographic}.
\stopit
\end{runningexample}

\begin{figure}[h]
\centering
\begin{tikzpicture}[scale=.295]\footnotesize
\fill[blue!50] (0,4) -- (4,4) -- (4,-4) -- (0,-4) -- cycle;
\draw[gray,->] (-4.2,0) -- (4.2,0) node[below right] {\(\heads\)};
\draw[gray,->] (0,-4.2) -- (0,4.2) node[above left] {\(\tails\)};
\draw[blue,semithick] (0,0) -- (0,4);
\draw[blue,densely dotted,thick] (0,0) -- (0,-4);
\node[draw=blue,fill=white,circle,inner sep=1pt] at (0,0) {};
\draw[red,thick] (1,0) -- node[midway,right] {\(\assessment_\heads\)} (1,-4);
\node[fill=red,circle,inner sep=1pt] at (1,0) {};
\node[white] at (2,2) {\(\desirset[\heads]\)};
% \node at (-2,-2) {(a)};
\end{tikzpicture}
\quad
\begin{tikzpicture}[scale=.295]\footnotesize
\fill[blue!50] (4,0) -- (4,4) -- (-4,4) -- (-4,0) -- cycle;
\draw[gray,->] (-4.2,0) -- (4.2,0) node[below right] {\(\heads\)};
\draw[gray,->] (0,-4.2) -- (0,4.2) node[above left] {\(\tails\)};
\draw[blue,semithick] (0,0) -- (4,0);
\draw[blue,densely dotted,thick] (0,0) -- (-4,0);
\node[draw=blue,fill=white,circle,inner sep=1pt] at (0,0) {};
\draw[red,thick] (0,1) -- node[midway,above] {\(\assessment_\tails\)} (-4,1);
\node[fill=red,circle,inner sep=1pt] at (0,1) {};
\node[white] at (2,2) {\(\desirset[\tails]\)};
% \node at (-2,-2) {(b)};
\end{tikzpicture}
\quad
\begin{tikzpicture}[scale=.295]\footnotesize
\fill[blue!50] (0,4) -- (4,4) -- (4,-4) -- (0,-4) -- cycle;
\draw[gray,->] (-4.2,0) -- (4.2,0) node[below right] {\(\heads\)};
\draw[gray,->] (0,-4.2) -- (0,4.2) node[above left] {\(\tails\)};
\draw[blue,densely dotted,thick] (0,0) -- (0,4);
\draw[blue,densely dotted,thick] (0,0) -- (0,-4);
\node[draw=blue,fill=white,circle,inner sep=1pt] at (0,0) {};
\draw[red,thick] (1,0) -- node[midway,right] {\(\assessment_\heads\)} (1,-4);
\node[fill=red,circle,inner sep=1pt] at (1,0) {};
\node[white] at (2,2) {\(\desirset[\heads]\)};
% \node at (-2,-2) {(c)};
\end{tikzpicture}
\quad
\begin{tikzpicture}[scale=.295]\footnotesize
\fill[blue!50] (4,0) -- (4,4) -- (-4,4) -- (-4,0) -- cycle;
\draw[gray,->] (-4.2,0) -- (4.2,0) node[below right] {\(\heads\)};
\draw[gray,->] (0,-4.2) -- (0,4.2) node[above left] {\(\tails\)};
\draw[blue,densely dotted,thick] (0,0) -- (4,0);
\draw[blue,densely dotted,thick] (0,0) -- (-4,0);
\node[draw=blue,fill=white,circle,inner sep=1pt] at (0,0) {};
\draw[red,thick] (0,1) -- node[midway,above] {\(\assessment_\tails\)} (-4,1);
\node[fill=red,circle,inner sep=1pt] at (0,1) {};
\node[white] at (2,2) {\(\desirset[\tails]\)};
% \node at (-2,-2) {(d)};
\end{tikzpicture}
\caption{Representations (in blue) of the coherent sets of desirable options~\(\desirset[\heads]\) and~\(\desirset[\tails]\) under the background ordering~\(\gblgt\) in the leftmost two plots, and the background ordering~\(\gblpgt\) in the rightmost two. Each set of desirable options~\(\desirset\) is graphically represented by the values~\(\dif(\bolleke,\bestreward)\) that its elements~\(\dif\in\desirset\) assume in the reward~\(\bestreward\). Full blue lines indicate `borders' that are included in the set, and dotted lines represent `borders' that aren't.}
\label{fig:twice:heads:and:tails}
\end{figure}

\begin{runningexample}\label{ex:assessment:lower:upper:betting:rates}
We also consider another type of assessment in our coin example:
\begin{equation*}
\assessment_I
\coloneqq\cset[\big]{(\indset{\heads}-\lp+\epsilon)\unitdif}{\epsilon\in\posreals}
\cup\cset[\big]{(\up+\epsilon-\indset{\heads})\unitdif}{\epsilon\in\posreals},
\end{equation*}
where \(0\leq\lp\leq\up\leq1\) and we let \(I\coloneqq\lup\).
Our subject uses this assessment to express that \(\lp\) is her supremum acceptable rate for betting on heads~\(\heads\), and \(1-\up\) her supremum acceptable rate for betting on tails~\(\tails\).
As a particular special case, for any \(p\in\unit\), the assessment
\begin{equation*}
\assessment_p
\coloneqq\assessment_{[p,p]}
=\cset[\big]{[\pm(\indset{\heads}-p)+\epsilon]\unitdif}{\epsilon\in\posreals},
\end{equation*}
expresses that \(p\) is our subject's fair (two-sided) betting rate for heads~\(\heads\).
We also define the maps \(\ex[p],\lex[I]\colon\gblson{\coinstates}\to\reals\) by
\begin{equation}\label{eq:coin:precise:expectation}
\ex[p](f)\coloneqq pf(\heads)+(1-p)f(\tails)
\text{ for all \(f\in\gblson{\coinstates}\)},
\end{equation}
and
\begin{equation}\label{eq:coin:imprecise:expectation}
\lex[I](f)
\coloneqq\min_{p\in I}\ex[p](f)
=\begin{cases}
\ex[\lp](f)&\text{ if \(f(\heads)\geq f(\tails)\)}\\
\ex[\up](f)&\text{ if \(f(\heads)\leq f(\tails)\)}
\end{cases}
\text{ for all \(f\in\gblson{\coinstates}\)}.
\end{equation}

Let's first consider the case that \(0<\lp\leq\up<1\).
It is then fairly easy to see that for the coherent closure~\(\desirset[I]\coloneqq\natexdesirset(\assessment_I)\) of the assessment~\(\assessment_I\),
\begin{equation}\label{eq:coin:imprecise:desirset}
\dif\in\desirset[I]\ifandonlyif\lex[I]\group{\dif(\bolleke,\bestreward)}>0
\text{ for all \(\dif\in\difs\)},
\end{equation}
under both background orderings~\(\gblgt\) and~\(\gblpgt\), which also tells us that the assessment~\(\assessment_I\) is consistent.
See also Figure~\ref{fig:imprecise:heads:and:tails} for a graphical representation of these facts.
It is also clear from these graphs that \(\desirset[I]\) is mixing, or in other words, by Proposition~\ref{prop:mixingequalslexicographic}, that \(\co{\desirset[I]}\) is a convex cone, if and only if \(\lp=\up\eqqcolon p\).
In that case, we'll use the notation \(\desirset[p]\coloneqq\desirset[I]\).

As soon as \(\lp=0\) or \(\up=1\), the respective borders~\(\cset{\dif\in\difs}{\dif(\tails,\bestreward)=0\text{ and }\dif(\heads,\bestreward)>0}\) and~\(\cset{\dif\in\difs}{\dif(\heads,\bestreward)=0\text{ and }\dif(\tails,\bestreward)>0}\) are included in~\(\desirset[I]\) under the background ordering~\(\gblgt\), and excluded under the background ordering~\(\gblpgt\).
In particular, it holds that~\(\desirset[1]=\desirset[\heads]\) and~\(\desirset[0]=\desirset[\tails]\); see also Figure~\ref{fig:twice:heads:and:tails}.
\stopit
\end{runningexample}

\begin{figure}[h]
\centering
\begin{tikzpicture}[scale=.5]\footnotesize
\fill[blue!50] (0,0) -- (4,-2) -- (4,4) -- (-1,4) -- cycle;
\draw[gray,->] (-2.2,0) -- (4.2,0) node[below right] {\(\heads\)};
\draw[gray,->] (0,-2.2) -- (0,4.2) node[above left] {\(\tails\)};
\draw[blue,densely dotted,thick] (0,0) -- (4,-2) node[below] {\(\ex[\lp]=0\)};
\draw[blue,dashed] (0,0) -- (-2,1);
\draw[blue,densely dotted,thick] (0,0) -- (-1,4) node[left] {\(\ex[\up]=0\)};
\draw[blue,dashed] (0,0) -- (.5,-2);
\node[draw=blue,fill=white,circle,inner sep=1pt] at (0,0) {};
\draw[red,thick] (4/3,-2/3) -- (2,0) -- (4,2);
\draw[red,thick] (-.4,1.6) -- (0,2) -- (2,4);
\node[fill=white,draw=red,circle,inner sep=1pt] at (4/3,-2/3) {};
\node[red,fill=white,below left] at (4/3,-2/3) {\(\indset{\heads}-\lp\)};
\node[fill=white,draw=red,circle,inner sep=1pt] at (-.4,1.6) {};
\node[red,fill=white,left=2pt] at (-.4,1.6) {\(\up-\indset{\heads}\)};
\node[red] at (1.5,1.5) {\(\assessment_I\)};
\node[white] at (3,3) {\(\desirset[I]\)};
\end{tikzpicture}
\quad
\begin{tikzpicture}[scale=.5]\footnotesize
\fill[blue!50] (0,0) -- (1,-2) -- (4,-2) -- (4,4) -- (-2,4) -- cycle;
\draw[gray,->] (-2.2,0) -- (4.2,0) node[below right] {\(\heads\)};
\draw[gray,->] (0,-2.2) -- (0,4.2) node[above left] {\(\tails\)};
\draw[blue,densely dotted,thick] (0,0) -- (-2,4) node[left] {\(\ex[p]=0\)};
\draw[blue,densely dotted,thick] (0,0) -- (1,-2) node[below] {\(\ex[p]=0\)};
\node[draw=blue,fill=white,circle,inner sep=1pt] at (0,0) {};
\draw[red,thick] (2/3,-4/3) -- (2,0) -- (4,2);
\draw[red,thick] (-2/3,4/3) -- (0,2) -- (2,4);
\node[fill=white,draw=red,circle,inner sep=1pt] at (2/3,-4/3) {};
\node[red,fill=white,left=3pt] at (2/3,-4/3) {\(\indset{\heads}-p\)};
\node[fill=white,draw=red,circle,inner sep=1pt] at (-2/3,4/3) {};
\node[red,fill=white,left=3pt] at (-2/3,4/3) {\(p-\indset{\heads}\)};
\node[red] at (1.5,1.5) {\(\assessment_p\)};
\node[white] at (3,3) {\(\desirset[p]\)};
\end{tikzpicture}
\caption{Representations (in blue) of the coherent sets of desirable options~\(\desirset[I]\) (leftmost plot) and ~\(\desirset[p]\) (rightmost plot), corresponding with the respective assessments~\(\assessment_I\) and~\(\assessment_p\) (in red), under both background orderings~\(\gblgt\) and~\(\gblpgt\). Each set of desirable options~\(\desirset\) is graphically represented by the values~\(\dif(\bolleke,\bestreward)\) that its elements~\(\dif\in\desirset\) assume in the reward~\(\bestreward\). Dotted blue lines represent `borders' that aren't included in the set.}
\label{fig:imprecise:heads:and:tails}
\end{figure}

\section{Coherent sets of desirable option sets}\label{sec:non-binary-choice}
We now turn from strict \emph{binary} preferences---of one option~\(\opt\) over another option~\(\altopt\)---to more general preferences that are not necessarily binary.
The simplest way to introduce these more general choice models in the present context goes as follows.
We call any finite subset~\(\optset\) of~\(\opts\) an \emph{option set}, and we collect all such option sets into the set~\(\optsets\).
We call an option set~\(\optset\) \emph{desirable} to a subject if she assesses that \emph{at least one option in~\(\optset\) is desirable}, meaning that it is strictly preferred to~\(0\).
We collect a subject's desirable option sets into her \emph{set of desirable option sets~\(\rejectset\)}.
We denote the set of all such possible sets of desirable option sets---all subsets of~\(\optsets\)---by~\(\rejectsets\).

Clearly, the option set \(\set{\opt}\) is desirable if and only if the option~\(\opt\) is.
On the other hand, stating that, say, the option set \(\set{\opt,\altopt}\) is desirable amounts to stating that at least one of the options~\(\opt\) and~\(\altopt\) is desirable, without specifying further which of these two actually is.
In other words, a desirability statement for option sets amounts to an OR of desirability statements for its elements:
\begin{center}
`\(\set{\opt,\altopt}\) is desirable' \(\ifandonlyif\) `\(\opt\) is desirable' OR  `\(\altopt\) is desirable'.
\end{center}

\begin{runningexample}\label{ex:identical:sides:nonbinary:assessment}
In our coin flipping example, we've seen a desirability model for a our subject's beliefs that the coin has two identical sides, namely heads~\(\heads\): the coherent set of desirable options~\(\desirset[\heads]\) based on the assessment~\(\assessment_\heads\).
Similarly, we found a coherent set of desirable options~\(\desirset[\tails]\) based on the assessment~\(\assessment_\tails\), representing beliefs that both sides of the coin are tails~\(\tails\).
But what about representing the beliefs that the coin has two identical sides, without knowing whether they are heads or tails?
This amounts to OR-ing the desirability statements present in the assessments~\(\assessment_\heads\) and~\(\assessment_\tails\),
which can't be represented using a coherent set of desirable options, as such a set is, in effect, only an AND of simple desirability statements.

Simple-minded attempts to force this assessment into a desirability framework lead to failure.
For instance, the assessment \(\assessment_\heads\cup\assessment_\tails\) is inconsistent: it isn't included in any coherent set of desirable options, as a quick look at Figure~\ref{fig:twice:heads:and:tails} will tell us immediately.
This is not surprising, as this assessment represents the belief that the coin has two identical sides that are \emph{both} heads and tails.
On the other hand, the assessment \(\assessment_\heads\cap\assessment_\tails\) is empty, and therefore leads to the \emph{vacuous} set of desirable options~\(\natexdesirset(\emptyset)=\difs_{\gblgt0}\) or \(\natexdesirset(\emptyset)=\difs_{\gblpgt0}\), depending on the choice of the background ordering.

This failure reflects the simple fact that the choice model of coherent sets of desirable options does not allow us to deal with OR-ing desirability statements.
But moving to sets of desirable option \emph{sets} provides a way out, as we will see further on.
We are thus led to considering a desirable option set assessment of the following type:
\begin{multline*}
\assessment_{\heads\,\text{or}\,\tails}
\coloneqq\cset[\big]{\set{\opt,\altopt}}{\opt\in\assessment_\heads\text{ and }\altopt\in\assessment_\tails}\\
=\cset[\Big]{\set[\big]{(\indset{\heads}-\alpha\indset{\tails})\unitdif,(\indset{\tails}-\beta\indset{\heads})\unitdif}}
{\alpha,\beta\in\nonnegreals}
\end{multline*}
and we'll need to come up with rationality requirements and an accompanying method of (conservative) inference in order to find out what this assessment implies.
\stopit
\end{runningexample}

The rationality requirements we will impose on sets of desirable option sets turn out to be fairly natural generalisations of those for sets of desirable options.
A set of desirable option sets~\(\rejectset\subseteq\optsets\) is called \emph{coherent}~\cite{debock2019:choice:isipta,debock2019:choice:arxiv} if it satisfies the following axioms:
\begin{enumerate}[label=\(\mathrm{K}_{\arabic*}\).,ref=\(\mathrm{K}_{\arabic*}\),leftmargin=*,start=0]
\item\label{ax:rejects:removezero} if \(\optset\in\rejectset\) then also \(\optset\setminus\set{0}\in\rejectset\), for all~\(\optset\in\optsets\);
\item\label{ax:rejects:nozero} \(\set{0}\notin\rejectset\);
\item\label{ax:rejects:cone} if \(\optset[1],\optset[2]\in\rejectset\) and if, for all \(\opt\in\optset[1]\) and \(\altopt\in\optset[2]\), \((\lambda_{\opt,\altopt},\mu_{\opt,\altopt})>0\), then also
\begin{equation*}
\cset{\lambda_{\opt,\altopt}\opt+\mu_{\opt,\altopt}\altopt}{\opt\in\optset[1],\altopt\in\optset[2]}
\in\rejectset;
\end{equation*}
\item\label{ax:rejects:mono} if \(\optset[1]\in\rejectset\) and \(\optset[1]\subseteq\optset[2]\), then also \(\optset[2]\in\rejectset\), for all~\(\optset[1],\optset[2]\in\optsets\);
\item\label{ax:rejects:pos} \(\set{\opt}\in\rejectset\), for all~\(\opt\in\posopts\).
\end{enumerate}
We denote the set of all coherent sets of desirable option sets by~\(\cohrejectsets\).

I refer to~\cite{debock2019:choice:isipta,debock2019:choice:arxiv} for a detailed justification of these axioms, starting from the desirability axioms and the above-mentioned interpretation of the desirability of an option set as an OR of desirability statements for its elements.

A coherent set of desirable option sets~\(\rejectset\) contains singletons, doubletons, and so on.
Moreover, it also contains all supersets of any of its elements, by Axiom~\ref{ax:rejects:mono}.
The singletons in~\(\rejectset\) represent the binary choices, or in other words, the pure desirability aspects.
We let
\begin{equation}\label{eq:rejectset:to:desirset}
\desirset[\rejectset]
\coloneqq\cset{\opt\in\opts}{\set{\opt}\in\rejectset}
\end{equation}
be the set of desirable options that represents the binary choices present in the model~\(\rejectset\).
Its elements are the options that---according to~\(\rejectset\)---are definitely desirable.
But there may be elements~\(\optset\) of~\(\rejectset\) of higher cardinality than one that are minimal in the sense that \(\rejectset\) has none of their strict subsets.
This means that our subject holds that at least one option in~\(\optset\) is desirable, but her model holds no more specific information about which of these options actually are desirable.
This indicates that the choice model~\(\rejectset\) has non-binary aspects.
If such is not the case, or in other words, if every element of~\(\rejectset\) goes back to some singleton in~\(\rejectset\), meaning that
\begin{equation*}
(\forall\optset\in\rejectset)
(\exists\opt\in\optset)
\set{\opt}\in\rejectset,
\end{equation*}
then we call the choice model~\(\rejectset\) \emph{binary}.
With any~\(\desirset\in\desirsets\), our interpretation inspires us to associate a set of desirable option sets~\(\rejectset[\desirset]\), defined by
\begin{equation}\label{eq:desirset:to:rejectset}
\rejectset[\desirset]
\coloneqq\cset{\optset\in\optsets}{\optset\cap\desirset\neq\emptyset}.
\end{equation}
It turns out that a set of desirable option sets~\(\rejectset\) is binary if and only if it has the form~\(\rejectset[\desirset]\), and the unique \emph{representing}~\(\desirset\) is then given by~\(\desirset[\rejectset]\).

\begin{proposition}[{\protect\cite{debock2019:choice:arxiv,debock2019:choice:isipta}}]\label{prop:binary:characterisation}
A set of desirable option sets~\(\rejectset\in\rejectsets\) is binary if and only if there is some~\(\desirset\in\desirsets\) such that \(\rejectset=\rejectset[\desirset]\).
This~\(\desirset\) is then necessarily unique, and equal to~\(\desirset[\rejectset]\).
\end{proposition}
\noindent
The coherence of a binary set of desirable option sets is completely determined by the coherence of its corresponding set of desirable options, which is an indication that our way of generalising coherence from sets of desirable options to sets of desirable option sets has some merits.
We'll see shortly that it has many more.

\begin{proposition}[{\protect\cite{debock2019:choice:arxiv,debock2019:choice:isipta}}]\label{prop:coherence:for:binary}
Consider any binary set of desirable option sets~\(\rejectset\in\rejectsets\) and let \(\desirset[\rejectset]\in\desirsets\) be its corresponding set of desirable options.
Then \(\rejectset\) is coherent if and only if \(\desirset[\rejectset]\)~is.
Conversely, consider any set of desirable options~\(\desirset\in\desirsets\) and let \(\rejectset[\desirset]\) be its corresponding binary set of desirable option sets, then \(\rejectset[\desirset]\) is coherent if and only if \(\desirset\) is.
\end{proposition}
\noindent
The map \(\rejectset[\bolleke]\colon\desirsets\to\rejectsets\colon\desirset\mapsto\rejectset[\desirset]\) is \emph{order preserving} in that \(\desirset[1]\subseteq\desirset[2]\) implies that \(\rejectset[{\desirset[1]}]\subseteq\rejectset[{\desirset[2]}]\).
In fact, Proposition~\ref{prop:coherence:for:binary} guarantees that it is an order isomorphism between the partially ordered sets~\(\structure{\cohdesirsets,\subseteq}\) and \(\structure{\cset{\rejectset[\desirset]}{\desirset\in\cohdesirsets},\subseteq}\).
It is therefore an \emph{order embedding} of \(\structure{\cohdesirsets,\subseteq}\) into \(\structure{\cohrejectsets,\subseteq}\), but, and this is very important, \emph{it fails to preserve meets}---intersections---between these partially ordered sets.
This failure is a happy one, because it leaves enough room for the representation result in Theorem~\ref{theo:coherent:representation:twosided} below.

We find that the binary coherent sets of desirable option sets are given by~\(\cset{\rejectset[\desirset]}{\desirset\in\cohdesirsets}\), allowing us to call any coherent set of desirable option sets in~\(\cohrejectsets\setminus\cset{\rejectset[\desirset]}{\desirset\in\cohdesirsets}\) \emph{non-binary}.
If we replace such a non-binary coherent set of desirable option sets~\(\rejectset\) by its corresponding set of desirable options~\(\desirset[\rejectset]\), we lose information, because then necessarily \(\rejectset[{\desirset[\rejectset]}]\subset\rejectset\).
\emph{Sets of desirable option sets are therefore, generally speaking, more expressive than sets of desirable options.}
But our coherence axioms lead to a representation result that allows us to still use sets of desirable options, or rather, sets of them, to completely characterise \emph{any} coherent choice model.

\begin{theorem}[Representation {\protect\cite{debock2019:choice:arxiv,debock2019:choice:isipta}}]\label{theo:coherent:representation:twosided}
A set of desirable option sets~\(\rejectset\in\rejectsets\) is coherent if and only if there is some non-empty set~\(\setofdesirsets\subseteq\cohdesirsets\) of coherent sets of desirable options such that~\(\rejectset=\bigcap\cset{\rejectset[\desirset]}{\desirset\in\setofdesirsets}\).
The largest such set~\(\setofdesirsets\) is then\/~\(\cohdesirsets\group{\rejectset}\coloneqq\cset{\desirset\in\cohdesirsets}{\rejectset\subseteq\rejectset[\desirset]}\).
\end{theorem}

It's also easy to see that \(\cohrejectsets\) is an intersection structure: if we consider any non-empty family of coherent sets of desirable option sets~\(\rejectset[i]\), \(i\in I\), then their intersection~\(\bigcap_{i\in I}\rejectset[i]\) is still coherent.
This implies that we can introduce a \emph{coherent closure} operator~\(\natexrejectset\colon\rejectsets\to\cohrejectsets\cup\set{\optsets}\) by letting
\begin{equation*}
\natexrejectset(\assessment)
\coloneqq\bigcap\cset{\rejectset\in\cohrejectsets}{\assessment\subseteq\rejectset}
\text{ for all~\(\assessment\subseteq\optsets\)}
\end{equation*}
be the smallest---if any---coherent set of desirable option sets that includes~\(\assessment\).
Such an~\(\assessment\) typically represents a so-called \emph{assessment}: a not necessarily exhaustive collection of option sets that a subject states to be desirable.

We call such an assessment~\(\assessment\subseteq\optsets\) \emph{consistent} if \(\natexrejectset(\assessment)\neq\optsets\), or equivalently, if \(\assessment\) is included in some coherent set of desirable option sets.
The closure operator~\(\natexrejectset\) implements \emph{conservative inference} with respect to the coherence axioms, in that it extends a consistent assessment~\(\assessment\) to the most conservative---smallest possible---coherent set of desirable option sets~\(\natexrejectset(\assessment)\).

It is clear that the so-called \emph{vacuous set of desirable option sets}~\(\smash{\rejectset[\posopts]}\) is the smallest element of \(\cohrejectsets\) with respect to set inclusion, and that therefore \(\smash{\natexrejectset(\emptyset)=\rejectset[\posopts]}\): it corresponds to making no assessment at all, whence, of course, its name.

If we combine these ideas with Theorem~\ref{theo:coherent:representation:twosided}, we are led to the following important result.

\begin{theorem}[Closure {\protect\cite{debock2018:choice:arxiv,debock2018:choice:smps,debock2019:choice:arxiv,debock2019:choice:isipta}}]\label{theo:essentially:archimedean:representation:rejectsets:noconstants}
Consider any~\(\rejectset\in\rejectsets\), then \(\natexrejectset(\rejectset)=\bigcap\cset{\rejectset[\desirset]}{\desirset\in\cohdesirsets(\rejectset)}\).
Hence, \(\rejectset\) is consistent if and only if~\(\cohdesirsets(\rejectset)\neq\emptyset\).
And a consistent~\(\rejectset\) is coherent if and only if \(\rejectset=\bigcap\cset{\rejectset[\desirset]}{\desirset\in\cohdesirsets(\rejectset)}\).
\end{theorem}

We can also lift the mixingness property from binary to general choice models, as Seidenfeld et al.~have done~\cite{seidenfeld2010}.
When we convert it into our language, this condition becomes~\cite{2017vancamp:phdthesis,2018vancamp:lexicographic}:
\begin{enumerate}[label=\(\mathrm{K}_{\mathrm{M}}\).,ref=\(\mathrm{K}_{\mathrm{M}}\),leftmargin=*]
\item\label{ax:rejects:removepositivecombinations} if \(\altoptset\in\rejectset\) and \(\optset\subseteq\altoptset\subseteq\posi\group{\optset}\), then also \(\optset\in\rejectset\), for all~\(\optset,\altoptset\in\optsets\).
\end{enumerate}
We call a set of desirable option sets~\(\rejectset\in\rejectsets\) \emph{mixing} if it's coherent and satisfies \ref{ax:rejects:removepositivecombinations}.
The set of all mixing sets of desirable option sets is denoted by~\(\mixrejectsets\).

The \emph{binary} elements of~\(\mixrejectsets\) are precisely the ones based on a mixing set of desirable options.

\begin{proposition}[Binary embedding {\protect\cite{debock2019:choice:arxiv,debock2019:choice:isipta}}]\label{prop:mixingbinaryiffD}
For any set of desirable options~\(\desirset\in\desirsets\), \(\rejectset[\desirset]\) is mixing if and only if~\(\desirset\) is, so \(\rejectset[\desirset]\in\mixrejectsets\ifandonlyif\desirset\in\mixdesirsets\).
\end{proposition}

Interestingly, and in contrast with its binary counterpart~\(\mixdesirsets\), the set of all mixing sets of desirable option sets~\(\mixrejectsets\) also constitutes an intersection structure, like the sets~\(\cohdesirsets\) and~\(\cohrejectsets\).
It therefore comes with its own mixing closure operator, associated conservative inference system, and notion of consistency.
I leave the details to be further explored by interested readers.
\noindent
For general mixing sets of desirable option sets that are not necessarily binary, we still have a representation theorem analogous to Theorem~\ref{theo:coherent:representation:twosided}.

\begin{theorem}[Representation {\protect\cite{debock2019:choice:arxiv,debock2019:choice:isipta}}]\label{theo:mixingrepresentation:twosided}
A set of desirable option sets~\(\rejectset\in\rejectsets\) is mixing if and only if there is some non-empty set~\(\setofdesirsets\subseteq\mixdesirsets\) of mixing sets of desirable options such that \(\rejectset=\bigcap\cset{\rejectset[\desirset]}{\desirset\in\setofdesirsets}\).
The largest such set~\(\setofdesirsets\) is then\/ \(\mixdesirsets\group{\rejectset}\coloneqq\cset{\desirset\in\mixdesirsets}{\rejectset\subseteq\rejectset[\desirset]}\).
\end{theorem}

How can we connect this choice of model, sets of desirable option sets, to the rejection and choice functions that I mentioned in the Introduction, and which are much more prevalent in the literature?
Their interpretation provides the clue.
Consider any option set~\(\optset\), and any option~\(\opt\in\optset\).
Then, with \(\optset\ominus\opt\coloneqq(\optset\setminus\set{\opt})-\opt=\cset{\altopt-\opt}{\altopt\in\optset,\altopt\neq\opt}\), we get
\begin{align*}
\opt\in\rejectfun(\optset)
&\ifandonlyif0\in\rejectfun(\optset-\opt)\\
&\ifandonlyif\text{ there is some~\(\altopt\in\optset\ominus\opt\) that is strictly preferred to~\(0\)}\\
&\ifandonlyif\optset\ominus\opt\in\rejectset.
\end{align*}
In these equivalences, the first one follows from compatibility of the rejection function~\(\rejectfun\) with vector addition, the second one follows from the particular interpretation we've given to the rejection function, and the third one follows from the definition of the set of desirable option sets~\(\rejectset\).
\emph{This tells us that, given this particular interpretation, choice and rejection functions are in a one-to-one relation with sets of desirable option sets.}

\begin{runningexample}\label{ex:identical:sides:nonbinary:inference}
In our coin example, we are now in a position to find out what are the consequences of making the assessment~\(\assessment_{\heads\,\text{or}\,\tails}\), or in other words, how to use the set of desirable option sets model in order to represent a subject's beliefs that the coin has two identical sides.
The conclusions we'll reach are valid on any choice of the background cone~\(\difs_{\gblgt0}\) or~\(\difs_{\gblpgt0}\).

Let
\begin{equation*}
\left\{
\begin{aligned}
\rejectset[\heads]
&\coloneqq\rejectset[{\desirset[\heads]}]
=\cset{\optset\in\optsets}{\optset\cap\desirset[\heads]\neq\emptyset}\\
\rejectset[\tails]
&\coloneqq\rejectset[{\desirset[\tails]}]
=\cset{\optset\in\optsets}{\optset\cap\desirset[\tails]\neq\emptyset}
\end{aligned}
\right.
\end{equation*}
be the coherent (and mixing) sets of desirable option sets that correspond to the respective desirability assessments~\(\assessment_\heads\) and~\(\assessment_\tails\), so to knowing that the coin has identical sides which are heads, or tails, respectively.
It follows after elementary considerations that the smallest coherent set of desirable option sets that includes the assessment~\(\assessment_{\heads\,\text{or}\,\tails}\) is given by
\begin{equation*}
\rejectset[{\heads\,\text{or}\,\tails}]
\coloneqq\natexrejectset(\assessment_{\heads\,\text{or}\,\tails})
=\cset{\optset\in\optsets}
{\group{\exists\opt\in\desirset[\heads]}\group{\exists\altopt\in\desirset[\tails]}\set{\opt,\altopt}\subseteq\optset}
=\rejectset[\heads]\cap\rejectset[\tails],
\end{equation*}
so this is the model we are after.
This coherent set of desirable option sets~\(\rejectset[{\heads\,\text{or}\,\tails}]\) is moreover mixing, because the coherent sets of desirable option sets~\(\rejectset[\heads]\) and~\(\rejectset[\tails]\) are.
\stopit
\end{runningexample}

\section{Linear and superlinear functionals}\label{sec:technical-aspects}
Because the notions of {\essarchimty} and {\archimty} that I intend to introduce further on rely on an idea of openness---and therefore closeness---I'll assume from now on that the option space~\(\opts\) constitutes a \emph{Banach space} with a norm~\(\optnorm{\bolleke}\) and a corresponding topological closure operator~\(\topcls\) and interior operator~\(\topint\).
In this section, I've gathered a few useful definitions and basic results for linear and superlinear bounded real functionals on the Banach space~\(\opts\).
I'll use these functionals to generalise to our more general context the linear previsions, and more generally, the coherent lower previsions defined by Peter Walley~\cite{walley1991} on spaces of gambles; see also~\cite{troffaes2013:lp} for more details on such coherent lower previsions.
The linear and superlinear bounded real functionals also generalise the exact functionals introduced and studied by Maa{\ss} \cite{maass2002,cooman2005e}. 

We will call a real functional~\(\ldualopt\colon\opts\to\reals\) \emph{superlinear} if it's superadditive and non-negatively homogeneous:
\begin{enumerate}[label=\({\mathrm{SL}}_{\arabic*}\).,ref=\({\mathrm{SL}}_{\arabic*}\),leftmargin=*]
\item\label{ax:superlinear:superadditive} \(\ldualopt(\opt+\altopt)\geq\ldualopt(\opt)+\ldualopt(\altopt)\) for all~\(\opt,\altopt\in\opts\);\hfill\upshape[superadditivity]
\item\label{ax:superlinear:homogeneous} \(\ldualopt(\lambda\opt)=\lambda\ldualopt(\opt)\) for all~\(\opt\in\opts\) and all real~\(\lambda\geq0\).\hfill\upshape[non-negative homogeneity]
\end{enumerate}
With any real functional~\(\ldualopt\colon\opts\to\reals\) we can associate its \emph{conjugate} (functional)~\(\udualopt\colon\opts\to\reals\) defined by
\begin{equation*}
\udualopt(\opt)\coloneqq-\ldualopt(-\opt)
\text{ for all~\(\opt\in\opts\)}.
\end{equation*}
A real functional~\(\dualopt\colon\opts\to\reals\) is called \emph{linear} if it's both superlinear and \emph{self-conjugate}, i.e.~equal to its conjugate.
This amounts to requiring that
\begin{enumerate}[label=\({\mathrm{L}}\).,ref=\({\mathrm{L}}\),leftmargin=*]
\item\label{ax:linear} \(\dualopt(\lambda\opt+\mu\altopt)=\lambda\dualopt(\opt)+\mu\dualopt(\altopt)\) for all~\(\opt,\altopt\in\opts\) and~\(\lambda,\mu\in\reals\).\hfill\upshape[linearity]
\end{enumerate}

Let's list a few useful properties of superlinear real functionals and their conjugates, where we use the shorthand notation \(\maxludualopt{\opt}\coloneqq\max\set{\abs{\ldualopt(\opt)},\abs{\udualopt(\opt)}}\) for all~\(\opt\in\opts\).

\begin{proposition}\label{prop:ludualopt:properties}
Consider any superlinear real functional~\(\ldualopt\colon\opts\to\reals\) and its conjugate~\(\udualopt\), then
\begin{enumerate}[label=\upshape(\roman*),leftmargin=*]
\item\label{it:ludualopt:properties:lsmalleru} \(\ldualopt(\opt)\leq\udualopt(\opt)\) for all~\(\opt\in\opts\);
\item\label{it:ludualopt:properties:mixed:additivity} \(\ldualopt(\opt)+\ldualopt(\altopt)\leq\ldualopt(\opt+\altopt)\leq\ldualopt(\opt)+\udualopt(\altopt)\leq\udualopt(\opt+\altopt)\leq\udualopt(\opt)+\udualopt(\altopt)\) for all~\(\opt,\altopt\in\opts\);
\item\label{it:ludualopt:properties:bounds} \(\max\set{\abs{\ldualopt(\opt)-\ldualopt(\altopt)},\abs{\udualopt(\opt)-\udualopt(\altopt)}}\leq\maxludualopt{\opt-\altopt}\) for all~\(\opt,\altopt\in\opts\);
\item\label{it:ludualopt:properties:first:bound:for:norm} \(\maxludualopt{\lambda\opt}=\abs{\lambda}\maxludualopt{\opt}\) for all~\(\opt\in\opts\) and~\(\lambda\in\reals\);
\item\label{it:ludualopt:properties:second:bound:for:norm} \(\maxludualopt{\lambda\opt+\mu\altopt}\leq\abs{\lambda}\maxludualopt{\opt}+\abs{\mu}\maxludualopt{\altopt}\) for all~\(\opt,\altopt\in\opts\) and~\(\lambda,\mu\in\reals\).
\end{enumerate}
\end{proposition}
\noindent Observe that such superlinear real functionals~\(\ldualopt\) are not necessarily monotone with respect to the background ordering~\(\optgt\), in the sense that~\(\altopt\optgt\opt\) implies that \(\ldualopt(\altopt)\geq\ldualopt(\opt)\).
We'll come back to this monotonicity issue in Section~\ref{sec:essential:archimedeanity}, where we focus on particular so-called \emph{positive} superlinear real functionals.

\begin{proof}[Proof of Proposition~\ref{prop:ludualopt:properties}]
% Checked by Gert
For~\ref{it:ludualopt:properties:lsmalleru}, observe that~\(\opt+(-\opt)=0\) and therefore
\begin{equation*}
0=\ldualopt(0)\geq\ldualopt(\opt)+\ldualopt(-\opt)=\ldualopt(\opt)-\udualopt(\opt),
\end{equation*}
where the first equality follows from Axiom~\ref{ax:superlinear:homogeneous}, and the inequality from Axiom~\ref{ax:superlinear:superadditive}.

For~\ref{it:ludualopt:properties:mixed:additivity}, the first and last inequalities follow from Axiom~\ref{ax:superlinear:superadditive} (and conjugacy).
We prove the second inequality; the third then follows again by considering the consequences of conjugacy.
Since \(\opt=(\opt+\altopt)-\altopt\), we infer from Axiom~\ref{ax:superlinear:superadditive} and conjugacy that, indeed,
\begin{equation*}
\ldualopt(\opt)
\geq\ldualopt(\opt+\altopt)+\ldualopt(-\altopt)
=\ldualopt(\opt+\altopt)-\udualopt(\altopt).
\end{equation*}

For~\ref{it:ludualopt:properties:bounds}, it suffices to prove that \(\abs{\udualopt(\opt)-\udualopt(\altopt)}\leq\max\set{\abs{\udualopt(\opt-\altopt)},\abs{\udualopt(\altopt-\opt)}}\), because replacing~\(\opt\) with~\(-\opt\) and~\(\altopt\) with~\(-\altopt\) then completes the proof.
Observe that it follows from Axiom~\ref{ax:superlinear:superadditive} and conjugacy that, since \(\opt=(\opt-\altopt)+\altopt\), \(\udualopt(\opt)\leq\udualopt(\opt-\altopt)+\udualopt(\altopt)\), and therefore
\begin{equation}\label{eq:ludualopt:properties:bounds}
\udualopt(\opt-\altopt)\geq\udualopt(\opt)-\udualopt(\altopt)
\text{ for all~\(\opt,\altopt\in\opts\)}.
\end{equation}
There are now two possibilities.
The first is that \(\udualopt(\opt)\geq\udualopt(\altopt)\), and then Equation~\eqref{eq:ludualopt:properties:bounds} guarantees that
\begin{equation*}
\abs{\udualopt(\opt)-\udualopt(\altopt)}
\leq\abs{\udualopt(\opt-\altopt)}
\leq\max\set{\abs{\udualopt(\opt-\altopt)},\abs{\udualopt(\altopt-\opt)}}.
\end{equation*}
The second case is that \(\udualopt(\opt)\leq\udualopt(\altopt)\), and then simply exchanging the roles of~\(\opt\) and~\(\altopt\) in the argument above leads to the (same) inequality:
\begin{equation*}
\abs{\udualopt(\altopt)-\udualopt(\opt)}
\leq\max\set{\abs{\udualopt(\altopt-\opt)},\abs{\udualopt(\opt-\altopt)}}.
\end{equation*}

For~\ref{it:ludualopt:properties:first:bound:for:norm}, we first prove the inequality~\(\maxludualopt{\lambda\opt}\leq\abs{\lambda}\maxludualopt{\opt}\).
To do this, we will prove that \(\abs{\ldualopt(\lambda\opt)}\leq\abs{\lambda}\max\set{\abs{\ldualopt(\opt)},\abs{\udualopt(\opt)}}\).
The rest of the proof for the inequality then follows from conjugacy and replacing~\(\lambda\) with~\(-\lambda\).
Let's first consider the case that~\(\lambda\geq0\).
It then follows from Axiom~\ref{ax:superlinear:homogeneous} that, indeed,
\begin{equation*}
\abs{\ldualopt(\lambda\opt)}
=\abs{\lambda\ldualopt(\opt)}
=\abs{\lambda}\abs{\ldualopt(\opt)}
\leq\abs{\lambda}\max\set{\abs{\ldualopt(\opt)},\abs{\udualopt(\opt)}}.
\end{equation*}
If, on the other hand, \(\lambda\leq0\), then it follows from Axiom~\ref{ax:superlinear:homogeneous} and conjugacy that, indeed,
\begin{equation*}
\abs{\ldualopt(\lambda\opt)}
=\abs{\lambda\udualopt(\opt)}
=\abs{\lambda}\abs{\udualopt(\opt)}
\leq\abs{\lambda}\max\set{\abs{\ldualopt(\opt)},\abs{\udualopt(\opt)}}.
\end{equation*}
Next, we prove that, actually, \(\maxludualopt{\lambda\opt}=\abs{\lambda}\maxludualopt{\opt}\).
It's obvious that the equality holds for \(\lambda=0\) [use Axiom~\ref{ax:superlinear:homogeneous}], so we'll assume that \(\lambda\neq0\). 
If we also invoke the inequality already proved with \(\lambda\) replaced by \(\lambda^{-1}\), we get that
\begin{equation*}
\maxludualopt{\opt}
=\maxludualopt{\lambda^{-1}(\lambda\opt)}
\leq\abs{\lambda^{-1}}\maxludualopt{\lambda\opt}
\leq\abs{\lambda^{-1}}\abs{\lambda}\maxludualopt{\opt}
=\maxludualopt{\opt},
\end{equation*}
which also proves the equality for~\(\lambda\neq0\).

Finally, for~\ref{it:ludualopt:properties:second:bound:for:norm}, we'll prove that
\begin{equation*}
\abs{\ldualopt(\lambda\opt+\mu\altopt)}
\leq\abs{\lambda}\max\set{\abs{\ldualopt(\opt)},\abs{\udualopt(\opt)}}
+\abs{\mu}\max\set{\abs{\ldualopt(\altopt)},\abs{\udualopt(\altopt)}}.
\end{equation*}
The rest of the proof for this statement then follows from conjugacy and replacing~\(\lambda\) with~\(-\lambda\) and~\(\mu\) with~\(-\mu\).
We first consider the case that~\(\ldualopt(\lambda\opt+\mu\altopt)\geq0\).
Then
\begin{align*}
\abs{\ldualopt(\lambda\opt+\mu\altopt)}
=\ldualopt(\lambda\opt+\mu\altopt)
&\leq\udualopt(\lambda\opt)+\udualopt(\mu\altopt)
\leq\abs{\udualopt(\lambda\opt)}+\abs{\udualopt(\mu\altopt)}\\
&\leq\abs{\lambda}\max\set{\abs{\ldualopt(\opt)},\abs{\udualopt(\opt)}}
+\abs{\mu}\max\set{\abs{\ldualopt(\altopt)},\abs{\udualopt(\altopt)}},
\end{align*}
where the first inequality follows from~\ref{it:ludualopt:properties:mixed:additivity} and the third one from~\ref{it:ludualopt:properties:first:bound:for:norm}.
For the case that~\(\ldualopt(\lambda\opt+\mu\altopt)\leq0\), we get
\begin{align*}
\abs{\ldualopt(\lambda\opt+\mu\altopt)}
=-\ldualopt(\lambda\opt+\mu\altopt)
&\leq-\ldualopt(\lambda\opt)-\ldualopt(\mu\altopt)
=\udualopt(-\lambda\opt)+\udualopt(-\mu\altopt)\\
&\leq\abs{\udualopt(-\lambda\opt)}+\abs{\udualopt(-\mu\altopt)}\\
&\leq\abs{\lambda}\max\set{\abs{\ldualopt(\opt)},\abs{\udualopt(\opt)}}
+\abs{\mu}\max\set{\abs{\ldualopt(\altopt)},\abs{\udualopt(\altopt)}},
\end{align*}
where the first inequality follows from~\ref{it:ludualopt:properties:mixed:additivity} and the third one from~\ref{it:ludualopt:properties:first:bound:for:norm}; the second equality follows from conjugacy.
\end{proof}

A real functional~\(\ddualopt\colon\opts\to\reals\) on~\(\opts\) is called \emph{bounded} if its \emph{operator norm}~\(\ddualoptnorm{\ddualopt}<+\infty\), where we let
\begin{equation*}
\ddualoptnorm{\ddualopt}
\coloneqq\sup_{\opt\in\opts\setminus\set{0}}\frac{\abs{\ddualopt(\opt)}}{\optnorm{\opt}}.
\end{equation*}
We'll denote by~\(\ddualopts\) the linear space of all such bounded real functionals on~\(\opts\).
For any bounded real functional~\(\ldualopt\in\ddualopts\), its conjugate functional~\(\udualopt\) is also bounded---so \(\udualopt\in\ddualopts\)---because obviously
\begin{equation*}
\ddualoptnorm{\udualopt}
=\sup_{\opt\in\opts\setminus\set{0}}
\frac{\abs{\udualopt(\opt)}}{\optnorm{\opt}}
=\sup_{-\altopt\in\opts\setminus\set{0}}
\frac{\abs{\udualopt(-\altopt)}}{\optnorm{-\altopt}}
=\sup_{\altopt\in\opts\setminus\set{0}}
\frac{\abs{\ldualopt(\altopt)}}{\optnorm{\altopt}}
=\ddualoptnorm{\ldualopt}.
\end{equation*}

The linear space~\(\ddualopts\) can be topologised by the operator norm~\(\ddualoptnorm{\bolleke}\), which leads to the so-called \emph{original topology} on~\(\ddualopts\).
If we associate with any~\(\opt\in\opts\) the so-called \emph{evaluation functional}~\(\opt^\circ\colon\ddualopts\to\reals\), defined by
\begin{equation*}
\opt^\circ(\ddualopt)\coloneqq\ddualopt(\opt)
\text{ for all~\(\ddualopt\in\ddualopts\)},
\end{equation*}
then \(\opt^\circ\) is clearly a real linear functional on the normed linear space~\(\ddualopts\), whose operator norm
\begin{equation*}
\sup_{\ddualopt\in\ddualopts\setminus\set{0}}
\frac{\abs{\opt^\circ(\ddualopt)}}{\ddualoptnorm{\ddualopt}}
\leq\sup_{\ddualopt\in\ddualopts\setminus\set{0}}
\abs{\opt^\circ(\ddualopt)}\frac{\optnorm{\opt}}{\abs{\ddualopt(\opt)}}
=\optnorm{\opt}<+\infty
\end{equation*}
is finite, which implies that \(\opt^\circ\) is a continuous real linear functional on~\(\ddualopts\) with respect to the original topology on~\(\ddualopts\)~\cite[Section~23.1]{schechter1997}.

We'll also retopologise~\(\ddualopts\) with the topology of pointwise convergence on~\(\ddualopts\), which is the weakest topology that makes all evaluation functionals~\(\opt^\circ\), \(\opt\in\opts\) continuous.
It is therefore weaker than the original topology induced by the norm~\(\ddualoptnorm{\bolleke}\).
We'll call this topology the \emph{{\weakcircle} topology} on~\(\ddualopts\).

An interesting subspace of~\(\ddualopts\) is the linear space~\(\dualopts\) of all \emph{linear} bounded---and therefore continuous \cite[Section~23.1]{schechter1997}---real functionals on~\(\opts\).
We'll also consider the set~\(\ldualopts\) of all \emph{superlinear} bounded real functionals~\(\ldualopt\) on~\(\opts\).
Obviously, \(\ldualopts\) is a convex cone in~\(\ddualopts\), and \(\dualopts\subseteq\ldualopts\subseteq\ddualopts\).
The relativisation to~\(\dualopts\) of the {\weakcircle} topology on~\(\ddualopts\) is of course the commonly considered so-called \emph{{\weakstar} topology} on~\(\dualopts\).

Using Proposition~\ref{prop:ludualopt:properties}, it now takes but a small (and fairly standard) effort to prove the Lipschitz continuity---and hence the (uniform) continuity---of all superlinear bounded real functionals.

\begin{proposition}\label{prop:ludualopt:continuous}
Any superlinear real functional on~\(\opts\), as well as its conjugate~\(\udualopt\), is Lipschitz continuous if and only if it's bounded.
\end{proposition}

\begin{proof}
% Checked by Gert
We provide the proof for~\(\ldualopt\).
The argument for~\(\udualopt\) is essentially the same.
First, assume that \(\ldualopt\) is bounded, so \(\ldualopt\in\ldualopts\).
We then have to prove that there is some real~\(K\geq0\) such that
\begin{equation}\label{eq:ludualopt:continuous}
(\forall\opt,\altopt\in\opts)
\abs{\ldualopt(\opt)-\ldualopt(\altopt)}\leq K\optnorm{\opt-\altopt}.
\end{equation}
It follows from the boundedness of~\(\ldualopt\) that \(\ddualoptnorm{\ldualopt}\in\nonnegreals\) and that
\begin{equation*}
(\forall\altopttoo\in\opts)
\abs{\udualopt(\altopttoo)}\leq\ddualoptnorm{\ldualopt}\optnorm{\altopttoo}.
\end{equation*}
This allows us to infer from Proposition~\ref{prop:ludualopt:properties}\ref{it:ludualopt:properties:bounds} that
\begin{equation*}
\abs{\ldualopt(\opt)-\ldualopt(\altopt)}
\leq\max\set{\abs{\udualopt(\opt-\altopt)},\abs{\udualopt(\altopt-\opt)}}
\leq \ddualoptnorm{\ldualopt}\optnorm{\opt-\altopt},
\end{equation*}
which indeed guarantees that the Lipschitz property~\eqref{eq:ludualopt:continuous} holds, if we let \(K\coloneqq\ddualoptnorm{\ldualopt}\).

Next, let us assume that \(\ldualopt\) is Lipschitz continuous, so there is some real~\(K\geq0\) such that the Lipschitz property~\eqref{eq:ludualopt:continuous} holds.
Then we have to prove that \(\ldualopt\) is bounded.
Observe that it follows from Axiom~\ref{ax:superlinear:homogeneous} that \(\ldualopt(0)=0\), so it follows from the Lipschitz continuity of~\(\ldualopt\) that in particular \(\group{\forall\opt\in\opts}\abs{\ldualopt(\opt)}\leq K\optnorm{\opt}\), and therefore also
\begin{equation*}
\ddualoptnorm{\ldualopt}
=\sup_{\opt\in\opts\setminus\set{0}}\frac{\abs{\ldualopt(\opt)}}{\optnorm{\opt}}
\leq K<+\infty,
\end{equation*}
so \(\ldualopt\) is indeed bounded.
\end{proof}

If we consider, for any~\(\ldualopt\in\ldualopts\), its \emph{set of dominating continuous linear functionals}
\begin{align*}
\dualopts(\ldualopt)
\coloneqq&\cset{\dualopt\in\dualopts}
{(\forall\opt\in\opts)\ldualopt(\opt)\leq\dualopt(\opt)}\\
=&\cset{\dualopt\in\dualopts}
{(\forall\opt\in\opts)\ldualopt(\opt)\leq\dualopt(\opt)\leq\udualopt(\opt)},
\end{align*}
then a well-known version of the Hahn--Banach theorem \cite[Section~28.4, HB17]{schechter1997} can be formulated as follows.

\begin{theorem}[Hahn--Banach theorem]\label{theo:hahn:banach}
Consider any~\(\ldualopt\in\ldualopts\).
Then for all~\(\opt\in\opts\) there is some~\(\dualopt\in\dualopts(\ldualopt)\) such that \(\ldualopt(\opt)=\dualopt(\opt)\).
\end{theorem}
\noindent An important condition for using this version is that~\(\ldualopt\) should be both superlinear and continuous.
It should be clear that \(\dualopts(\ldualopt)\) is a convex and {\weakstar}-closed subset of~\(\dualopts\), and also a convex and {\weakcircle}-closed subset of~\(\ddualopts\).
We can use this version to give a straightforward proof for the following representation result.

\begin{corollary}[Lower envelope theorem]\label{cor:lower:envelope:for:cslfs}
Consider any real functional~\(\ldualopt\) on~\(\opts\).
Then \(\ldualopt\) is bounded and superlinear---so \(\ldualopt\in\ldualopts\)---if and only if there is some non-empty norm-bounded subset~\(\setofdualopts\) of~\(\dualopts\) such that \(\ldualopt\) is the \emph{lower envelope} of~\(\setofdualopts\), meaning that \(\ldualopt(\opt)=\inf\cset{\dualopt(\opt)}{\dualopt\in\setofdualopts}\) for all~\(\opt\in\opts\).
In that case the largest such set~\(\setofdualopts\) is the convex and {\weakstar}-compact subset~\(\dualopts(\ldualopt)\) of \(\dualopts\), and \(\ldualopt(\opt)=\min\cset{\dualopt(\opt)}{\dualopt\in\dualopts(\ldualopt)}\) for all~\(\opt\in\opts\).
\end{corollary}

\begin{proof}
% Checked by Gert
For sufficiency, assume that \(\ldualopt\) is the lower envelope of a norm-bounded~\(\setofdualopts\subseteq\dualopts\), then it's easy to see that \(\ldualopt\) is superlinear, so it's enough to show that it's bounded.
That~\(\setofdualopts\) is norm-bounded means that there is some real \(K\in\posreals\) such that \(\ddualoptnorm{\dualopt}\leq K\) for all~\(\dualopt\in\setofdualopts\), and therefore also
\begin{equation}\label{eq:normbounded}
\abs{\dualopt(\opt)}\leq K\optnorm{\opt}\text{ for all~\(\dualopt\in\setofdualopts\) and all~\(\opt\in\opts\)}.
\end{equation}
Now consider any~\(\altopt\in\opts\setminus\set{0}\).
There are two possibilities.
The first is that \(\ldualopt(\altopt)\geq0\) and then it follows from~\eqref{eq:normbounded} that
\begin{equation}\label{eq:towards:bounded:first:case}
\abs{\ldualopt(\altopt)}
=\ldualopt(\altopt)
\leq\dualopt(\altopt)
=\abs{\dualopt(\altopt)}
\leq K\optnorm{\altopt}
\text{ for all~\(\dualopt\in\setofdualopts\).}
\end{equation}
The other possibility is that \(\ldualopt(\altopt)<0\) and then~\(\udualopt(-\altopt)=-\ldualopt(\altopt)=\abs{\ldualopt(\altopt)}>0\).
Then, for any~\(\epsilon\in\posreals\) such that \(0<\epsilon<\udualopt(-\altopt)\), there is, by assumption, some~\(\dualopt[\epsilon]\in\setofdualopts\) such that \(0<\udualopt(-\altopt)-\epsilon<\dualopt[\epsilon](-\altopt)\leq\udualopt(-\altopt)\), so it follows from~\eqref{eq:normbounded} that for all~\(\epsilon\) small enough
\begin{equation*}
\abs{\ldualopt(\altopt)}
=\udualopt(-\altopt)
<\dualopt[\epsilon](-\altopt)+\epsilon
=\abs{\dualopt[\epsilon](-\altopt)}+\epsilon
\leq K\optnorm{\altopt}+\epsilon,
\end{equation*}
and therefore also
\begin{equation}\label{eq:towards:bounded:second:case}
\abs{\ldualopt(\altopt)}
\leq K\optnorm{\altopt}.
\end{equation}
Combining~\eqref{eq:towards:bounded:first:case} and~\eqref{eq:towards:bounded:second:case} then guarantees that \(\ddualoptnorm{\ldualopt}\leq K\), so \(\ldualopt\) is indeed bounded.

For necessity, assume that \(\ldualopt\) is bounded, and observe that, by Theorem~\ref{theo:hahn:banach}, \(\ldualopt\) is then the lower envelope of the set~\(\dualopts(\ldualopt)\subseteq\dualopts\), and that this set is trivially convex and {\weakstar}-closed.
So we still need to prove that \(\dualopts(\ldualopt)\) is norm-bounded, which will then also guarantee that it's {\weakstar}-compact.
To do so, consider any~\(\dualopt\in\dualopts(\ldualopt)\) and any~\(\opt\in\opts\setminus\set{0}\), then there are two possibilities.
The first is that \(\dualopt(\opt)\geq0\), and then
\begin{equation*}
0
\leq\abs{\dualopt(\opt)}
=\dualopt(\opt)
\leq\udualopt(\opt)
=\abs{\udualopt(\opt)}
\leq\ddualoptnorm{\ldualopt}\optnorm{\opt}.
\end{equation*}
The other possibility is that \(\dualopt(\opt)\leq0\), and then
\begin{equation*}
0
\leq\abs{\dualopt(\opt)}
=-\dualopt(\opt)
=\dualopt(-\opt)
\leq\udualopt(-\opt)
=\abs{\udualopt(-\opt)}
\leq\ddualoptnorm{\ldualopt}\optnorm{\opt}.
\end{equation*}
This guarantees that \(\ddualoptnorm{\dualopt}\leq\ddualoptnorm{\ldualopt}\), and since this holds for all~\(\dualopt\in\dualopts(\ldualopt)\), this tells us that \(\dualopts(\ldualopt)\) is indeed norm-bounded.
\end{proof}

\begin{runningexample}\label{ex:functionals}
Going back to our coin example, it's easy to see that the linear space~\(\difs\) is two-dimensional, because for any \(\dif\in\difs\) it clearly holds that
\begin{equation*}
\dif
=\dif(\bolleke,\bestreward)\unitdif
=\sqgroup[\big]{\dif(\heads,\bestreward)\indset{\heads}+\dif(\tails,\bestreward)\indset{\tails}}\unitdif
=\dif(\heads,\bestreward)\indset{\heads}\unitdif+\dif(\tails,\bestreward)\indset{\tails}\unitdif,
\end{equation*}
so the two maps~\(\indset{\heads}\unitdif\) and~\(\indset{\tails}\unitdif\) constitute a basis for the linear space~\(\difs\).
We can provide the linear space~\(\difs\) with the norm~\(\difnorm{\bolleke}\) defined by
\begin{equation*}
\difnorm{\dif}
\coloneqq\max_{(x,r)\in\coinstatesandrewards}\abs{\dif(x,r)}
=\max_{x\in\coinstates}\abs{\dif(x,\bestreward)}
=\max\set{\abs{\dif(\heads,\bestreward)},\abs{\dif(\tails,\bestreward)}},
\end{equation*}
which, essentially, corresponds to the supremum norm on the linear space~\(\gblson{\coinstates}\).
It is then a standard result that this (supremum) norm turns~\(\difs\) into a Banach space.

Any linear real functional \(\dualopt\) on~\(\difs\) then satisfies
\begin{equation}\label{eq:coins:linear:functionals:basis}
\dualopt(\dif)
=\dif(\heads,\bestreward)\dualopt\group[\big]{\indset{\heads}\unitdif}
+\dif(\tails,\bestreward)\dualopt\group[\big]{\indset{\tails}\unitdif},
\end{equation}
and therefore
\begin{align*}
\ddualdifnorm{\dualopt}
&=\sup_{\dif\in\difs\setminus\set{0}}\frac{\abs{\dualopt(\dif)}}{\difnorm{\dif}}
=\sup_{\dif\in\difs\setminus\set{0}}
\frac{\abs[\big]{\dif(\heads,\bestreward)\dualopt\group[\big]{\indset{\heads}\unitdif}
+\dif(\tails,\bestreward)\dualopt\group[\big]{\indset{\tails}\unitdif}}}
{\max\set{\abs{\dif(\heads,\bestreward)},\abs{\dif(\tails,\bestreward)}}}\\
&=\abs[\big]{\dualopt\group[\big]{\indset{\heads}\unitdif}}
+\abs[\big]{\dualopt\group[\big]{\indset{\tails}\unitdif}},
\end{align*}
so all linear real functionals on~\(\difs\) are bounded (and therefore Lipschitz continuous), and the linear space~\(\dualdifs\) of all these linear (bounded) real functionals is two-dimensional too.

For superlinear real functionals~\(\ldualopt\) on~\(\difs\), we infer from Proposition~\ref{prop:ludualopt:properties}\ref{it:ludualopt:properties:second:bound:for:norm} that
\begin{align*}
\ddualdifnorm{\ldualopt}
&=\sup_{\dif\in\difs\setminus\set{0}}\frac{\abs{\ldualopt(\dif)}}{\difnorm{\dif}}
=\sup_{\dif\in\difs\setminus\set{0}}
\frac{\abs[\big]{\ldualopt\group[\big]{\dif(\heads,\bestreward)\indset{\heads}\unitdif+\dif(\tails,\bestreward)\indset{\tails}\unitdif}}}
{\max\set{\abs{\dif(\heads,\bestreward)},\abs{\dif(\tails,\bestreward)}}}\\
&\leq\sup_{\dif\in\difs\setminus\set{0}}
\frac{\abs{\dif(\heads,\bestreward)}\maxludualopt[\big]{\indset{\heads}\unitdif}
+\abs{\dif(\tails,\bestreward)}\maxludualopt[\big]{\indset{\tails}\unitdif}}
{\max\set{\abs{\dif(\heads,\bestreward)},\abs{\dif(\tails,\bestreward)}}}\\
&=\maxludualopt[\big]{\indset{\heads}\unitdif}+\maxludualopt[\big]{\indset{\tails}\unitdif}\\
&=\max\set[\Big]{\abs[\big]{\ldualopt\group[\big]{\indset{\heads}\unitdif}},\abs[\big]{\udualopt\group[\big]{\indset{\heads}\unitdif}}}
+\max\set[\Big]{\abs[\big]{\ldualopt\group[\big]{\indset{\tails}\unitdif}},\abs[\big]{\udualopt\group[\big]{\indset{\tails}\unitdif}}},
\end{align*}
so all superlinear functionals on~\(\difs\) are bounded (and therefore Lipschitz continuous) too.

It follows from Equation~\eqref{eq:coins:linear:functionals:basis} and Corollary~\ref{cor:lower:envelope:for:cslfs} that there is a one-to-one correspondence between the superlinear functionals~\(\ldualopt\) on~\(\difs\) and the convex compact (closed and bounded) subsets~\(C\) of~\(\reals^2\), given by
\begin{equation*}
C_{\ldualopt}
\coloneqq\cset[\Big]{(\lambda,\mu)\in\reals^2}
{(\forall(\alpha,\beta)\in\reals^2)
\ldualopt\group[\big]{\alpha\indset{\heads}\unitdif+\beta\indset{\tails}\unitdif}\leq\alpha\lambda+\beta\mu}
\end{equation*}
and
\begin{equation*}
\ldualopt[C]\group[\big]{\alpha\indset{\heads}\unitdif+\beta\indset{\tails}\unitdif}
\coloneqq\min\cset{\alpha\lambda+\beta\mu}{(\lambda,\mu)\in C}.
\end{equation*}
It extends the one-to-one correspondence between the linear functionals~\(\dualopt\) on~\(\difs\) and the elements of~\(\reals^2\) that is expressed by Equation~\eqref{eq:coins:linear:functionals:basis}: the points in~\(C_{\ldualopt}\) are the representations in~\(\reals^2\) of the elements~\(\dualopt\in\dualopts(\ldualopt)\).

As a special case, let us define the real functionals~\(\dualopt[p]\) and~\(\ldualopt[I]\) on~\(\difs\) by
\begin{equation}\label{eq:define:the:functionals}
\dualopt[p](\dif)\coloneqq\ex[p](\dif(\bolleke,\bestreward))
\text{ and }
\ldualopt[I](\dif)\coloneqq\lex[I](\dif(\bolleke,\bestreward))
\text{ for all~\(\dif\in\difs\)},
\end{equation}
where the functionals~\(\ex[p]\) and~\(\lex[I]\) were defined in Equations~\eqref{eq:coin:precise:expectation} and~\eqref{eq:coin:imprecise:expectation} of instalment~\ref{ex:assessment:lower:upper:betting:rates} of the coin example, for \(p\in\unit\) and and \(I=\lup\) with \(0\leq\lp\leq\up\leq1\).
Then it follows readily that \(\dualopt[p]\) is a linear bounded real functional on~\(\difs\), so~\(\dualopt[p]\in\dualdifs\), and that \(\ldualopt[I]\) is a superlinear bounded real functional on~\(\difs\), so~\(\dualopt[I]\in\ldualdifs\).
After a few elementary algebraic manipulations, we find that \(C_{\dualopt[p]}=\set{(p,1-p)}\) and that \(C_{\ldualopt[I]}=\cset{(\lambda,1-\lambda)}{\lp\leq\lambda\leq\up}\); see also Figure~\ref{fig:essential:archimedeanity} further on.
\stopit
\end{runningexample}

\section{{\Essarchimty} for sets of desirable options}\label{sec:essential:archimedeanity}
We are now ready to consider our first basic notion of {\archimty} for binary choice models.
To set the stage, we remark that the background ordering~\(\optgt\) on~\(\opts\) introduced in Section~\ref{sec:binary:choice} allows us to define convex cones of \emph{positive} (super)linear bounded real functionals:
\begin{align}
\posdualopts
&\coloneqq\cset{\dualopt\in\dualopts}
{(\forall\opt\in\posopts)\dualopt(\opt)>0}
\label{eq:posdualopts}\\
\posldualopts
&\coloneqq
\cset{\ldualopt\in\ldualopts}
{(\forall\opt\in\posopts)\ldualopt(\opt)>0}.
\label{eq:posldualopts}
\end{align}
Observe that \(\posdualopts\subseteq\posldualopts\).
Any positive (super)linear bounded real functional~\(\ldualopt\in\posdualopts\) satisfies the following \emph{strict monotonicity} property:
\begin{equation}\label{eq:strict:monotonicity}
\group{\forall\opt,\altopt\in\opts}
\group[\big]{\altopt\optgt\opt\then\ldualopt(\altopt)>\ldualopt(\opt)}.
\end{equation}

\begin{proof}[Proof of Equation~\eqref{eq:strict:monotonicity}]
Assume that \(\altopt\optgt\opt\), then also~\(\altopt-\opt\optgt0\), and it therefore follows from the positivity of~\(\ldualopt\) that \(\ldualopt(\altopt-\opt)>0\).
Proposition~\ref{prop:ludualopt:properties}\ref{it:ludualopt:properties:mixed:additivity} then guarantees that, indeed, \(\ldualopt(\altopt)=\ldualopt(\altopt-\opt+\opt)\geq\ldualopt(\altopt-\opt)+\ldualopt(\opt)>\ldualopt(\opt)\).
\end{proof}

\emph{We'll assume from now on that the strict vector ordering \(\optgt\) is such that \(\posdualopts\neq\emptyset\), which then of course also implies that \(\posldualopts\neq\emptyset\).
We'll also require for the remainder of this paper that\/ \(\topint(\posopts)\neq\emptyset\): the background cone of positive options has a non-empty interior.}

With any~\(\ldualopt\in\ldualopts\), we can associate a set of desirable options as follows:
\begin{equation}\label{eq:desirset:from:cslf}
\desirset[\ldualopt]
\coloneqq\ldualopt[>0]
=\cset{\opt\in\opts}{\ldualopt(\opt)>0}.
\end{equation}
It is an immediate consequence of Corollary~\ref{cor:lower:envelope:for:cslfs} that \(\desirset[\ldualopt]=\bigcap\cset{\desirset[\dualopt]}{\dualopt\in\dualopts(\ldualopt)}\).

Also, with a set of desirable options~\(\desirset\in\desirsets\), we can associate the following sets of superlinear and of linear bounded real functionals:
\begin{align}
\dualopts(\desirset)
&\coloneqq
\cset{\dualopt\in\dualopts}
{(\forall\opt\in\desirset)\dualopt(\opt)>0}
=\cset{\dualopt\in\dualopts}{\desirset\subseteq\desirset[\dualopt]}
\label{eq:clfs:from:desirset}\\
\ldualopts(\desirset)
&\coloneqq
\cset{\ldualopt\in\ldualopts}
{(\forall\opt\in\desirset)\ldualopt(\opt)>0}
=\cset{\ldualopt\in\ldualopts}{\desirset\subseteq\desirset[\ldualopt]}
\label{eq:cslfs:from:desirset}
\end{align}
and, similarly, the following sets of positive superlinear and of positive linear bounded real functionals:
\begin{align}
\posdualopts(\desirset)
&\coloneqq
\cset{\dualopt\in\posdualopts}
{(\forall\opt\in\desirset)\dualopt(\opt)>0}
=\cset{\dualopt\in\posdualopts}{\desirset\subseteq\desirset[\dualopt]}
% =\dualopts(\desirset)\cap\posdualopts
\label{eq:posclfs:from:desirset}\\
\posldualopts(\desirset)
&\coloneqq
\cset{\ldualopt\in\posldualopts}
{(\forall\opt\in\desirset)\ldualopt(\opt)>0}
=\cset{\ldualopt\in\posldualopts}{\desirset\subseteq\desirset[\ldualopt]},
% =\ldualopts(\desirset)\cap\posldualopts,
\label{eq:poscslfs:from:desirset}
\end{align}
where we used Equation~\eqref{eq:desirset:from:cslf} for the second equalities.
Clearly, \(\dualopts(\desirset)\subseteq\ldualopts(\desirset)\) and \(\posdualopts(\desirset)\subseteq\posldualopts(\desirset)\).
These sets are convex subcones of the convex cones~\(\ldualopts\) and~\(\posldualopts\), respectively, and \(\dualopts(\desirset)\) and \(\posdualopts(\desirset)\) are also convex cones in the dual linear space~\(\dualopts\) of continuous real linear functionals on~\(\opts\).
Observe that, in particular for \(\desirset=\posopts\), taking into account Equations~\eqref{eq:posdualopts} and~\eqref{eq:posldualopts},
\begin{align*}
\dualopts(\posopts)
&\coloneqq
\cset{\dualopt\in\dualopts}
{(\forall\opt\in\posopts)\dualopt(\opt)>0}{}
=\posdualopts
% \label{eq:clfs:from:vacuous:desirset}
\\
\ldualopts(\posopts)
&\coloneqq
\cset{\ldualopt\in\ldualopts}
{(\forall\opt\in\posopts)\ldualopt(\opt)>0}
=\posldualopts.
% \label{eq:cslfs:from:vacuous:desirset}
\end{align*}
% \stilltodo{More may be useful or necessary here. For instance, all of such convex cones are intersections of weak*-open and therefore open semi-spaces, and are therefore evenly convex. This points to a fairly complete notion of duality between preference orderings and their representations.}

The following result, which is an immediate consequence of Corollary~\ref{cor:lower:envelope:for:cslfs}, identifies a number of interesting relationships between these sets of functionals.
It tells us that, in some specific sense, the sets of linear functionals~\(\dualopts(\desirset)\) and~\(\posdualopts(\desirset)\) play a more fundamental role than the corresponding sets of superlinear functionals.

\begin{proposition}\label{prop:relationships:between:sets:of:functionals}
Consider a set of desirable options~\(\desirset\in\desirsets\), and a superlinear bounded real functional \(\ldualopt\in\ldualopts\).
Then the following statements hold:
\begin{enumerate}[label=\upshape(\roman*),ref=\upshape(\roman*),leftmargin=*]
\item \(\posdualopts(\desirset)=\dualopts(\desirset\cup\posopts)=\dualopts(\desirset)\cap\posdualopts\);
\item \(\posldualopts(\desirset)=\ldualopts(\desirset\cup\posopts)=\ldualopts(\desirset)\cap\posldualopts\);
\item \(\ldualopt\in\ldualopts(\desirset)\ifandonlyif\dualopts(\ldualopt)\subseteq\dualopts(\desirset)\);
\item \(\ldualopt\in\posldualopts(\desirset)\ifandonlyif\dualopts(\ldualopt)\subseteq\posdualopts(\desirset)\).
\end{enumerate}{}
\end{proposition}

\begin{proof}
We only prove the third statement; the proof of the first two statements is trivial, and the proof of the fourth statement is an immediate consequence of the first three.
Now simply observe the following chain of equivalences:
\begin{align*}
\ldualopt\in\ldualopts(\desirset)
&\ifandonlyif\group{\forall\opt\in\desirset}\ldualopt(\opt)>0
\ifandonlyif\group{\forall\opt\in\desirset}\group{\forall\dualopt\in\dualopts(\ldualopt)}\dualopt(\opt)>0\\
&\ifandonlyif\group{\forall\dualopt\in\dualopts(\ldualopt)}\group{\forall\opt\in\desirset}\dualopt(\opt)>0
\ifandonlyif\group{\forall\dualopt\in\dualopts(\ldualopt)}\dualopt\in\dualopts(\desirset).
\qedhere
\end{align*}
\end{proof}

\begin{runningexample}\label{ex:sets:of:functionals}
Let's go back to our coin example, in order to illustrate a few of the notions introduced in this section.
Taking into account Equation~\eqref{eq:define:the:background:orderings}, we find that, with obvious notations,
\begin{align}
\dualopt\in\dualdifs_{\gblgt0}
&\ifandonlyif\group{\forall(\alpha,\beta)\in\reals^2}
\group[\big]{(\alpha,\beta)\gblgt0\then\alpha\dualopt(\indset{\heads}\unitdif)+\beta\dualopt(\indset{\tails}\unitdif)>0}
\notag\\
&\ifandonlyif\group[\big]{\dualopt(\indset{\heads}\unitdif),\dualopt(\indset{\tails}\unitdif)}\gblpgt0,
\label{eq:positive:dualopts:first}
\end{align}
because we can identify options in~\(\difs\) as well as linear functionals in~\(\dualdifs\) with elements of~\(\reals^2\).
Similarly,
\begin{align}
\dualopt\in\dualdifs_{\gblpgt0}
&\ifandonlyif\group{\forall(\alpha,\beta)\in\reals^2}
\group[\big]{(\alpha,\beta)\gblpgt0\then\alpha\dualopt(\indset{\heads}\unitdif)+\beta\dualopt(\indset{\tails}\unitdif)>0}
\notag\\
&\ifandonlyif\group[\big]{\dualopt(\indset{\heads}\unitdif),\dualopt(\indset{\tails}\unitdif)}\gblgt0.
\label{eq:positive:dualopts:second}
\end{align}
This guarantees that \(\dualdifs_{\gblgt0}\neq\emptyset\) and \(\dualdifs_{\gblpgt0}\neq\emptyset\).
Also, \(\topint(\difs_{\gblgt0})=\topint(\difs_{\gblpgt0})=\difs_{\gblpgt0}\neq\emptyset\); see also Figure~\ref{fig:some:heads:and:tails}.
All of this is in accordance with the assumptions we've made in the beginning of this section.

More generally, we also find that for any set of desirable options~\(\desirset\), after identifying both options in~\(\difs\) as well as linear functionals in~\(\dualdifs\) with elements of~\(\reals^2\)
\begin{align}
\dualopt\in\dualdifs(\desirset)
&\ifandonlyif\group{\forall(\alpha,\beta)\in\desirset}
\alpha\dualopt(\indset{\heads}\unitdif)+\beta\dualopt(\indset{\tails}\unitdif)>0
\notag\\
&\ifandonlyif\group{\forall(\alpha,\beta)\in\desirset}
\group[\big]{\dualopt(\indset{\heads}\unitdif),\dualopt(\indset{\tails}\unitdif)}\cdot(\alpha,\beta)\gblgt0,
\label{eq:positive:dualopts:third}
\end{align}
where `\(\cdot\)' represents the dot product in \(\reals^2\).
These expressions suggest simple geometrical interpretations for the sets~\(\dualdifs(\difs_{\gblgt0})\), \(\dualdifs(\difs_{\gblpgt0})\) and~\(\dualdifs(\desirset)\) as subsets of \(\reals^2\).
We give a few examples in Figure~\ref{fig:dual:cones:for:coins}.
Proposition~\ref{prop:relationships:between:sets:of:functionals} and Corollary~\ref{cor:lower:envelope:for:cslfs} then allow us to see~\(\ldualdifs(\difs_{\gblgt0})\), \(\ldualdifs(\difs_{\gblpgt0})\) and~\(\ldualdifs(\desirset)\) as the collections of convex and compact subsets of these respective subsets~\(\dualdifs(\difs_{\gblgt0})\), \(\dualdifs(\difs_{\gblpgt0})\) and~\(\dualdifs(\desirset)\) of \(\reals^2\).
\stopit
\end{runningexample}

\begin{figure}[h]
\begin{tikzpicture}[scale=.4]\footnotesize
\fill[red!50] (0,4) -- (4,4) -- (4,0) -- (0,0) -- cycle;
\draw[gray,->] (-2.2,0) -- (4.2,0) node[below,red] {\(\dualopt(\indset{\heads}\unitdif)\)};
\draw[gray,->] (0,-2.2) -- (0,4.2) node[above,red] {\(\dualopt(\indset{\tails}\unitdif)\)};
\draw[red,semithick] (0,0) -- (0,4);
\draw[red,semithick] (0,0) -- (4,0);
\node[draw=red,fill=white,circle,inner sep=1pt] at (0,0) {};
\node[white] at (2,2) {\(\dualdifs_{\gblpgt0}\)};
\end{tikzpicture}
\quad
\begin{tikzpicture}[scale=.4]\footnotesize
\fill[red!50] (0,4) -- (4,4) -- (4,0) -- (0,0) -- cycle;
\draw[gray,->] (-2.2,0) -- (4.2,0) node[below,red] {\(\dualopt(\indset{\heads}\unitdif)\)};
\draw[gray,->] (0,-2.2) -- (0,4.2) node[above,red] {\(\dualopt(\indset{\tails}\unitdif)\)};
\draw[red,densely dotted,thick] (0,0) -- (4,0);
\draw[red,densely dotted,thick] (0,0) -- (0,4);
\node[draw=red,fill=white,circle,inner sep=1pt] at (0,0) {};
\node[white] at (2,2) {\(\dualdifs_{\gblgt0}\)};
\end{tikzpicture}
\\
\begin{tikzpicture}[scale=.35]\footnotesize
\fill[blue!50] (4,4) -- (4,-1) -- (0,0) -- (-2,4) -- cycle;
\draw[gray,->] (-2.2,0) -- (4.2,0) node[below right,red] {\(\dualopt(\indset{\heads}\unitdif)\)};
\draw[gray,->] (0,-2.2) -- (0,4.2) node[above,red] {\(\dualopt(\indset{\tails}\unitdif)\)};
\draw[blue,semithick] (0,0) -- (-2,4);
\draw[blue,semithick] (0,0) -- (4,-1);
\node[white] at (3,.5) {\(\desirset\)};
\fill[red!50] (0,0) -- (4,2) -- (4,4) -- (1,4) -- cycle;
\draw[red,densely dotted,thick] (0,0) -- (1,4);
\draw[red,densely dotted,thick] (0,0) -- (4,2);
\node[white] at (2,2) {\(\dualdifs(\desirset)\)};
\node[draw=red,fill=white,circle,inner sep=1pt] at (0,0) {};
\end{tikzpicture}
\quad
\begin{tikzpicture}[scale=.35]\footnotesize
\fill[blue!50] (4,4) -- (4,-2) -- (0,0) -- (-1,4) -- cycle;
\draw[gray,->] (-2.2,0) -- (4.2,0) node[below right,red] {\(\dualopt(\indset{\heads}\unitdif)\)};
\draw[gray,->] (0,-2.2) -- (0,4.2) node[above,red] {\(\dualopt(\indset{\tails}\unitdif)\)};
\draw[blue,densely dotted,thick] (0,0) -- (-1,4);
\draw[blue,densely dotted,thick] (0,0) -- (4,-2);
\node[white] at (2.5,-.5) {\(\desirset\)};
\fill[red!50] (0,0) -- (4,1) -- (4,4) -- (2,4) -- cycle;
\draw[red,semithick] (0,0) -- (2,4);
\draw[red,semithick] (0,0) -- (4,1);
\node[white] at (2.5,2) {\(\dualdifs(\desirset)\)};
\node[draw=red,fill=white,circle,inner sep=1pt] at (0,0) {};
\end{tikzpicture}
\quad
\begin{tikzpicture}[scale=.35]\footnotesize
\fill[blue!50] (4,4) -- (4,-2) -- (1,-2) -- (0,0) -- (-2,4) -- cycle;
\draw[gray,->] (-2.2,0) -- (4.2,0) node[below right,red] {\(\dualopt(\indset{\heads}\unitdif)\)};
\draw[gray,->] (0,-2.2) -- (0,4.2) node[above,red] {\(\dualopt(\indset{\tails}\unitdif)\)};
\draw[blue,densely dotted,thick] (0,0) -- (-2,4);
\draw[blue,densely dotted,thick] (0,0) -- (1,-2);
\node[white] at (3,-1) {\(\desirset\)};
\draw[red,semithick] (0,0) -- (4,2);
\node[red] at (1.8,2) {\(\dualdifs(\desirset)\)};
\node[draw=red,fill=white,circle,inner sep=1pt] at (0,0) {};
\end{tikzpicture}
\caption{Representations (in red) of the sets of linear real functionals~\(\dualdifs(\desirset)\), for various choices of coherent sets of desirable options~\(\desirset\) (in blue). Each set of desirable options~\(\desirset\) is graphically represented by the values~\(\dif(\bolleke,\bestreward)\) that its elements~\(\dif\in\desirset\) assume in the reward~\(\bestreward\). Each set of linear real functionals~\(\dualdifs(\desirset)\) is graphically represented by the values~\((\dualopt(\indset{\heads}\unitdif),\dualopt(\indset{\tails}\unitdif))\) that its elements~\(\dualopt\in\dualdifs(\desirset)\) assume in the options~\(\indset{\heads}\unitdif\) and~\(\indset{\tails}\unitdif\). Full blue or red lines indicate `borders' that are included in the sets, dotted lines represent `borders' that aren't.}
\label{fig:dual:cones:for:coins}
\end{figure}

Inspired by Walley's~\cite{walley1991} discussion of `strict desirability', I'll call a set of desirable options~\(\desirset\in\desirsets\) \emph{\essarchim} if it is coherent and open.

It turns out that there is a close connection between {\essarchim} sets of desirable options and positive superlinear bounded real functionals.
Before we can lay it bare in Propositions~\ref{prop:ldualopt:from:desirset}--\ref{prop:isomorphisms:arch:noconstants}, we need to find a way to associate a superlinear bounded real functional with a set of desirable options~\(\desirset\in\desirsets\).

There are a number of different ways to achieve this, but I've found the following approach to be especially productive.
Since we assumed from the outset that \(\topint(\posopts)\neq\emptyset\), we can fix any~\(\opt[o]\in\topint(\posopts)\).
We use this special option~\(\opt[o]\) to associate with a set of desirable options~\(\desirset\) a specific (possibly extended) real functional~\(\ldualopto\colon\opts\to\reals\cup\set{-\infty,+\infty}\) by letting
\begin{equation}\label{eq:ldualopt:from:desirset}
\ldualopto(\opt)\coloneqq\sup\cset{\alpha\in\reals}{\opt-\alpha\opt[o]\in\desirset}
\text{ for all~\(\opt\in\opts\)},
\end{equation}
and, for its conjugate functional,
\begin{equation*}
\udualopto(\opt)\coloneqq-\ldualopto(-\opt)=\inf\cset{\beta\in\reals}{\beta\opt[o]-\opt\in\desirset}
\text{ for all~\(\opt\in\opts\)}.
\end{equation*}

\begin{proposition}\label{prop:ldualopt:from:desirset}
If the set of desirable options \(\desirset\) is coherent, then \(\ldualopto\in\ldualopts\).
Moreover, \(\ldualopto(\opt)\geq0\) for all~\(\opt\in\desirset\) and \(\ldualopto(\altopt)\leq0\) for all~\(\altopt\in\co{\desirset}\).
\end{proposition}

\begin{proof}
% Checked by Gert
We begin by proving that \(\ldualopto\leq\udualopto\).
Indeed, assume {\itshape ex absurdo} that there is some~\(\opt\in\opts\) for which \(\ldualopto(\opt)>\udualopto(\opt)\), so there are real~\(\alpha>\beta\) such that \(\opt-\alpha\opt[o]\in\desirset\) and \(\beta\opt[o]-\opt\in\desirset\).
We then infer from coherence [Axiom~\ref{ax:desirs:cone}] that \((\beta-\alpha)\opt[o]=(\opt-\alpha\opt[o])+(\beta\opt[o]-\opt)\in\desirset\), which is impossible because \(\beta<\alpha\) and \(\opt[o]\in\desirset\) [indeed, Axiom~\ref{ax:desirs:cone} would then imply that \(0=(\beta-\alpha)\opt[o]+(\alpha-\beta)\opt[o]\in\desirset\), contradicting Axiom~\ref{ax:desirs:nozero}].

Next, we show that \(\ldualopto\) is positively homogeneous.
Consider any~\(\opt\in\opts\) and any real~\(\lambda>0\).
By the coherence of~\(\desirset\) [Axiom~\ref{ax:desirs:cone}], \(\lambda\opt-\alpha\opt[o]\in\desirset\) if and only if \(\opt-\frac{\alpha}{\lambda}\opt[o]\in\desirset\).
This guarantees that, indeed, \(\ldualopto(\lambda\opt)=\lambda\ldualopto(\opt)\).
It follows immediately that \(\udualopto\) is positively homogeneous too.

Next, we show that \(\udualopto\)---and therefore also \(\ldualopto\)---is bounded, and therefore real-valued.
Assume {\itshape ex absurdo} that it isn't, so \(\ddualoptnorm{\udualopto}=+\infty\), implying that for any~\(n\in\naturals\), there is some~\(\opt[n]\in\opts\setminus\set{0}\) such that \(\abs{\udualopto(\opt[n])}/\optnorm{\opt[n]}>n\).
Due to the positive homogeneity of~\(\udualopto\) we've just proved, we may assume without loss of generality that \(\optnorm{\opt[n]}=1\), so we find that \(\udualopto(\opt[n])>n\) or \(\udualopto(\opt[n])<-n\).
The latter alternative is equivalent to \(\ldualopto(-\opt[n])>n\), and therefore implies that also \(\udualopto(-\opt[n])>n\), because we've just proved that \(\udualopto\geq\ldualopto\).
We may therefore assume without loss of generality that \(\udualopto(\opt[n])>n\), and therefore that \(n\opt[o]-\opt[n]\in\co{\desirset}\), or equivalently [by coherence, Axiom~\ref{ax:desirs:cone}] that \(\opt[o]-\frac{\opt[n]}{n}\in\co{\desirset}\) for all~\(n\in\naturals\).
But \(\frac{\opt[n]}{n}\to0\), so we are led to the conclusion that \(\opt[o]\in\topcls(\co{\desirset})=\co{\topint(\desirset)}\).
This contradicts that \(\opt[o]\in\topint(\posopts)\subseteq\topint(\desirset)\) [use Axiom~\ref{ax:desirs:pos}].

To show that \(\ldualopto\) is non-negatively homogeneous, it now simply suffices to prove that \(\ldualopto(0)=0\).
Since \(0-\mu\opt[o]\notin\desirset\) for~\(\mu\geq0\) and at the same time \(0-\mu\opt[o]\in\desirset\) for~\(\mu<0\), by coherence [use Axioms~\ref{ax:desirs:cone} and~\ref{ax:desirs:nozero}], this is immediate from Equation~\eqref{eq:ldualopt:from:desirset}.

For the superadditivity of~\(\ldualopto\), consider any~\(\opt,\altopt\in\opts\).
Consider any~\(\alpha,\beta\in\reals\) such that both \(\alpha<\ldualopto(\opt)\) and \(\beta<\ldualopto(\altopt)\), so \(\opt-\alpha\opt[o]\in\desirset\) and \(\altopt-\beta\opt[o]\in\desirset\).
But then \(\opt+\altopt-(\alpha+\beta)\opt[o]\in\desirset\) by the coherence of~\(\desirset\) [Axiom~\ref{ax:desirs:cone}], implying that \(\alpha+\beta\leq\ldualopto(\opt+\altopt)\).
Hence, indeed, \(\ldualopto(\opt+\altopt)\geq\ldualopto(\opt)+\ldualopto(\altopt)\).

Finally, for any~\(\altopt\notin\desirset\), we have that \(\altopt-\alpha\opt[o]\notin\desirset\) for all~\(\alpha\geq0\), by coherence [Axioms~\ref{ax:desirs:cone} and~\ref{ax:desirs:nozero}], and therefore indeed \(\ldualopto(\altopt)\leq0\).
And for any~\(\opt\in\desirset\), we get that \(\opt-\alpha\opt[o]\in\desirset\) for all~\(\alpha\leq0\), by coherence [Axiom~\ref{ax:desirs:cone}], and therefore indeed \(\ldualopto(\opt)\geq0\).
\end{proof}

\begin{proposition}\label{prop:ldualopto:and:openness}
If the set of desirable options~\(\desirset\) is coherent, then \(\desirset[\ldualopto]=\topint(\desirset)\).
\end{proposition}

\begin{proof}
% Checked by Gert
The coherence of~\(\desirset\) and Proposition~\ref{prop:ldualopt:from:desirset} guarantee that \(\ldualopto\in\ldualopts\).
It therefore suffices to show that for any~\(\opt\in\opts\), \(\opt\in\topint(\desirset)\ifandonlyif\ldualopto(\opt)>0\).

First, assume that \(\opt\notin\topint(\desirset)\), then since \(\co{\topint(\desirset)}=\topcls(\co{\desirset})\), there is some sequence~\(\opt[n]\in\co{\desirset}\) such that \(\opt[n]\to\opt\), and therefore also \(\ldualopto(\opt[n])\to\ldualopto(\opt)\), by the continuity of~\(\ldualopto\) [use \(\ldualopto\in\ldualopts\) and Proposition~\ref{prop:ludualopt:continuous}].
Since \(\ldualopto(\opt[n])\leq0\) by Proposition~\ref{prop:ldualopt:from:desirset}, we see that, indeed, \(\ldualopto(\opt)\leq0\).

Next, assume that \(\opt\in\topint(\desirset)\), so \(B(\opt,\delta)\coloneqq\cset{\altopt\in\opts}{\optnorm{\altopt-\opt}<\delta}\subseteq\desirset\) for some real~\(\delta>0\).
Let \(\opt[\epsilon]\coloneqq\opt-\epsilon\opt[o]\) for all real~\(\epsilon>0\), then \(\optnorm{\opt[\epsilon]-\opt}=\optnorm{-\opt[o]\epsilon}=\epsilon\optnorm{\opt[o]}\), so if we pick any~\(0<\epsilon<\delta/\optnorm{\opt[o]}\), we find that \(\opt[\epsilon]\in B(\opt,\delta)\) and therefore \(\opt[\epsilon]\in\desirset\), whence
\begin{equation*}
0
\leq\ldualopto(\opt[\epsilon])
=\ldualopto(\opt-\epsilon\opt[o])
=\ldualopto(\opt)-\epsilon,
\end{equation*}
where the inequality follows from Proposition~\ref{prop:ldualopt:from:desirset}, and the last equality from Equation~\eqref{eq:ldualopt:from:desirset}.
Hence, indeed, \(\ldualopto(\opt)\geq\epsilon>0\).
\end{proof}

\begin{proposition}[Representation]\label{prop:isomorphisms:arch:noconstants}
A set of desirable options~\(\desirset\in\desirsets\) is {\essarchim} if and only if there is some~\(\ldualopt\in\posldualopts\) such that \(\desirset=\desirset[\ldualopt]\).
In that case, we always have that \(\desirset=\desirset[\ldualopto]\), and therefore in particular also \(\ldualopto\in\posldualopts(\desirset)\).
\end{proposition}

\begin{proof}
% Checked by Gert
For sufficiency, we need to consider any~\(\ldualopt\in\posldualopts\) and prove that \(\desirset[\ldualopt]\) is {\essarchim}.
That \(\desirset[\ldualopt]\) is open, follows at once from the continuity of~\(\ldualopt\), so we only need to focus on coherence.
Since Axiom~\ref{ax:superlinear:superadditive} implies that \(\ldualopt(0)+\ldualopt(0)\leq\ldualopt(0)\), and since \(\ldualopt\) is real-valued, we know that \(\ldualopt(0)\leq0\), so Equation~\eqref{eq:desirset:from:cslf} implies that \(0\notin\desirset[\ldualopt]\).
Hence, \(\desirset[\ldualopt]\) satisfies Axiom~\ref{ax:desirs:nozero}.
For Axiom~\ref{ax:desirs:pos}, it suffices to observe that for any~\(\opt\in\posopts\), \(\ldualopt(\opt)>0\) because \(\dualopt\in\posldualopts\), so \(\opt\in\desirset[\ldualopt]\) by Equation~\eqref{eq:desirset:from:cslf}.
Let's now turn to Axiom~\ref{ax:desirs:cone}.
Consider any~\(\opt,\altopt\in\desirset[\ldualopt]\) and any~\((\lambda,\mu)>0\).
It then follows from Equation~\eqref{eq:desirset:from:cslf} that \(\ldualopt(\opt)>0\) and \(\ldualopt(\altopt)>0\), and from Axioms~\ref{ax:superlinear:homogeneous} and~\ref{ax:superlinear:superadditive} that \(\ldualopt(\lambda\opt+\mu\altopt)\geq\lambda\ldualopt(\opt)+\mu\ldualopt(\altopt)\).
Since \((\lambda,\mu)>0\), this implies that \(\ldualopt(\lambda\opt+\mu\altopt)>0\), which in turn implies that \(\lambda\opt+\mu\altopt\in\desirset[\ldualopt]\) by Equation~\eqref{eq:desirset:from:cslf}.
So \(\desirset[\ldualopt]\) satisfies Axiom~\ref{ax:desirs:cone}.

To prove necessity, we consider any {\essarchim} \(\desirset\), so \(\desirset\) is both coherent and open.
Proposition~\ref{prop:ldualopto:and:openness} then guarantees that \(\desirset=\topint(\desirset)=\desirset[\ldualopto]\).
Also recall that, by coherence [Axiom~\ref{ax:desirs:pos}], \(\posopts\subseteq\desirset\), and therefore also \(\ldualopto(\opt)>0\) for all~\(\opt\optgt0\).
Hence also \(\ldualopto\in\posldualopts\), which completes the proof of the `if and only if' statement.
For the rest of the proof, assume that there is some~\(\ldualopt\in\ldualopts\) such that \(\desirset=\desirset[\ldualopt]\).
But then \(\desirset\) is open, because \(\ldualopt\) is continuous, and therefore \(\desirset=\topint(\desirset)\).
Since Proposition~\ref{prop:ldualopto:and:openness} guarantees that also \(\desirset[\ldualopto]=\topint(\desirset)\), we are done.
\end{proof}

For sets of desirable options that are {\essarchim} and mixing, we have similar results in terms of \emph{linear} rather than superlinear bounded real functionals.

\begin{proposition}\label{prop:dualopt:from:desirset}
If the set of desirable options~\(\desirset\) is mixing, then \(\ldualopto\in\dualopts\), and we'll then denote this linear bounded real functional by~\(\dualopto\).
Moreover, \(\dualopto(\opt)\geq0\) for all~\(\opt\in\desirset\) and \(\dualopto(\altopt)\leq0\) for all~\(\altopt\in\co{\desirset}\).
\end{proposition}

\begin{proof}
% Checked by Gert
We only need to prove, by Proposition~\ref{prop:ldualopt:from:desirset}, that \(\ldualopto\) is linear, which amounts to showing that it's also subadditive.
So for any~\(\opt,\altopt\in\opts\), we show that \(\ldualopto(\opt+\altopt)\leq\ldualopto(\opt)+\ldualopto(\altopt)\).
Fix any~\(\mu\in\reals\) such that \(\mu>\ldualopto(\opt)+\ldualopto(\altopt)\eqqcolon\lambda\) and let \(\epsilon\coloneqq\nicefrac{1}{2}\group{\mu-\lambda}>0\).
Then \(\opt[\epsilon]\coloneqq\opt-(\ldualopto(\opt)+\epsilon)\opt[o]\notin\desirset\) and \(\altopt[\epsilon]\coloneqq\altopt-(\ldualopto(\altopt)+\epsilon)\opt[o]\notin\desirset\), by Equation~\eqref{eq:ldualopt:from:desirset}.
So if we let \(\optset\coloneqq\{\opt[\epsilon],\altopt[\epsilon]\}\), then \(\optset\cap\desirset=\emptyset\).
Since \(\{\opt[\epsilon],\altopt[\epsilon],\opt[\epsilon]+\altopt[\epsilon]\}\subseteq\posi(\optset)\), it follows from the mixingness of~\(\desirset\) [Axiom~\ref{ax:desirs:mixing} and contraposition] that also \(\{\opt[\epsilon],\altopt[\epsilon],\opt[\epsilon]+\altopt[\epsilon]\}\cap\desirset=\emptyset\).
Hence, we find that
\begin{equation*}
\opt+\altopt-\mu\opt[o]
=\opt+\altopt-(\lambda+2\epsilon)\opt[o]
=\opt[\epsilon]+\altopt[\epsilon]\notin\desirset.
\end{equation*}
Since this holds for all real~\(\mu>\lambda\), it follows from Equation~\eqref{eq:ldualopt:from:desirset} that \(\ldualopto(\opt+\altopt)\leq\lambda\), as desired.
\end{proof}

\begin{proposition}[Representation]\label{prop:isomorphisms:arch:mixing:noconstants}
A set of desirable options~\(\desirset\in\desirsets\) is {\essarchim} and mixing if and only if there is some~\(\dualopt\in\posdualopts\) such that \(\desirset=\desirset[\dualopt]\).
In that case, we always have that \(\desirset=\desirset[\dualopto]\), and therefore in particular also \(\dualopto\in\posdualopts(\desirset)\).
\end{proposition}

\begin{proof}
% Checked by Gert
For sufficiency, we have to consider any~\(\dualopt\in\posdualopts\), and prove that \(\desirset[\dualopt]\) is {\essarchim} and mixing.
Since also \(\dualopt\in\posldualopts\), we infer from Proposition~\ref{prop:isomorphisms:arch:noconstants} that \(\desirset[\dualopt]\) is {\essarchim}, so it only remains to show that it is mixing.
So let us consider any~\(\optset\in\optsets\) and assume that \(\posi(\optset)\cap\desirset[\dualopt]\neq\emptyset\), then we must prove that \(\optset\cap\desirset[\dualopt]\neq\emptyset\).
That \(\posi(\optset)\cap\desirset[\dualopt]\neq\emptyset\) means that there are~\(n\in\naturals\), \(\opt[k]\in\optset\) and real~\(\lambda_k>0\) such that \(0<\dualopt\group{\sum_{k=1}^n\lambda_k\opt[k]}=\sum_{k=1}^n\lambda_k\dualopt(\opt[k])\).
This indeed implies that there is some~\(\opt[k]\in\optset\) for which \(\dualopt(\opt[k])>0\).

The rest of the proof is immediate if we combine Propositions~\ref{prop:isomorphisms:arch:noconstants} and~\ref{prop:dualopt:from:desirset}.
\end{proof}

\begin{runningexample}\label{ex:essential:archimedeanity}
Going back to our coin example, we now intend to show that both the coherent sets of desirable options~\(\desirset[\heads]\) and~\(\desirset[\tails]\), introduced in its instalment~\ref{ex:assessment:identical:sides}, are {\essarchim} and mixing.
A related argument will show that the set of desirable options~\(\desirset[I]\) from instalment~\ref{ex:assessment:lower:upper:betting:rates} is {\essarchim}, but only mixing if~\(\lp=\up\).

Indeed, if we recall from instalment~\ref{ex:functionals} the definition in Equation~\eqref{eq:define:the:functionals} of the linear bounded real functional~\(\dualopt[p]\in\dualopts\) and the superlinear bounded real functional~\(\ldualopt[I]\in\ldualopts\) on~\(\difs\), then it follows from the equivalence in Equation~\eqref{eq:coin:imprecise:desirset} that \(\desirset[I]=\desirset[{\ldualopt[I]}]\), and therefore also \(\desirset[p]=\desirset[{\dualopt[p]}]\) for \(0<\lp\leq p\leq\up<1\); see Figure~\ref{fig:essential:archimedeanity}.
When~\(\lp=p=0\) or~\(p=\up=1\), similar conclusions hold only when we consider~\(\gblpgt\) as a background ordering.
In all those cases, Proposition~\ref{prop:isomorphisms:arch:noconstants} guarantees that~\(\desirset[I]=\desirset[{\ldualopt[I]}]\) is {\essarchim}, and Proposition~\ref{prop:isomorphisms:arch:mixing:noconstants} guarantees that~\(\desirset[p]=\desirset[{\ldualopt[p]}]\) is {\essarchim} and mixing.
This holds in particular also for~\(\desirset[\heads]=\desirset[{\dualopt[1]}]\) (with~\(\lp=p=\up=1\)) and~\(\desirset[\tails]=\desirset[{\dualopt[0]}]\) (with~\(\lp=p=\up=0\)) when we consider~\(\gblpgt\) as a background ordering.
\stopit
\end{runningexample}

\begin{figure}[h]
\begin{tikzpicture}[scale=.5]\footnotesize
\fill[blue!50] (4,4) -- (4,-2) -- (0,0) -- (-1,4) -- cycle;
\draw[gray,->] (-2.2,0) -- (4.2,0) node[right,red] {\(\dualopt(\indset{\heads}\unitdif)\)};
\draw[gray,->] (0,-2.2) -- (0,4.2) node[above,red] {\(\dualopt(\indset{\tails}\unitdif)\)};
\draw[blue,densely dotted,thick] (0,0) -- (-1,4) node[sloped,midway,below] {\(\ldualopt[I]=0\)};
\draw[blue,densely dotted,thick] (0,0) -- (4,-2) node[sloped,midway,below] {\(\ldualopt[I]=0\)};
\node[white] at (1,3.5) {\(\desirset[I]\)};
\fill[red!50] (0,0) -- (4,1) -- (4,4) -- (2,4) -- cycle;
\draw[red,semithick] (0,0) -- (2,4);
\draw[red,semithick] (0,0) -- (4,1);
\node[white] at (3,2.5) {\(\dualdifs(\desirset[I])\)};
\node[draw=red,fill=white,circle,inner sep=1pt] at (0,0) {};
\draw[white,dashed] (3,0) -- (0,3);
\draw[white,thick] (2.4,.6) -- node[midway,below left=-3pt] {\(C_{\ldualopt[I]}\)} (1,2);
\node[fill=white,circle,inner sep=1pt] at (1,2) {};
\node[fill=white,circle,inner sep=1pt] at (2.4,.6) {};
\draw[white] (3,.1) -- (3,-.1) node[below,white] {\(1\)};
\draw[white] (.1,3) -- (-.1,3) node[left,white] {\(1\)};
\node[blue,right] at (4,-1) {\(\ldualopt[I]>0\)};
\node[blue] at (-1,-1) {\(\ldualopt[I]<0\)};
\end{tikzpicture}
\caption{Representation (in red) of the set of linear bounded real functionals~\(\dualdifs(\desirset[I])\) for the {\essarchim} sets of desirable options~\(\desirset[I]\) (in blue), represented by the values~\(\dif(\bolleke,\bestreward)\) that its elements~\(\dif\in\desirset[I]\) assume in the reward~\(\bestreward\). The set of linear real functionals~\(\dualdifs(\desirset[I])\) is represented by the values~\((\dualopt(\indset{\heads}\unitdif),\dualopt(\indset{\tails}\unitdif))\) that its elements~\(\dualopt\in\dualdifs(\desirset[I])\) assume in the options~\(\indset{\heads}\unitdif\) and~\(\indset{\tails}\unitdif\). Full blue or red lines indicate `borders' that are included in the sets, dotted lines represent `borders' that aren't. The superlinear bounded real functional \(\ldualopt[I]\) divides the option space \(\difs\) in three regions according to its sign. Also indicated is the representation \(C_{\ldualopt[I]}\) of the set of linear bounded real functionals \(\dualdifs(\ldualopt[I])\).}
\label{fig:essential:archimedeanity}
\end{figure}

\section{Normalisation}\label{sec:normalisation}
Proposition~\ref{prop:isomorphisms:arch:noconstants} has interesting implications, and it will be helpful to pay more attention to them, in order to achieve a better understanding of what we are actually doing in Propositions~\ref{prop:ldualopt:from:desirset}--\ref{prop:isomorphisms:arch:mixing:noconstants} above.
The {\essarchim} sets of desirable options~\(\desirset\) are all those, and only those, for which there is some superlinear bounded real functional~\(\ldualopt\in\posldualopts\) such that \(\desirset=\desirset[\ldualopt]\).
But Proposition~\ref{prop:isomorphisms:arch:noconstants} also guarantees that in this representation~\(\desirset[\ldualopt]\) for~\(\desirset\), the superlinear bounded real functional~\(\ldualopt\in\posldualopts\) can always be replaced by the superlinear bounded real functional~\(\ldualopto\in\posldualopts\), as we know that
\begin{equation*}
\desirset=\desirset[\ldualopt]=\desirset[\ldualopto].
\end{equation*}
The import of all this is that we can associate, with any~\(\opt[o]\in\topint(\posopts)\), the following so-called \emph{normalisation} transformation of~\(\ldualopts\):
\begin{equation*}
\nml\colon\ldualopts\to\ldualopts\colon\ldualopt\mapsto\nml\ldualopt\coloneqq\ldualopto[{\desirset[\ldualopt]}]
\end{equation*}
where, after a few elementary manipulations, we get
\begin{equation}\label{eq:normalisation:map}
\nml\ldualopt(\opt)=\sup\cset{\alpha\in\reals}{\ldualopt(\opt-\alpha\opt[o])>0}
\text{ for all~\(\opt\in\opts\)}.
\end{equation}

We list a few properties of the normalisation~\(\nml\) below.
They show, amongst other things, that \(\nml\) converts any positive bounded superlinear functional~\(\ldualopt\) to a version~\(\nml\ldualopt\) that is normalised and constant additive with respect to the option~\(\opt[o]\).
Moreover, it's the purport of Proposition~\ref{prop:isomorphisms:arch:noconstants} that if \(\ldualopt\) represents an {\essarchim}~\(\desirset\) in the sense that \(\desirset=\desirset[\ldualopt]\), then so does the version~\(\nml\ldualopt\), in the sense that also \(\desirset=\desirset[\nml\ldualopt]\).
Finally, the normalisation~\(\nml\) is also internal in the set of positive bounded linear functionals~\(\posdualopts\): if \(\dualopt\in\posdualopts\) then also \(\nml\dualopt\in\posdualopts\).

\begin{proposition}\label{prop:normalisation:map:properties}
Consider any~\(\opt[o]\in\topint(\posopts)\) and let \(\desirset\in\desirsets\) be any set of desirable options.
Then for all~\(\ldualopt\in\posldualopts\) and~\(\dualopt\in\posdualopts\), the following statements hold:
\begin{enumerate}[label=\upshape(\roman*),leftmargin=*]
\item\label{it:normalisation:map:properties:preserves:open:cone} \(\desirset[\ldualopt]=\desirset[\nml\ldualopt]\);\hfill\textup{[preserves open cone]}
\item\label{it:normalisation:map:properties:internality} if \(\ldualopt\in\posldualopts(\desirset)\) then \(\nml\ldualopt\in\posldualopts(\desirset)\);\hfill\textup{[internality]}
\item\label{it:normalisation:map:properties:normalisation} \(\nml\ldualopt(\opt[o])=-\nml\ldualopt(-\opt[o])=1\);\hfill\textup{[normalisation]}
\item\label{it:normalisation:map:properties:constant:additivity} \(\nml\ldualopt(\opt+\mu\opt[o])=\nml\ldualopt(\opt)+\mu\) for all~\(\opt\in\opts\) and~\(\mu\in\reals\);\hfill\textup{[constant additivity]}
\item\label{it:normalisation:map:properties:fixpoint} if \(\ldualopt(\opt+\mu\opt[o])=\ldualopt(\opt)+\mu\) for all~\(\opt\in\opts\) and~\(\mu\in\reals\), then \(\nml\ldualopt=\ldualopt\);\hfill\textup{[fixpoints]}
\item\label{it:normalisation:map:properties:idempotency} \(\nml(\nml\ldualopt)=\nml\ldualopt\);\hfill\textup{[idempotency]}
\item\label{it:normalisation:map:properties:fixed:point} \(\nml\ldualopto=\ldualopto\);\hfill\textup{[special fixpoint]}
\item\label{it:normalisation:map:properties:linear:internality} if \(\dualopt\in\posdualopts(\desirset)\) then \(\nml\dualopt\in\posdualopts(\desirset)\);\hfill\textup{[linear internality]}%
\item\label{it:normalisation:map:properties:linear:normalisation} \(\nml\dualopt=\dualopt/\dualopt(\opt[o])\);\hfill\textup{[linear normalisation]}%
% \stilltodo{\item\label{it:normalisation:map:properties:norm} \(\ddualoptnorm{\nml\ldualopt}=1/\optnorm{\opt[o]}\)?\hfill\textup{[norm]}}
\end{enumerate}
\end{proposition}

\begin{proof}
% Checked by Gert
First of all, observe that \(\opt[o]\in\topint(\posopts)\) implies that \(\ldualopt(\opt[o])>0\) for all~\(\ldualopt\in\posldualopts\) and \(\dualopt(\opt[o])>0\) for all~\(\dualopt\in\posdualopts\).
Statement~\ref{it:normalisation:map:properties:preserves:open:cone} is an immediate consequence of the definition of the normalisation~\(\nml\), which is enabled by Proposition~\ref{prop:isomorphisms:arch:noconstants}.
It leads trivially to~\ref{it:normalisation:map:properties:internality}, and the linear counterpart~\ref{it:normalisation:map:properties:linear:internality} follows directly from Proposition~\ref{prop:isomorphisms:arch:mixing:noconstants}.
For~\ref{it:normalisation:map:properties:normalisation}, consider that
\begin{equation*}
\nml\ldualopt(\opt[o])=\sup\cset{\alpha\in\reals}{\ldualopt\group{(1-\alpha)\opt[o]}>0},
\end{equation*}
and recall that \(\ldualopt\in\posldualopts\) implies that \(\ldualopt\group{(1-\alpha)\opt[o]}>0\ifandonlyif\alpha<1\).
Similarly, consider that
\begin{equation*}
\nml\ldualopt(-\opt[o])=\sup\cset{\alpha\in\reals}{\ldualopt\group{-(1+\alpha)\opt[o]}>0},
\end{equation*}
and recall that \(\ldualopt\in\posldualopts\) implies that \(\ldualopt\group{-(1+\alpha)\opt[o]}>0\ifandonlyif\alpha<-1\).
For~\ref{it:normalisation:map:properties:constant:additivity}, observe that
\begin{align*}
\nml\ldualopt(\opt+\mu\opt[o])
&=\sup\cset{\alpha\in\reals}{\ldualopt\group{\opt-(\alpha-\mu)\opt[o]}>0}\\
&=\sup\cset{\beta+\mu\in\reals}{\ldualopt\group{\opt-\beta\opt[o]}>0}
=\nml\ldualopt(\opt)+\mu.
\end{align*}
For~\ref{it:normalisation:map:properties:fixpoint}, consider that it follows from the constant additivity assumption that, indeed,
\begin{equation*}
\nml\ldualopt(\opt)
=\sup\cset{\alpha\in\reals}{\ldualopt\group{\opt-\alpha\opt[o]}>0}
=\sup\cset{\alpha\in\reals}{\ldualopt(\opt)>\alpha}
=\ldualopt(\opt)
\end{equation*}
for all~\(\opt\in\opts\).
For~\ref{it:normalisation:map:properties:idempotency} it's enough to combine~\ref{it:normalisation:map:properties:fixpoint} and~\ref{it:normalisation:map:properties:constant:additivity}.
Similarly, for~\ref{it:normalisation:map:properties:fixed:point} it suffices to observe that the definition of~\(\ldualopto\) guarantees that it is constant additive, and to apply~\ref{it:normalisation:map:properties:fixpoint}.
For~\ref{it:normalisation:map:properties:linear:normalisation}, take into account that \(\dualopt(\opt-\alpha\opt[0])>0\ifandonlyif\dualopt(\opt)>\alpha\dualopt(\opt[o])\), and that \(\dualopt(\opt[o])>0\).
% \stilltodo{Give the actual proof for~\ref{it:normalisation:map:properties:norm}.}
\end{proof}

\begin{runningexample}\label{ex:normalisation}
In our coin example, we see that \(\unitdif=\indset{\heads}\unitdif+\indset{\tails}\unitdif\) and therefore, using the notations from  instalment~\ref{ex:functionals},
\begin{equation*}
\difnorm{\dif-\unitdif}
=\max\set{\abs{\dif(\heads,\bestreward)-1},\abs{\dif(\tails,\bestreward)-1}}
<\nicefrac12
\ifandonlyif
\nicefrac12\gblplt\dif(\bolleke,\bestreward)\gblplt\nicefrac32,
\end{equation*}
which implies that the unit ball around~\(\unitdif\) with radius~\(\nicefrac12\) lies entirely inside the positive cones~\(\difs_{\gblpgt0}\) and~\(\difs_{\gblgt0}\), and therefore \(\unitdif\in\topint(\difs_{\gblpgt0})=\topint(\difs_{\gblgt0})\).

Consider the linear bounded real functional~\(\dualopt[p]\in\dualopts\) and the superlinear bounded real functional~\(\ldualopt[I]\in\ldualopts\) on~\(\difs\), defined in Equation~\eqref{eq:define:the:functionals} of instalment~\ref{ex:functionals}, then we find that
\begin{multline*}
\ldualopt[I](\dif+\mu\unitdif)
=\lex[I](\dif(\bolleke,\bestreward)+\mu)
=\lex[I](\dif(\bullet,\bestreward))+\mu
=\ldualopt[I](\dif)+\mu\\
\text{ for all \(\dif\in\difs\) and all~\(\mu\in\reals\)}.
\end{multline*}
Proposition~\ref{prop:normalisation:map:properties}\ref{it:normalisation:map:properties:fixpoint} then guarantees that \(\nml[\unitdif]\ldualopt[I]=\ldualopt[I]\), and similarly, \(\nml[\unitdif]\dualopt[p]=\dualopt[p]\), so these functionals are already `normalised'.
\stopit
\end{runningexample}

\section{{\archimty} for sets of desirable options}\label{sec:archimedanity:binary}
One of the drawbacks of working with {\essarchim} sets of desirable options in Section~\ref{sec:essential:archimedeanity}, is that they don't constitute an intersection structure---and therefore don't come with a conservative inference method: an arbitrary intersection of {\essarchim} sets of desirable options is no longer necessarily {\essarchim}, simply because openness is not necessarily preserved under arbitrary intersections.

In order to remedy this, we now turn to arbitrary intersections of {\essarchim} models, which of course do constitute an intersection structure.
We'll see that these types of models also allow for a very elegant and general representation, and are related to the evenly continuous models described by Cozman~\cite{cozman2018:evenlyconvex} in the more restrictive context of finite-dimensional gamble spaces.

We'll call a set of desirable options~\(\desirset\in\cohdesirsets\) \emph{{\archim}} if it is coherent and if the following separation property is satisfied; see also Figure~\ref{fig:separation}:\footnote{Observe, by the way, that if \(\desirset\) is coherent, then \(\ldualopt\in\ldualopts(\desirset)\) implies that \(\posopts\subseteq\desirset\subseteq\desirset[\ldualopt]\), and therefore also \(\ldualopt\in\posldualopts\). This implies that \(\ldualopts(\desirset)=\posldualopts(\desirset)\). Similarly, \(\dualopts(\desirset)=\posdualopts(\desirset)\).}
\begin{enumerate}[label=\({\mathrm{D}}_{\mathrm{A}}\).,ref=\({\mathrm{D}}_{\mathrm{A}}\),leftmargin=*]
\item\label{ax:desirset:essentially:archimedean:cslf} \((\forall\opt\notin\desirset)(\exists\ldualopt\in\ldualopts(\desirset))\ldualopt(\opt)\leq0\),
\end{enumerate}
and we denote by~\(\archdesirsets\) the set of all {\archim} sets of desirable options.
\begin{figure}[h]
\centering
\begin{tikzpicture}[scale=.6]\small
\path[fill=blue!75] (0,0) circle [x radius=1, y radius=2, rotate=30];
\node[white] at (0,0) {\(\desirset\)};
\node[fill=blue,circle,inner sep=1pt] at (-1.5,1.5) {};
\node[left,blue] at (-1.5,1.5) {\(\opt\)};
\draw[red, thick] (-1.5,-2) -- (-1.3,3) node[above] {\(\ldualopt=0\)};
\node[red,left] at (-1.5,-1) {\(\ldualopt<0\)};
\node[red,right] at (-1.3,2.25) {\(\ldualopt>0\)};
\end{tikzpicture}
\caption{Separation property for {\archimty}: every option~\(\opt\) outside~\(\desirset\) can be separated from it by a hypersurface \(\ldualopt=0\) with \(\ldualopt>0\) on~\(\desirset\) and \(\ldualopt(\opt)\leq0\).}
\label{fig:separation}
\end{figure}
\par
It is a fairly immediate consequence the Hahn--Banach theorem version for superlinear bounded real functionals [Theorem~\ref{theo:hahn:banach}] that this separation property is equivalent to
\begin{enumerate}[label=\({\mathrm{D}}_{\mathrm{A}}^{\mathrm{p}}\).,ref=\({\mathrm{D}}_{\mathrm{A}}^{\mathrm{p}}\),leftmargin=*]
\item\label{ax:desirset:essentially:archimedean:clf} \((\forall\opt\notin\desirset)(\exists\dualopt\in\dualopts(\desirset))\dualopt(\opt)\leq0\),
\end{enumerate}
which shows that the {\archim} sets of desirable option sets are in particular also \emph{evenly convex}; see for instance~\cite[Definition~1]{daniilidis2002:even:convexity}.

\begin{proof}
That \({\mathrm{D}}_{\mathrm{A}}^{\mathrm{p}}\then{\mathrm{D}}_{\mathrm{A}}\) is immediate, because \(\dualopts\subseteq\ldualopts\) and, therefore, \(\dualopts(\desirset)\subseteq\ldualopts(\desirset)\).
For the converse implication, namely \({\mathrm{D}}_{\mathrm{A}}\then{\mathrm{D}}_{\mathrm{A}}^{\mathrm{p}}\), fix any~\(\opt\in\co{\desirset}\), then we know from the assumption~\({\mathrm{D}}_{\mathrm{A}}\) that there is some~\(\ldualopt\in\ldualopts\) such that \(\desirset\subseteq\desirset[\ldualopt]\) and \(\ldualopt(\opt)\leq0\).
When we now invoke Theorem~\ref{theo:hahn:banach}, we find that there is some \(\dualopt\in\dualopts(\ldualopt)\) such that \(\dualopt(\opt)=\ldualopt(\opt)\leq0\).
And since \(\dualopt\in\dualopts(\ldualopt)\) implies that \(\dualopt(\altopt)\geq\ldualopt(\altopt)>0\) for all~\(\altopt\in\desirset\), we find that \(\desirset\subseteq\desirset[\dualopt]\) and therefore \(\dualopt\in\dualopts(\desirset)\).
\end{proof}

Since, by Proposition~\ref{prop:isomorphisms:arch:noconstants}, all {\essarchim} sets of desirable options have the form~\(\desirset[\ldualopt]\) for \(\ldualopt\in\posldualopts\), Equation~\eqref{eq:desirset:from:cslf} points to the fact that all {\essarchim} models are also {\archim}:
\begin{equation}\label{eq:archimedean:is:essentially:archimedean:desirsets:noconstants}
(\forall\ldualopt\in\posldualopts)\desirset[\ldualopt]\in\archdesirsets
\text{ and in particular also }
(\forall\dualopt\in\posdualopts)\desirset[\dualopt]\in\archdesirsets.
\end{equation}

It is also easy to see that \(\archdesirsets\) is an intersection structure.
Indeed, consider any non-empty family of {\archim} sets of desirable options~\(\desirset[i]\), \(i\in I\) and let \(\desirset\coloneqq\bigcap_{i\in I}\desirset[i]\), then we already know that \(\desirset\) is coherent, so we only need to check that the separation condition~\ref{ax:desirset:essentially:archimedean:cslf} is satisfied.
So consider any~\(\opt\notin\desirset\), meaning that there is some~\(i\in I\) such that \(\opt\notin\desirset[i]\).
Hence there is some~\(\ldualopt\in\ldualopts(\desirset[i])\) such that \(\ldualopt(\opt)\leq0\).
Since it follows from~\(\desirset\subseteq\desirset[i]\) and Equation~\eqref{eq:cslfs:from:desirset} that also \(\ldualopt\in\ldualopts(\desirset)\), we see that, indeed, \(\desirset\) is {\archim}.

That \(\archdesirsets\) is an intersection structure also implies that we can introduce an \emph{{\archim} closure} operator~\(\archnatexdesirset\colon\desirsets\to\archdesirsets\cup\set{\opts}\) by letting
\begin{equation*}
\archnatexdesirset(\assessment)
\coloneqq\bigcap\cset{\desirset\in\archdesirsets}{\assessment\subseteq\desirset}
\text{ for all~\(\assessment\subseteq\opts\)}
\end{equation*}
be the smallest---if any---{\archim} set of desirable options that includes~\(\assessment\).
We call an assessment~\(\assessment\subseteq\opts\) \emph{{\archim} consistent} if \(\archnatexdesirset(\assessment)\neq\opts\), or equivalently, if \(\assessment\) is included in some {\archim} set of desirable options.
The closure operator~\(\archnatexdesirset\) implements \emph{conservative inference} with respect to the {\archimty} axioms, in that it extends an {\archim} consistent assessment~\(\assessment\) to the most conservative---smallest possible---{\archim} set of desirable options~\(\archnatexdesirset(\assessment)\).

\begin{theorem}[Closure]\label{theo:archimedean:representation:desirsets:noconstants}
For any set of desirable options~\(\desirset\in\desirsets\), we have that
\(
\archnatexdesirset(\desirset)
=\bigcap\cset{\desirset[\dualopt]}{\dualopt\in\posdualopts(\desirset)}
=\bigcap\cset{\desirset[\ldualopt]}{\ldualopt\in\posldualopts(\desirset)}
\).
Hence, \(\rejectset\) is {\archim} consistent if and only if \(\posdualopts(\desirset)\neq\emptyset\), or equivalently, \(\posldualopts(\desirset)\neq\emptyset\).
And an {\archim} consistent~\(\desirset\) is {\archim} if and only if \(\desirset=\bigcap\cset{\desirset[\dualopt]}{\dualopt\in\posdualopts(\desirset)}=\bigcap\cset{\desirset[\ldualopt]}{\ldualopt\in\posldualopts(\desirset)}\).
\end{theorem}

\begin{proof}
% Checked by Gert
We only give a proof for the first statement, because the second and third statements are trivial consequences of the first.
First of all, it follows from Equation~\eqref{eq:archimedean:is:essentially:archimedean:desirsets:noconstants} that if \(\desirset\) is not {\archim} consistent, the statement is trivially true, as all terms in the chain of equalities are then equal to~\(\opts\), as empty intersections.
To see this, assume that \(\desirset\) is not Archimedean consistent, so \(\desirset\not\subseteq\desirset'\) for all \(\desirset'\) in \(\archdesirsets\), and hence, by Equation~\eqref{eq:archimedean:is:essentially:archimedean:desirsets:noconstants}, in particular \(\desirset\not\subseteq\desirset[\ldualopt]\) for all \(\ldualopt\) in \(\posldualopts\).
This implies, by Equation~\eqref{eq:poscslfs:from:desirset}, that indeed \(\posldualopts(\desirset)=\emptyset\).
A similar argument shows that then also \(\posdualopts(\desirset)=\emptyset\).

So we may assume without loss of generality that \(\desirset\) is  {\archim} consistent.
But then Equation~\eqref{eq:archimedean:is:essentially:archimedean:desirsets:noconstants} and \(\posdualopts(\desirset)\subseteq\posldualopts(\desirset)\) imply that
\begin{equation*}
\desirset[\ast]
\coloneqq\archnatexdesirset(\desirset)
\subseteq\bigcap\cset{\desirset[\ldualopt]}{\ldualopt\in\posldualopts(\desirset)}
\subseteq\bigcap\cset{\desirset[\dualopt]}{\dualopt\in\posdualopts(\desirset)},
\end{equation*}
because if \(\ldualopt\in\posldualopts(\desirset)\) then \(\desirset[\ldualopt]\) is {\archim} and \(\desirset\subseteq\desirset[\ldualopt]\).
So our proof will be complete if we can show that \(\bigcap\cset{\desirset[\dualopt]}{\dualopt\in\posdualopts(\desirset)}\subseteq\desirset[\ast]\).
To do so, assume that \(\opt\notin\desirset[\ast]\), then since \(\desirset[\ast]\) is {\archim}, the separation property~\ref{ax:desirset:essentially:archimedean:clf} tells us that there is some~\(\dualopt\in\dualopts(\desirset[\ast])\) such that \(\opt\notin\desirset[\dualopt]\).
Since \(\desirset\) is in particular coherent, \(\dualopt\in\dualopts(\desirset)\) implies that \(\posopts\subseteq\desirset\subseteq\desirset[\dualopt]\), and therefore also \(\dualopt\in\posdualopts\), whence \(\dualopt\in\posdualopts(\desirset[\ast])\).
But since \(\desirset\subseteq\desirset[\ast]\), we infer from Equation~\eqref{eq:posclfs:from:desirset} that \(\posdualopts(\desirset[\ast])\subseteq\posdualopts(\desirset)\).
We conclude that there is some~\(\dualopt\in\posdualopts(\desirset)\) such that \(\opt\notin\desirset[\dualopt]\), and therefore, also \(\opt\notin\bigcap\cset{\desirset[\dualopt]}{\dualopt\in\posdualopts(\desirset)}\).
\end{proof}
\noindent
The following important representation theorem confirms that the {\essarchim} sets of desirable options can be used to represent all {\archim} sets of desirable options via intersection.

\begin{corollary}[Representation]\label{cor:archimedean:representation:desirsets:noconstants:twosided}
For any set of desirable options~\(\desirset\in\desirsets\), the following statements are equivalent:
\begin{enumerate}[label=\upshape(\roman*),leftmargin=*]
\item\label{it:archimedean:representation:desirsets:noconstants:twosided:archim} \(\desirset\) is {\archim};
\item\label{it:archimedean:representation:desirsets:noconstants:twosided:superlinear} there is some non-empty set~\(\setofldualopts\subseteq\posldualopts\) of positive superlinear bounded real functionals such that \(\desirset=\bigcap\cset{\desirset[\ldualopt]}{\ldualopt\in\setofldualopts}\);
\item\label{it:archimedean:representation:desirsets:noconstants:twosided:linear} there is some non-empty set~\(\setofdualopts\subseteq\posdualopts\) of positive linear bounded real functionals such that \(\desirset=\bigcap\cset{\desirset[\dualopt]}{\dualopt\in\setofdualopts}\).
\end{enumerate}
In that case, the largest such set~\(\setofldualopts\) is\/~\(\posldualopts\group{\desirset}\), and the largest such set~\(\setofdualopts\) is\/~\(\posdualopts\group{\desirset}\).
\end{corollary}

\begin{proof}
% Checked by Gert
We prove that \ref{it:archimedean:representation:desirsets:noconstants:twosided:archim}\(\ifandonlyif\)\ref{it:archimedean:representation:desirsets:noconstants:twosided:superlinear}.
The proof that \ref{it:archimedean:representation:desirsets:noconstants:twosided:archim}\(\ifandonlyif\)\ref{it:archimedean:representation:desirsets:noconstants:twosided:linear} is completely analogous.

\ref{it:archimedean:representation:desirsets:noconstants:twosided:archim}\(\then\)\ref{it:archimedean:representation:desirsets:noconstants:twosided:superlinear}.
If \(\desirset\) is {\archim}, then \(\archnatexdesirset(\desirset)=\desirset\).
Now use Theorem~\ref{theo:archimedean:representation:desirsets:noconstants}.

\ref{it:archimedean:representation:desirsets:noconstants:twosided:superlinear}\(\then\)\ref{it:archimedean:representation:desirsets:noconstants:twosided:archim}.
Since all \(\desirset[\ldualopt]\), \(\ldualopt\in\setofldualopts\) are {\archim}---because {\essarchim} by Proposition~\ref{prop:isomorphisms:arch:noconstants}---so is their intersection~\(\desirset\).

The final statement now follows without too many difficulties from Theorem~\ref{theo:archimedean:representation:desirsets:noconstants}.
\end{proof}
\noindent
The discussion in Section~\ref{sec:normalisation} shows that the sets of functionals in Theorem~\ref{theo:archimedean:representation:desirsets:noconstants} and Corollary~\ref{cor:archimedean:representation:desirsets:noconstants:twosided} can also be replaced by~\(\nml(\posldualopts(\desirset))\), \(\nml(\posdualopts(\desirset))\), \(\nml(\setofldualopts)\) and~\(\nml(\setofdualopts)\) respectively, where \(\opt[o]\) is any option in~\(\topint(\posopts)\).
The sets~\(\nml(\posldualopts(\desirset))\) and~\(\nml(\posdualopts(\desirset))\) will now be the largest sets of positive bounded superlinear (linear) functionals that achieve representation and all of whose members are constant additive with respect to the option~\(\opt[o]\).

Corollary~\ref{cor:archimedean:representation:desirsets:noconstants:twosided}, in combination with Proposition~\ref{prop:isomorphisms:arch:noconstants}, shows that a set of desirable options is {\archim} if and only if it is an intersection of open coherent sets of desirable options.
The present notion of {\archimty} for sets of desirable option sets therefore coincides with the even continuity introduced by Cozman~\cite{cozman2018:evenlyconvex} in the more restrictive context of finite-dimensional option---or rather gamble---spaces.

Finally, Proposition~\ref{prop:isomorphisms:arch:mixing:noconstants} leads us to conclude that for sets of desirable options that are \emph{mixing}, the notions of {\archimty} and {\essarchimty} coincide: the mixing (essentially) {\archim} sets of desirable options are exactly the \(\desirset[\dualopt]\) for~\(\dualopt\in\posdualopts\).
To see in more detail how the argument works, it suffices to consider the following result.

\begin{proposition}\label{prop:mixing:and:essentially:archimedean:noconstants}
For all mixing sets of desirable options~\(\desirset\in\mixdesirsets\) and any~\(\opt[o]\in\topint(\posopts)\), the following statements are equivalent:
\begin{enumerate}[label=\upshape(\roman*),leftmargin=*]
\item\label{it:mixing:and:essentially:archimedean:consistent:noconstants} \(\desirset\) is {\archim} consistent;
\item\label{it:mixing:and:essentially:archimedean:a:noconstants} \(\desirset\) is {\archim};
\item\label{it:mixing:and:essentially:archimedean:ea:noconstants} \(\desirset\) is {\essarchim};
\item\label{it:mixing:and:essentially:archimedean:prevs:noconstants} \(\nml\group{\posldualopts(\desirset)}=\nml\group{\posdualopts(\desirset)}=\set{\dualopto}\);
\item\label{it:mixing:and:essentially:archimedean:desirset:noconstants} \(\desirset=\desirset[\dualopto]\).
\end{enumerate}
\end{proposition}

\begin{proof}
% Checked by Gert
We begin by recalling from Proposition~\ref{prop:dualopt:from:desirset} that \(\ldualopto\) is a linear bounded real functional, also denoted by~\(\dualopto\), because \(\desirset\) is mixing, by assumption.
Moreover, it holds that
\begin{equation}\label{eq:mixing:and:essentially:archimedean:noconstants}
\dualopto(\opt)>0\then\opt\in\desirset\text{ for all~\(\opt\in\opts\)}.
\end{equation}
To see this, observe that \(\dualopto(\opt)>0\) implies that there is some real~\(\alpha>0\) such that \(\opt-\alpha\opt[o]\in\desirset\).
Moreover, also \(\opt[o]\in\topint(\posopts)\subseteq\posopts\subseteq\desirset\), where the last inclusion follows from Axiom~\ref{ax:desirs:pos}.
But then Axiom~\ref{ax:desirs:cone} guarantees that \(\opt=(\opt-\alpha\opt[o])+\alpha\opt[o]\in\desirset\).

We'll now prove that \ref{it:mixing:and:essentially:archimedean:consistent:noconstants}\(\then\)\ref{it:mixing:and:essentially:archimedean:prevs:noconstants}\(\then\)\ref{it:mixing:and:essentially:archimedean:desirset:noconstants}\(\then\)\ref{it:mixing:and:essentially:archimedean:ea:noconstants}\(\then\)\ref{it:mixing:and:essentially:archimedean:a:noconstants}\(\then\)\ref{it:mixing:and:essentially:archimedean:consistent:noconstants}.

\ref{it:mixing:and:essentially:archimedean:consistent:noconstants}\(\then\)\ref{it:mixing:and:essentially:archimedean:prevs:noconstants}.
Assume that \(\desirset\) is {\archim} consistent.
Then we infer from Theorem~\ref{theo:archimedean:representation:desirsets:noconstants} that \(\posldualopts(\desirset)\neq\emptyset\).
So, consider any~\(\ldualopt\in\posldualopts(\desirset)\), implying that \(\ldualopt\in\posldualopts\) and \(\desirset\subseteq\desirset[\ldualopt]=\desirset[\nml\ldualopt]\) [use Proposition~\ref{prop:normalisation:map:properties}\ref{it:normalisation:map:properties:preserves:open:cone}].
We'll first show that \(\nml\ldualopt\) is linear.
First of all, if we consider any~\(\opt,\altopt\in\opts\), then the mixingness of~\(\desirset\) implies in particular that if \(\opt+\altopt\in\desirset\), then \(\set{\opt,\altopt}\cap\desirset\neq\emptyset\), and therefore \(\nml\ldualopt(\opt)>0\) or \(\nml\ldualopt(v)>0\).
Now consider any~\(\altopttoo\in\opts\) and any real~\(\epsilon>0\), and let \(\opt\coloneqq\altopttoo-\nml\ldualopt(\altopttoo)\opt[o]\) and \(\altopt\coloneqq\nml\ldualopt(\altopttoo)\opt[o]-\altopttoo+\epsilon\opt[o]\).
Then \(\opt+\altopt=\epsilon\opt[o]\in\desirset\), and therefore [use Proposition~\ref{prop:normalisation:map:properties}\ref{it:normalisation:map:properties:constant:additivity}]
\begin{equation*}
0
<\nml\ldualopt(\opt)
=\nml\ldualopt(\altopttoo)-\nml\ldualopt(\altopttoo)
\end{equation*}
or
\begin{equation*}
0<\nml\ldualopt(\altopt)=\nml\ldualopt(\altopttoo)+\epsilon+\nml\ldualopt(-\altopttoo),
\end{equation*}
whence \(-\nml\ldualopt(-\altopttoo)<\nml\ldualopt(\altopttoo)+\epsilon\).
Since this holds for all~\(\epsilon>0\) and all~\(\altopttoo\in\opts\), this implies that the superlinear bounded real functional~\(\nml\ldualopt\) is self-conjugate [use Proposition~\ref{prop:ldualopt:from:desirset} and Proposition~\ref{prop:ludualopt:properties}\ref{it:ludualopt:properties:lsmalleru}], and therefore indeed linear.

Equation~\eqref{eq:mixing:and:essentially:archimedean:noconstants}, the linearity of both~\(\dualopto\) and~\(\nml\ldualopt\), and the fact that \(\desirset\subseteq\desirset[{\nml\ldualopt}]\) now allow us to conclude that both
\begin{equation}\label{eq:mixing:and:essentially:archimedean:noconstants:combined}
\dualopto(\opt)>0\then\nml\ldualopt(\opt)>0
\text{ and }
\dualopto(\opt)<0\then\nml\ldualopt(\opt)<0
\text{ for all~\(\opt\in\opts\)}.
\end{equation}
Now consider any~\(\altopttoo\in\opts\), and let in particular~\(\opt\coloneqq\altopttoo-\nml\ldualopt(\altopttoo)\opt[o]\), then [use Proposition~\ref{prop:normalisation:map:properties}\ref{it:normalisation:map:properties:constant:additivity} and Equation~\eqref{eq:ldualopt:from:desirset}]
\begin{equation*}
\nml\ldualopt(\opt)=\nml\ldualopt(\altopttoo)-\nml\ldualopt(\altopttoo)=0
\text{ and }
\dualopto(\opt)=\dualopto(\altopttoo)-\nml\ldualopt(\altopttoo).
\end{equation*}
Equation~\eqref{eq:mixing:and:essentially:archimedean:noconstants:combined} then guarantees that \(\dualopto(\altopttoo)=\nml\ldualopt(\altopttoo)\).
This shows that \(\nml\group{\posldualopts(\desirset)}=\set{\dualopto}\), whence also \(\nml\group{\posdualopts(\desirset)}=\set{\dualopto}\).

\ref{it:mixing:and:essentially:archimedean:prevs:noconstants}\(\then\)\ref{it:mixing:and:essentially:archimedean:desirset:noconstants}.
Since \(\posldualopts(\desirset)\) is non-empty, by assumption, we can consider any~\(\ldualopt\in\posldualopts(\desirset)\).
But then \(\desirset\subseteq\desirset[\ldualopt]=\desirset[\nml\ldualopt]=\desirset[\dualopto]=\topint(\desirset)\subseteq\desirset\), where the first equality follows from Proposition~\ref{prop:normalisation:map:properties}\ref{it:normalisation:map:properties:preserves:open:cone}, the second equality from the assumption, and the third equality from Propositions~\ref{prop:dualopt:from:desirset} and~\ref{prop:ldualopto:and:openness}.
% Hence \(\desirset\) is indeed open (besides coherent).

\ref{it:mixing:and:essentially:archimedean:desirset:noconstants}\(\then\)\ref{it:mixing:and:essentially:archimedean:ea:noconstants}.
Use Proposition~\ref{prop:isomorphisms:arch:mixing:noconstants}.

\ref{it:mixing:and:essentially:archimedean:ea:noconstants}\(\then\)\ref{it:mixing:and:essentially:archimedean:a:noconstants}.
Use Equation~\eqref{eq:archimedean:is:essentially:archimedean:desirsets:noconstants}.

\ref{it:mixing:and:essentially:archimedean:a:noconstants}\(\then\)\ref{it:mixing:and:essentially:archimedean:consistent:noconstants}.
Trivial.
\end{proof}

\begin{runningexample}
In our coin example, let us consider the set of desirable options \(\desirset[(\lp,\up)]\) defined by 
\begin{equation*}
\dif\in\desirset[(\lp,\up)]\ifandonlyif
\begin{cases}
\ex[\lp]\group{\dif(\bolleke,\bestreward)}\geq0
&\text{if \(\dif(\heads,\bestreward)\geq\dif(\tails,\bestreward)\)}\\
\ex[\up]\group{\dif(\bolleke,\bestreward)}\geq0
&\text{if \(\dif(\heads,\bestreward)\leq\dif(\tails,\bestreward)\)}
\end{cases}
\text{ and }
\dif(\bolleke,\bestreward)\neq0,
\text{ for all \(\dif\in\difs\)},
\end{equation*}
where we assume that \(0<\lp<\up<1\); see Figure~\ref{fig:archim:but:not:essarchim} for a graphical representation. 
It's not difficult to see that \(\desirset[(\lp,\up)]=\bigcap_{p\in(\lp,\up)}\desirset[p]\), where, as in instalment~\ref{ex:assessment:lower:upper:betting:rates},
\begin{equation*}
\dif\in\desirset[p]\ifandonlyif\ex[p]\group{\dif(\bolleke,\bestreward)}=p\dif(\heads,\bestreward)+(1-p)\dif(\tails,\bestreward)>0,
\text{ for all \(\dif\in\difs\)}.
\end{equation*}
We've seen in instalment~\ref{ex:essential:archimedeanity} that all such~\(\desirset[p]\) are {\essarchim}, under both background orderings~\(\gblgt\) and~\(\gblpgt\).
It therefore follows from Theorem~\ref{theo:archimedean:representation:desirsets:noconstants} that the intersection~\(\desirset[(\lp,\up)]\) is {\archim} under these background orderings as well.
But \(\desirset[(\lp,\up)]\) is not {\essarchim}, as it's clearly not an open set. 
\stopit
\end{runningexample}

\begin{figure}[h]
\centering
\begin{tikzpicture}[scale=.5]\footnotesize
\fill[blue!50] (0,0) -- (4,-2) -- (4,4) -- (-1,4) -- cycle;
\draw[gray,->] (-2.2,0) -- (4.2,0) node[below right] {\(\heads\)};
\draw[gray,->] (0,-2.2) -- (0,4.2) node[above left] {\(\tails\)};
\draw[blue,thick] (0,0) -- (4,-2) node[below] {\(\ex[\lp]=0\)};
\draw[blue,dashed] (0,0) -- (-2,1);
\draw[blue,thick] (0,0) -- (-1,4) node[left] {\(\ex[\up]=0\)};
\draw[blue,dashed] (0,0) -- (.5,-2);
\node[draw=blue,fill=white,circle,inner sep=1pt] at (0,0) {};
\node[white] at (3,3) {\(\desirset[(\lp,\up)]\)};
\end{tikzpicture}
\caption{The {\archim} but not {\essarchim} set of desirable options~\(\desirset[(\lp,\up)]\), represented by the values~\(\dif(\bolleke,\bestreward)\) that its elements~\(\dif\in\desirset[(\lp,\up)]\) assume in the reward~\(\bestreward\). The full blue lines represent `borders' that are included in the set.}
\label{fig:archim:but:not:essarchim}
\end{figure}

\section{{\archimty} for sets of desirable option sets}\label{sec:archimedeanity:non-binary}
We can now `lift' our discussion of {\archimty} from binary to general choice models, that is, sets of desirable option sets.

We begin by setting up the relevant machinery, in close analogy with Section~\ref{sec:archimedanity:binary}.
Given a set of desirable option sets~\(\rejectset\in\rejectsets\), we let
\begin{align}
\ldualopts(\rejectset)
\coloneqq
&\cset{\ldualopt\in\ldualopts}
{(\forall\optset\in\rejectset)(\exists\opt\in\optset)\ldualopt(\opt)>0}
=\cset{\ldualopt\in\ldualopts}{(\forall\optset\in\rejectset)\optset\cap\desirset[\ldualopt]\neq\emptyset}\notag\\
=&\cset{\ldualopt\in\ldualopts}{(\forall\optset\in\rejectset)\optset\in\rejectset[\ldualopt]}
=\cset{\ldualopt\in\ldualopts}{\rejectset\subseteq\rejectset[\ldualopt]},
\label{eq:cslfs:from:rejectset}
\end{align}
where we used Equation~\eqref{eq:desirset:from:cslf} and let
\begin{equation}\label{eq:cslf:to:rejectset}
\rejectset[\ldualopt]
\coloneqq\rejectset[{\desirset[\ldualopt]}]
=\cset{\optset\in\optsets}{\optset\cap\desirset[\ldualopt]\neq\emptyset}
=\cset{\optset\in\optsets}{(\exists\opt\in\optset)\ldualopt(\opt)>0}
\text{ for any~\(\ldualopt\in\ldualopts\).}
\end{equation}
Similarly,
\begin{equation}\label{eq:clfs:from:rejectset}
\dualopts(\rejectset)
\coloneqq\cset{\dualopt\in\dualopts}
{(\forall\optset\in\rejectset)(\exists\opt\in\optset)\dualopt(\opt)>0}
=\cset{\dualopt\in\dualopts}{\rejectset\subseteq\rejectset[\dualopt]}
\subseteq\ldualopts(\rejectset).
\end{equation}
We'll also use the positive varieties:
\begin{align}
\posldualopts(\rejectset)
&\coloneqq\cset{\ldualopt\in\posldualopts}
{(\forall\optset\in\rejectset)(\exists\opt\in\optset)\ldualopt(\opt)>0}
=\cset{\ldualopt\in\posldualopts}{\rejectset\subseteq\rejectset[\ldualopt]}
\label{eq:poscslfs:from:rejectset}\\
\posdualopts(\rejectset)
&\coloneqq\cset{\dualopt\in\posdualopts}{\rejectset\subseteq\rejectset[\dualopt]}
\subseteq\posldualopts(\rejectset).
\label{eq:posclfs:from:rejectset}
\end{align}

If we pick any~\(\opt[o]\in\topint(\posopts)\) and associate with it the normalisation~\(\nml\), then since we know from the discussion in Section~\ref{sec:normalisation} that \(\desirset[\ldualopt]=\desirset[\nml\ldualopt]\) and \(\desirset[\dualopt]=\desirset[\nml\dualopt]\), we also infer from Equation~\eqref{eq:cslf:to:rejectset} that
\begin{equation}\label{eq:normalised:cslf:to:rejectset}
\rejectset[\ldualopt]=\rejectset[\nml\ldualopt]
\text{ and }
\rejectset[\dualopt]=\rejectset[\nml\dualopt]
\text{ for all~\(\ldualopt\in\posldualopts\) and~\(\dualopt\in\posdualopts\).}
\end{equation}

We'll call a set of desirable option sets~\(\rejectset\in\rejectsets\) \emph{{\archim}} if it is coherent and if the following separation property is satisfied:\footnote{Observe here too, by the way, that if \(\rejectset\) is coherent, then \(\ldualopt\in\ldualopts(\rejectset)\) implies that \(\rejectset\subseteq\rejectset[\ldualopt]\) and therefore also \(\set{\opt}\in\rejectset[\ldualopt]\), or equivalently, \(\ldualopt(\opt)>0\), for all \(\opt\in\posopts\).
This implies that also \(\ldualopt\in\posldualopts\), whence \(\ldualopts(\rejectset)=\posldualopts(\rejectset)\).}
\begin{enumerate}[label=\({\mathrm{K}}_{\mathrm{A}}\).,ref=\({\mathrm{K}}_{\mathrm{A}}\),leftmargin=*]
\item\label{ax:rejectset:essentially:archimedean:noconstants} \((\forall\optset\notin\rejectset)(\exists\ldualopt\in\ldualopts(\rejectset))(\forall\opt\in\optset)\ldualopt(\opt)\leq0\),
\end{enumerate}
and we denote by~\(\archrejectsets\) the set of all {\archim} sets of desirable option sets.

If we look at Proposition~\ref{prop:isomorphisms:arch:noconstants}, we see that the {\essarchim} sets of desirable options are all the \(\desirset[\ldualopt]\) for \(\ldualopt\in\posldualopts{}\), and Equation~\eqref{eq:cslf:to:rejectset} then tells us that the corresponding binary sets of desirable option sets~\(\rejectset[\ldualopt]\) are all {\archim}:
\begin{equation}\label{eq:archimedean:is:essentially:archimedean:rejectsets:noconstants}
(\forall\ldualopt\in\posldualopts)\rejectset[\ldualopt]\in\archrejectsets.
\end{equation}
But we can go further than this, and establish a strong connection between {\archim} sets of desirable options on the one hand, and binary {\archim} sets of desirable option sets on the other.
This suggests that there is at least some merit in our defining {\archimty} for sets of desirable option sets in the way that we did.

\begin{proposition}[Binary embedding]\label{prop:archimedeanbinaryiffD}
For any~\(\desirset\in\desirsets\), \(\rejectset[\desirset]\) is {\archim} if and only if \(\desirset\) is.
\end{proposition}

\begin{proof}
% Checked by Gert
Before we begin with the actual argument, we observe that
\begin{equation*}
\ldualopts(\rejectset[\desirset])
=\cset{\ldualopt\in\ldualopts}{(\forall\optset\in\rejectset[\desirset])\optset\cap\desirset[\ldualopt]\neq\emptyset}
=\cset{\ldualopt\in\ldualopts}{\desirset\subseteq\desirset[\ldualopt]}
=\ldualopts(\desirset).
\end{equation*}
For the proof, we first assume that \(\desirset\) is {\archim}.
Then \(\desirset\) is in particular coherent, and therefore so is \(\rejectset[\desirset]\), by Proposition~\ref{prop:coherence:for:binary}.
So it remains to prove that \(\rejectset[\desirset]\) satisfies the separation requirement~\ref{ax:rejectset:essentially:archimedean:noconstants}.
Consider any~\(\optset\notin\rejectset[\desirset]\), meaning that \(\optset\cap\desirset=\emptyset\).
Since \(\desirset\) is {\archim}, we infer from the separation requirement~\ref{ax:desirset:essentially:archimedean:clf} that for all~\(\opt\in\optset\), there is some~\(\dualopt[\opt]\in\dualopts(\desirset)\) such that \(\dualopt[\opt](\opt)\leq0\).
If we let \(\ldualopt\coloneqq\min_{\opt\in\optset}\dualopt[\opt]\), then
\begin{equation*}
\ldualopt(\altopt)
=\min_{\opt\in\optset}\dualopt[\opt](\altopt)
\leq\dualopt[\altopt](\altopt)\leq0
\text{ for all~\(\altopt\in\optset\)},
\end{equation*}
so we are done if we can show that \(\ldualopt\in\ldualopts(\rejectset[\desirset])\), or equivalently, that \(\ldualopt\in\ldualopts(\desirset)\).
Consider to this end any~\(\altopt\in\desirset\), then because \(\dualopt[\opt]\in\dualopts(\desirset)\) we see that \(\dualopt[\opt](\altopt)>0\), for all~\(\opt\in\optset\), and therefore also \(\ldualopt(\altopt)>0\).
Hence, indeed, \(\ldualopt\in\ldualopts(\desirset)\).

Next, we assume that \(\rejectset[\desirset]\) is {\archim}.
Then \(\rejectset[\desirset]\) is in particular coherent, and therefore so is \(\desirset\), by Proposition~\ref{prop:coherence:for:binary}.
It therefore remains to prove that \(\desirset\) satisfies the separation requirement~\ref{ax:desirset:essentially:archimedean:cslf}.
So consider any~\(\opt\notin\desirset\).
That \(\rejectset[\desirset]\) is assumed to satisfy the separation requirement~\ref{ax:rejectset:essentially:archimedean:noconstants} implies, with~\(\optset\coloneqq\set{\opt}\), that there is some~\(\ldualopt\in\ldualopts(\rejectset[\desirset])=\ldualopts(\desirset)\) such that \(\ldualopt(\opt)\leq0\).
\end{proof}

It is also easy to see that \(\archrejectsets\) is an intersection structure.
Indeed, consider any non-empty family of {\archim} sets of desirable option sets~\(\rejectset[i]\), \(i\in I\) and let~\(\rejectset\coloneqq\bigcap_{i\in I}\rejectset[i]\), then we already know that \(\rejectset\) is coherent, so we only need to show that the separation condition~\ref{ax:rejectset:essentially:archimedean:noconstants} is satisfied.
So consider any~\(\optset\notin\rejectset\), meaning that there is some~\(i\in I\) such that \(\optset\notin\rejectset[i]\).
Hence there is some~\(\ldualopt\in\ldualopts(\rejectset[i])\) such that \(\ldualopt(\opt)\leq0\) for all~\(\opt\in\optset\).
Since it follows from \(\rejectset\subseteq\rejectset[i]\) and Equation~\eqref{eq:cslfs:from:rejectset} that also \(\ldualopt\in\ldualopts(\rejectset)\), we see that, indeed, \(\rejectset\) is {\archim}.

That \(\archrejectsets\) is an intersection structure also implies that we can introduce an \emph{{\archim} closure} operator~\(\archnatexrejectset\colon\rejectsets\to\archrejectsets\cup\set{\optsets}\) by letting
\begin{equation*}
\archnatexrejectset(\assessment)
\coloneqq\bigcap\cset{\rejectset\in\archrejectsets}{\assessment\subseteq\rejectset}
\text{ for all~\(\assessment\subseteq\optsets\)}
\end{equation*}
be the smallest---if any---{\archim} set of desirable option sets that includes~\(\assessment\).
We call an assessment~\(\assessment\subseteq\optsets\) \emph{{\archim} consistent} if~\(\archnatexrejectset(\assessment)\neq\optsets\), or equivalently, if~\(\assessment\) is included in some {\archim} set of desirable option sets.

\begin{theorem}[Closure]\label{theo:archimedean:representation:rejectsets:noconstants}
For any set of desirable option sets~\(\rejectset\in\rejectsets\), we have that~\(\archnatexrejectset(\rejectset)=\bigcap\cset{\rejectset[\ldualopt]}{\ldualopt\in\posldualopts(\rejectset)}\).
Hence, \(\rejectset\) is {\archim} consistent if and only if~\(\posldualopts(\rejectset)\neq\emptyset\).
And an {\archim} consistent set of desirable option sets~\(\rejectset\) is {\archim} if and only if~\(\rejectset=\bigcap\cset{\rejectset[\ldualopt]}{\ldualopt\in\posldualopts(\rejectset)}\).
\end{theorem}

\begin{proof}
% Checked by Gert
We only give a proof for the first statement, because the later statements are trivial consequences of the first.
First of all, it follows from Equation~\eqref{eq:archimedean:is:essentially:archimedean:rejectsets:noconstants} that if~\(\rejectset\) is not {\archim} consistent, the statement is trivially true, as both sides are then equal to~\(\optsets\), as empty intersections.
To see this, assume that \(\rejectset\) is not Archimedean consistent, so \(\rejectset\not\subseteq\rejectset'\) for all \(\rejectset'\) in \(\archcohrejectsets\), and hence, by Equation~\eqref{eq:archimedean:is:essentially:archimedean:rejectsets:noconstants}, in particular \(\rejectset\not\subseteq\rejectset[\ldualopt]\) for all \(\ldualopt\) in \(\posldualopts\).
This implies, by Equation~\eqref{eq:poscslfs:from:rejectset}, that indeed \(\posldualopts(\rejectset)=\emptyset\).

So we may assume without loss of generality that \(\rejectset\) is  {\archim} consistent.
But then Equation~\eqref{eq:archimedean:is:essentially:archimedean:rejectsets:noconstants} implies that
\begin{equation*}
\rejectset[\ast]
\coloneqq\archnatexrejectset(\rejectset)
\subseteq\bigcap\cset{\rejectset[\ldualopt]}{\ldualopt\in\posldualopts(\rejectset)},
\end{equation*}
because if \(\ldualopt\in\posldualopts(\rejectset)\) then \(\rejectset[\ldualopt]\) is {\archim} and \(\rejectset\subseteq\rejectset[\ldualopt]\).
So our proof will be complete if we can show that \(\bigcap\cset{\rejectset[\ldualopt]}{\ldualopt\in\posldualopts(\rejectset)}\subseteq\rejectset[\ast]\).
Assume to this end that \(\optset\notin\rejectset[\ast]\), then since \(\rejectset[\ast]\) is {\archim}, the separation requirement~\ref{ax:rejectset:essentially:archimedean:noconstants} tells us that there is some~\(\ldualopt\in\ldualopts(\rejectset[\ast])\) such that \(\optset\notin\rejectset[\ldualopt]\).
Since \(\rejectset[\ast]\) is in particular also coherent, \(\ldualopt\in\ldualopts(\rejectset[\ast])\) implies that \(\rejectset[\ast]\subseteq\rejectset[\ldualopt]\) and therefore also \(\set{\opt}\in\rejectset[\ldualopt]\), or equivalently, \(\ldualopt(\opt)>0\), for all \(\opt\in\posopts\).
This implies that also \(\ldualopt\in\posldualopts(\rejectset[\ast])\).
But since \(\rejectset\subseteq\rejectset[\ast]\), we infer from Equation~\eqref{eq:poscslfs:from:rejectset} that \(\posldualopts(\rejectset[\ast])\subseteq\posldualopts(\rejectset)\).
Hence, indeed, \(\optset\notin\bigcap\cset{\rejectset[\ldualopt]}{\ldualopt\in\posldualopts(\rejectset)}\).
\end{proof}
\noindent
And here too, the following important representation theorem confirms that the positive superlinear bounded real functionals---and therefore the {\essarchim} sets of desirable options---can be used to represent all {\archim} sets of desirable option sets.

\begin{corollary}[Representation]\label{cor:archimedean:representation:rejectsets:noconstants:twosided}
A set of desirable option sets~\(\rejectset\in\rejectsets\) is {\archim} if and only if there is some non-empty set~\(\setofldualopts\subseteq\posldualopts\) of positive superlinear bounded real functionals such that \(\rejectset=\bigcap\cset{\rejectset[\ldualopt]}{\ldualopt\in\setofldualopts}\).
In that case, the largest such set~\(\setofldualopts\) is\/~\(\posldualopts\group{\rejectset}\).
\end{corollary}

\begin{proof}
% Checked by Gert
For necessity, assume that \(\rejectset\) is {\archim}, then \(\archnatexrejectset(\rejectset)=\rejectset\).
Now use Theorem~\ref{theo:archimedean:representation:rejectsets:noconstants}.
For sufficiency, since all~\(\rejectset[\ldualopt]\), \(\ldualopt\in\setofldualopts\) are {\archim} [see Equation~\eqref{eq:archimedean:is:essentially:archimedean:rejectsets:noconstants}], so is their intersection~\(\rejectset\).
The final statement now follows at once from Theorem~\ref{theo:archimedean:representation:rejectsets:noconstants}.
\end{proof}
\noindent
The discussion in Section~\ref{sec:normalisation} shows that the sets of functionals in Theorem~\ref{theo:archimedean:representation:rejectsets:noconstants} and Corollary~\ref{cor:archimedean:representation:rejectsets:noconstants:twosided} can also be replaced by~\(\nml(\posldualopts(\rejectset))\) and~\(\nml(\setofldualopts)\) respectively, where \(\opt[o]\) is any option in~\(\topint(\posopts)\).
The set~\(\nml(\posldualopts(\rejectset))\) will now be the largest set of positive bounded superlinear functionals that achieves representation and all of whose members are constant additive with respect to the option~\(\opt[o]\).

Interestingly, if we want to find the {\archim} closure of a set of desirable option sets, we can first look for the dominating coherent binary models, and then---if possible---turn these into {\archim} binary models by taking their {\archim} closure.

\begin{theorem}
For any {\archim}~\(\rejectset\in\archrejectsets\), \(\rejectset=\bigcap\cset{\rejectset[{\archnatexdesirset(\desirset)}]}{\desirset\in\cohdesirsets(\rejectset)}\).
\end{theorem}

\begin{proof}
% Checked by Gert
Since \(\rejectset\) is {\archim}, it's in particular also coherent, so Theorem~\ref{theo:coherent:representation:twosided} already makes sure that
\begin{equation*}
\rejectset
=\bigcap\cset{\rejectset[\desirset]}{\desirset\in\cohdesirsets(\rejectset)}
\subseteq\bigcap\cset{\rejectset[{\archnatexdesirset(\desirset)}]}{\desirset\in\cohdesirsets(\rejectset)}.
\end{equation*}
For the converse inclusion, consider any~\(\optset\in\optsets\) and assume that \(\optset\notin\rejectset\), so we infer from the separation property~\ref{ax:rejectset:essentially:archimedean:noconstants} that there is some~\(\ldualopt\) in~\(\ldualopts(\rejectset)\) such that \(\optset\cap\desirset[\ldualopt]=\emptyset\).
Now, \(\ldualopt\in\ldualopts(\rejectset)\) implies that \(\desirset[\ldualopt]\in\cohdesirsets(\rejectset)\), by Equation~\eqref{eq:cslfs:from:rejectset}, and \(\optset\cap\desirset[\ldualopt]=\emptyset\) means that \(\optset\notin\rejectset[{\desirset[\ldualopt]}]\).
But \(\desirset[\ldualopt]\) is {\archim}---see Equation~\eqref{eq:archimedean:is:essentially:archimedean:desirsets:noconstants}---and therefore \(\archnatexdesirset(\desirset[\ldualopt])=\desirset[\ldualopt]\), whence also \(\optset\notin\rejectset[{\archnatexdesirset(\desirset[\ldualopt])}]\).
\end{proof}

To conclude this discussion of {\archimty} for general choice models, let us see what happens if we also impose mixingness: what can we say about {\archim} and mixing sets of desirable option sets?
We begin by showing that any representation must then needs involve positive \emph{linear} bounded real functionals, at least \emph{after normalisation}.

\begin{proposition}\label{prop:mixing:rejectsets:linear:dominance}
Consider any set of desirable option sets~\(\rejectset\in\rejectsets\), any~\(\ldualopt\in\posldualopts\) and any~\(\opt[o]\in\topint(\posopts)\).
If \(\rejectset\) is mixing, then \(\rejectset\subseteq\rejectset[\ldualopt]\) implies that\/ \(\nml\ldualopt\) is linear.
Therefore, \(\nml\group{\posldualopts(\rejectset)}=\nml\group{\posdualopts(\rejectset)}\).
\end{proposition}

\begin{proof}
% Checked by Gert
Assume that \(\rejectset\) is mixing, and consider any~\(\altopttoo\in\opts\) and any real~\(\epsilon>0\).
Let \(\opt\coloneqq\altopttoo-\nml\ldualopt(\altopttoo)\opt[o]\) and \(\altopt\coloneqq(\nml\ldualopt(\altopttoo)+\epsilon)\opt[o]-\altopttoo\), then \(\opt+\altopt=\epsilon\opt[o]\in\posopts\) and therefore \(\set{\opt,\altopt,\opt+\altopt}\in\rejectset\), because the mixing~\(\rejectset\) is in particular coherent.
The mixingness of~\(\rejectset\) then implies that also \(\set{\opt,\altopt}\in\rejectset\).
Since \(\nml\ldualopt(\opt)=\nml\ldualopt(\altopttoo)-\nml\ldualopt(\altopttoo)=0\) [use Proposition~\ref{prop:normalisation:map:properties}\ref{it:normalisation:map:properties:constant:additivity}], we infer from~\(\rejectset\subseteq\rejectset[\ldualopt]=\rejectset[\nml\ldualopt]\) that, necessarily, \(0<\nml\ldualopt(\altopt)=\nml\ldualopt(\altopttoo)+\epsilon+\nml\ldualopt(-\altopttoo)\) [again use Proposition~\ref{prop:normalisation:map:properties}\ref{it:normalisation:map:properties:constant:additivity} for the equality].
Since this holds for all \(\epsilon>0\), we get that~\(-\nml\ldualopt(-\altopttoo)\leq\nml\ldualopt(\altopttoo)\) for all~\(\altopttoo\in\opts\), so \(\nml\ldualopt\) is self-conjugate, and therefore indeed linear.

To prove the last statement, it suffices to show that \(\nml\group{\posldualopts(\rejectset)}\subseteq\nml\group{\posdualopts(\rejectset)}\).
Use Equations~\eqref{eq:cslfs:from:rejectset} and~\eqref{eq:normalised:cslf:to:rejectset} and the argumentation above to find that \(\nml\group{\posldualopts(\rejectset)}\subseteq\posdualopts(\rejectset)\).
Taking the direct image of both sides of this inclusion under the map~\(\nml\) yields that, indeed, \(\nml(\posldualopts(\rejectset))=\nml(\nml(\posldualopts(\rejectset)))\subseteq\nml(\posdualopts(\rejectset))\), where the equality follows from Proposition~\ref{prop:normalisation:map:properties}\ref{it:normalisation:map:properties:idempotency}.
\end{proof}
\noindent
And as a sort of converse, the following result identifies the mixing and {\archim} binary sets of desirable option sets.
It extends Proposition~\ref{prop:isomorphisms:arch:mixing:noconstants} from {\essarchimty} to {\archimty}.

\begin{proposition}[Binary embedding]\label{prop:mixing:and:archimedean:binary}
Consider any~\(\opt[o]\in\topint(\posopts)\) and any set of desirable options~\(\desirset\in\desirsets\), then \(\rejectset[\desirset]\) is mixing and {\archim} if and only if \(\desirset=\desirset[\dualopt]\) for some~\(\dualopt\in\posdualopts\), and we can always make sure that \(\dualopt(\opt[o])=1\).
\end{proposition}

\begin{proof}
% Checked by Gert
Combining the results of Propositions~\ref{prop:mixingbinaryiffD} and~\ref{prop:archimedeanbinaryiffD}, we find that \(\rejectset[\desirset]\) is mixing and {\archim} if and only if \(\desirset\) is.
Propositions~\ref{prop:isomorphisms:arch:mixing:noconstants} and~\ref{prop:mixing:and:essentially:archimedean:noconstants} then make sure that \(\desirset=\desirset[\dualopto]\), where \(\dualopto\in\dualopts(\desirset)\subseteq\posdualopts\).
Proposition~\ref{prop:normalisation:map:properties}\ref{it:normalisation:map:properties:fixed:point}\&\ref{it:normalisation:map:properties:linear:normalisation} now guarantee that \(\dualopto(\opt[o])=\nml\dualopto(\opt[o])=1\).
\end{proof}
\noindent
In combination with Corollary~\ref{cor:archimedean:representation:rejectsets:noconstants:twosided}, these propositions lead to another important representation result.

\begin{corollary}[Representation]\label{cor:archimedean:mixing:representation:rejectsets:noconstants:twosided}
Consider any~\(\opt[o]\in\topint(\posopts)\), then a set of desirable option sets~\(\rejectset\in\rejectsets\) is mixing and {\archim} if and only if there is some non-empty set~\(\setofdualopts\subseteq\posdualopts\) of positive linear bounded real functionals~\(\dualopt\), with moreover \(\dualopt(\opt[o])=1\), such that \(\rejectset=\bigcap\cset{\rejectset[\dualopt]}{\dualopt\in\setofdualopts}\).
In that case, the largest such set~\(\setofdualopts\) is~\(\nml\group{\posldualopts(\rejectset)}=\nml\group{\posdualopts(\rejectset)}\).
\end{corollary}

\begin{proof}
% Checked by Gert
For necessity, assume that \(\rejectset\) is mixing and {\archim}.
Then we infer from Corollary~\ref{cor:archimedean:representation:rejectsets:noconstants:twosided} that there is some non-empty set~\(\setofldualopts\subseteq\posldualopts\) of positive superlinear bounded real functionals such that \(\rejectset=\bigcap\cset{\rejectset[\ldualopt]}{\ldualopt\in\setofldualopts}\) and that \(\setofldualopts\subseteq\ldualopts\group{\rejectset}\).
It is the import of the discussion in Section~\ref{sec:normalisation}---and Equation~\eqref{eq:normalised:cslf:to:rejectset}---that we can replace these sets by~\(\nml\setofldualopts\subseteq\nml(\posldualopts\group{\rejectset})=\nml(\posdualopts\group{\rejectset})\), where the last equality follows from Proposition~\ref{prop:mixing:rejectsets:linear:dominance}.
If we let \(\setofdualopts\coloneqq\nml(\setofldualopts)\), then we are done, because Proposition~\ref{prop:normalisation:map:properties}\ref{it:normalisation:map:properties:linear:normalisation} guarantees that \(\dualopt(\opt[o])=1\) for all~\(\dualopt\in\nml(\dualopts(\rejectset))\).

For sufficiency, consider any set~\(\setofdualopts\subseteq\posdualopts\) of positive linear bounded real functionals~\(\dualopt\), with moreover \(\dualopt(\opt[o])=1\).
It then follows from Proposition~\ref{prop:mixing:and:archimedean:binary} that all~\(\rejectset[\dualopt]\) are mixing and {\archim}, and so is, therefore, their intersection.

The proof of the last statement goes along the same lines as the necessity proof, where we can simply replace \(\setofldualopts\) by its superset~\(\posldualopts(\rejectset)\).
\end{proof}

\begin{runningexample}\label{ex:nonbinary:archimedean}
For our coin example, it follows from Corollary~\ref{cor:archimedean:mixing:representation:rejectsets:noconstants:twosided} and the discussions in instalments~\ref{ex:identical:sides:nonbinary:inference}, \ref{ex:functionals}, \ref{ex:essential:archimedeanity} and~\ref{ex:normalisation} that the set of desirable option sets~\(\rejectset[{\heads\,\text{or}\,\tails}]\) is {\archim} and mixing.
\stopit
\end{runningexample}

\section{Conclusions}\label{sec:conclusions}
With all these results still fresh in our minds, a brief comparison between the {\archimty} notion introduced here, and related notions discussed elsewhere, will be useful in putting them in perspective.
\emph{{\Essarchimty}} for sets of desirable \emph{options}, as discussed in Section~\ref{sec:essential:archimedeanity}, is inspired by Walley's notion of \emph{strictly desirable gambles} \cite{walley1991}.\footnote{Walley's use of the term `strict' in a context where it can be confused with the much general notion of `strict preference' is a bit unfortunate.} 
One direct way of lifting the openness condition to sets of desirable \emph{option sets} is based on the `includes some open neighbourhood of each of its points' aspect of openness, and led Jasper De Bock and me to introduce and study what we now prefer to call `\emph{strong {\archimty}}' for sets of desirable option sets.
That it is indeed stronger than {\archimty}, is exemplified by the fact that it admits representation results in terms of \emph{closed} sets of (super)linear functionals \cite{debock2019:choice:arxiv,debock2019:choice:isipta,debock2020:axiomatic:archimedean:arxiv}.
De Bock has shown recently \cite{debock2020:axiomatic:archimedean:arxiv} that there is alternative way of lifting the `includes some open neighbourhood of each of its points' aspect of openness to sets of desirable option sets that does not lead to strong {\archimty}, but rather to the {\archimty} notion I have introduced here.

\emph{{\archimty}} for sets of desirable \emph{options} is based on an `intersection of open sets' idea, and is therefore closely related to \emph{even convexity} \cite{daniilidis2002:even:convexity}, and the \emph{even continuity} that Cozman \cite{cozman2018:evenlyconvex} defined for finite-dimensional option spaces.
In their seminal work on partially ordered (binary) preferences on horse lotteries, Seidenfeld et al.~\cite{seidenfeld1995} introduced an {\archimty} condition, which I'll call \emph{SSK-{\archimty}} here, and which Van Camp and Seidenfeld \cite{vancamp2019:isipta:poster} have recently argued to be strictly weaker than Cozman's even continuity.
It is this \emph{SSK-{\archimty}} that Seidenfeld et al.~\cite{seidenfeld2010} lift to the context of choice functions on horse lotteries, and which, in combination with mixingness, is instrumental in allowing them to prove their representation result. 
How this lifted \emph{SSK-{\archimty}} condition is related to the {\archimty} condition studied here, is after quite some study still far from clear to me, and therefore an open question, especially because (i) as I've just mentioned, the unlifted SSK-{\archimty} is strictly weaker than even convexity---which in turn is equivalent to our unlifted {\archimty} on finite-dimensional option spaces---and (ii) the representation result for mixing choice models that Seidenfeld et al.~\cite{seidenfeld2010} have proved based on this lifted \emph{SSK-{\archimty}} condition is somewhat more involved, and needs an extra condition in order to more closely resemble the one in Corollary~\ref{cor:archimedean:mixing:representation:rejectsets:noconstants:twosided}.

The results presented here constitute the basis for a very general theory of binary and non-binary choice.
Its foundations are laid by the coherence axioms, which can be made stronger by adding mixingness and {\archimty}, separately or jointly.
For each of the sets of axioms thus formed, we get a conservative inference framework with corresponding closure operators, as well as representation theorems that allow us to construe all coherent, {\archim} or mixing models---as well as combinations of them---as intersections (infima) of specific types of binary ones.
These representations are especially interesting because they lead to a complete axiomatic characterisation of various well-known decision making schemes.
To give but one example, the (coherent and) {\archim} and mixing models are exactly the ones that correspond to decision making using Levi's E-admissibility scheme~\cite{levi1980a,troffaes2007} associated with general---not necessarily closed or convex---sets of linear previsions.
I believe that such a characterisation---jointly with the one in Jasper De Bock's paper~\cite{debock2020:axiomatic:archimedean:arxiv}---is achieved here for the first time in its full simplicity and generality, in Corollary~\ref{cor:archimedean:mixing:representation:rejectsets:noconstants:twosided}.\footnote{The representation result by Seidenfeld et al.~\cite{seidenfeld2010} for these types of models in the context of horse lotteries has, to the best of my knowledge, so far only been proved in a less general context and for the special case of mixing choice models, and the conditions for representation seem, at least to me, less intuitive and more involved than the ones presented here.}
And the theory is also flexible enough to allow for characterisations for a plethora of other schemes, amongst which Walley--Sen maximality~\cite{walley1991,troffaes2007}.
Indeed, for the \emph{binary choice models} we get the decision making schemes based on
\begin{enumerate}[label=\upshape(\roman*),leftmargin=*]
\item maximality for sets of desirable gambles, covered by coherent sets of desirable options;
\item lexicographic probability orderings, covered by mixing sets of desirable options;
\item evenly convex sets of positive superlinear bounded real functionals---lower previsions essentially---covered by {\archim} sets of desirable options; and
\item evenly convex sets of positive linear bounded real functionals---linear previsions essentially---covered by {\archim} and mixing sets of desirable options.
\end{enumerate}
And for their more general, \emph{non-binary} counterparts we get, through the representation theorems, schemes that are based on arbitrary intersections---infima---of a whole variety of such binary cases.

What I haven't talked about here are the more constructive aspects of the various conservative inference frameworks.
The representation results in this paper essentially allow us to express the closure operator that effects the conservative inference as an intersection of dominating special binary models, which are not always easy (and in some cases even impossible) to identify constructively.
We therefore also need to look for other and more constructive ways of tackling the conservative inference problem; early work on this topic seems to suggest that this is not an entirely hopeless endeavour~\cite{debock2019:choice:isipta,debock2019:choice:arxiv}.
On a related note, the {\archimty} axioms~\ref{ax:desirset:essentially:archimedean:cslf}, \ref{ax:desirset:essentially:archimedean:clf} and~\ref{ax:rejectset:essentially:archimedean:noconstants} are similarly `nonconstructive', as they're based on the existence of (super)linear functionals that `do certain things'.
For an equivalent approach to these axioms with a more constructive flavour, and with gambles as options, I refer to Jasper De Bock's recent paper on this topic~\cite{debock2020:axiomatic:archimedean:arxiv}.

In future work it may be interesting to use the results presented here to derive similar axiomatic characterisations, conservative inference frameworks and representation theorems when the option space is not a linear space but a convex set of horse lotteries.
While I gave a few indications about how this can be done in the Introduction, and worked out the details for the special case of uncertainty about variables with only two possible values in the running coin flip example, the focus there is on the linear option space the horse lotteries are `embedded' in.
Using the embedding to directly formulate and study the axioms, inference methods and representations in horse lottery language, could also allow us to better lay bare the relationship with the (less general) approach followed by Seidenfeld et al.~\cite{seidenfeld2010}.

The main advantage of working with sets of desirable option sets is that, on the one hand, it is mathematically equivalent to working with choice or rejection functions, as explained in Section~\ref{sec:non-binary-choice}, right before instalment~\ref{ex:identical:sides:nonbinary:inference} of the running example. 
On the other hand, the framework is formally close enough to that of sets of desirable options in order to make `lifting' axioms from the latter to the former at the same time easy enough to do and likely enough to be successful in producing conservative inference mechanisms and representation theorems. 
This was already showcased in previous work by the author in collaboration with Jasper De Bock \cite{debock2018:choice:smps,debock2018:choice:arxiv,debock2019:choice:arxiv,debock2019:choice:isipta}.
But, the power of the `lifting' procedure is again made clear here in the transition from Section~\ref{sec:archimedanity:binary} to Section~\ref{sec:archimedeanity:non-binary}, where I lift Archimedeanity for sets of desirable options to the corresponding notion for sets of desirable option sets, and therefore indirectly also for choice or rejection functions. 
This, I believe, explains why the results obtained here seem, to me at least, more powerful and elegantly provable than those in earlier work \cite{seidenfeld1995,seidenfeld2010,2017vancamp:phdthesis,2018vancamp:lexicographic,vancamp2018:choice:and:indifference,vancamp2018:exchangeable:choice}, which focuses on other (but mathematically equivalent) types of choice models such as choice or rejection functions, and which relies on other axiomatisations.

\section*{Acknowledgements}
This work owes an intellectual debt to Teddy Seidenfeld, who introduced me to the topic of choice functions, and to Arthur Van Camp, together with whom I started exploring it, several years ago.
I've been working closely together with Jasper De Bock on various aspects of coherent choice, and my continuous discussions with him have, as ever, been very helpful in finding the right questions to ask here.
I also want to say `Thank you so much!' to both anonymous reviewers for their knowledgeable, detailed and extremely helpful comments. 
Any author should feel blessed and honoured for getting reviews like theirs.
The basis for much of the work presented above was laid during a research stay as visiting professor at Durham University in the late spring of 2019.
I am grateful to Frank Coolen and Durham University's Department of Mathematical Sciences for making that stay pleasant, fruitful, and possible.
I finished work on this paper during the second Covid-19 lockdown, in the fall of 2020 and early spring of 2021.

% \bibliographystyle{plain}
% \bibliography{general}

\end{document}